\documentclass{lmcs}
\pdfoutput=1
\usepackage[utf8]{inputenc}

\usepackage{lastpage}
\lmcsdoi{22}{2}{21}
\lmcsheading{}{\pageref{LastPage}}{}{}%
{Apr.~03,~2025}{May~25,~2026}{}


\usepackage[T1]{fontenc}
\usepackage{graphicx}
\usepackage{amssymb,amsmath,latexsym,mathrsfs}
\usepackage{xspace}
\usepackage{enumitem}
\usepackage{comment}
\usepackage{tcolorbox}

\usepackage{tabularray}%
\allowdisplaybreaks

\usepackage{tikz}
\usetikzlibrary{arrows, calc, patterns, positioning, shapes}
\usetikzlibrary{decorations.pathmorphing}
\tikzset{
	modal/.style={>=stealth', shorten >=1pt, shorten <=1pt, auto, node distance=1.5cm, semithick},
	world/.style={circle, draw, fill=gray!15},
	point/.style={circle, draw, fill=black, inner sep=0.5mm},
	reflexive/.style={-,in=120, out=60, loop, looseness=#1},
	reflexive/.default={5},
	reflexive point/.style={-, in=135, out=45, loop, looseness=#1},
	reflexive point/.default={25},
	reflexive above/.style={-, loop, in=120, out=60, looseness=#1},
	reflexive above/.default={7},
	reflexive below/.style={-, loop, in=240, out=300, looseness=#1},
	reflexive below/.default={7},
	reflexive left/.style={-, loop, in=150, out=210, looseness=#1},
	reflexive left/.default={7},
	reflexive right/.style={-,loop, in=30, out=330, looseness=#1},
	reflexive right/.default={7}
}

\usepackage{algorithm,algorithmicx,algpseudocodex}
\algrenewcommand\algorithmicrequire{\textbf{Input:}}
\algrenewcommand\algorithmicensure{\textbf{Output:}}
\algnewcommand{\Initialize}[1]{\State\textbf{Initialize:} #1}
\algnewcommand\True{\textbf{true}\xspace}
\algnewcommand\False{\textbf{false}\xspace}
\usepackage{multicol}
\usepackage{hyperref}

\catcode`@=11
\def\diamondplus{\mathbin{\mathpalette\diamondplus@\relax}}
\def\diamondplus@#1#2{%
  \vcenter{%
    \hbox{%
      \setbox\z@=\hbox{$\m@th#1\oplus$}%
      \dimen@=\ht\z@ \advance\dimen@ \dp\z@
      \resizebox{!}{\dimen@}{%
        \rotatebox[origin=c]{45}{$\m@th#1\boxtimes$}%
      }% resizebox
    }% hbox
  }% vcenter
}
\catcode`@=12
\def\ImportFromMnSymbol#1{%
  \DeclareFontFamily{U} {MnSymbol#1}{}
  \DeclareFontShape{U}{MnSymbol#1}{m}{n}{
   <-6> MnSymbol#15
   <6-7> MnSymbol#16
   <7-8> MnSymbol#17
   <8-9> MnSymbol#18
   <9-10> MnSymbol#19
   <10-12> MnSymbol#110
   <12-> MnSymbol#112}{}
  \DeclareFontShape{U}{MnSymbol#1}{b}{n}{
   <-6> MnSymbol#1-Bold5
   <6-7> MnSymbol#1-Bold6
   <7-8> MnSymbol#1-Bold7
   <8-9> MnSymbol#1-Bold8
   <9-10> MnSymbol#1-Bold9
   <10-12> MnSymbol#1-Bold10
   <12-> MnSymbol#1-Bold12}{}
  \DeclareSymbolFont{MnSy#1} {U} {MnSymbol#1}{m}{n}
}
\newcommand\DeclareMnSymbol[4]{\DeclareMathSymbol{#1}{#2}{MnSy#3}{#4}}
\ImportFromMnSymbol{C}
\DeclareMnSymbol{\diamondminus}{\mathord}{C}{120}
\makeatletter
\providecommand{\bigsqcap}{%
  \mathop{%
    \mathpalette\@updown\bigsqcup
  }%
}
\newcommand*{\@updown}[2]{%
  \rotatebox[origin=c]{180}{$\m@th#1#2$}%
}
\makeatother

\newcommand{\algand}{\textbf{and}\xspace}

\newcommand{\ab}{\ensuremath{A}\xspace}
\newcommand{\ag}{\text{\normalfont\textsf{A}}\xspace}

\newcommand{\cl}{cl}
\newcommand{\creach}{\rightsquigarrow}

\newcommand{\gr}{\text{\normalfont\textsf{G}}\xspace}
\newcommand{\Lra}{\Leftrightarrow}
\newcommand{\lra}{\leftrightarrow}

\renewcommand{\phi}{\varphi}
\newcommand{\pr}{\text{\normalfont\textsf{P}}\xspace}
\newcommand{\Ra}{\Rightarrow}
\newcommand{\ra}{\rightarrow}
\newcommand{\reach}{\leadsto}
\newcommand{\sk}{\text{\normalfont\textsf{S}}\xspace}
\newcommand{\tail}{tail}
\newcommand{\univ}{U}

\newcommand{\mbN}{\mathbb{N}}
\newcommand{\mbR}{\mathbb{R}}

\newcommand{\lang}{\ensuremath{\mathcal{L}}\xspace}
\newcommand{\langc}{\ensuremath{\mathcal{L}_{C}}\xspace}

\newcommand{\langd}{\ensuremath{\mathcal{L}_{D}}\xspace}
\newcommand{\langdef}{\ensuremath{\mathcal{L}_{DEF}}\xspace}

\newcommand{\langcu}{\ensuremath{\mathcal{L}_{CU}}\xspace}

\newcommand{\langdm}{\ensuremath{\mathcal{L}_{DF}}\xspace}

\newcommand{\langu}{\ensuremath{\mathcal{L}_{U}}\xspace}

\newcommand{\langG}{\ensuremath{\mathcal{L}_{CDEF}}\xspace}
\newcommand{\langGA}{\ensuremath{\mathcal{L}_{CDEF\univ}}\xspace}

\newcommand{\langGUQ}{\ensuremath{\mathcal{L}_{CDEF+-=\equiv\boxplus\boxminus\Box}}\xspace}

\newcommand{\cpdl}{\text{\normalfont CPDL}\xspace}
\renewcommand{\l}{\text{\normalfont L}\xspace}
\newcommand{\lc}{\text{\normalfont L$_C$}\xspace}

\newcommand{\lcdefu}{\text{\normalfont L$_{CDEF\univ}$}\xspace}
\newcommand{\lcu}{\text{\normalfont L$_{C\univ}$}\xspace}
\newcommand{\ld}{\text{\normalfont L$_D$}\xspace}
\newcommand{\ldef}{\text{\normalfont L$_{DEF}$}\xspace}
\newcommand{\lG}{\text{\normalfont L$_{CDEF}$}\xspace}
\newcommand{\lGA}{\text{\normalfont L$_{CDEF\univ}$}\xspace}
\newcommand{\logic}[1]{\text{\normalfont L$_{#1}$}\xspace}
\newcommand{\lu}{\text{\normalfont L$_{\univ}$}\xspace}
\newcommand{\kbone}{\text{\normalfont KB$_1$}\xspace}
\newcommand{\kc}{\text{\normalfont K$^C_2$}\xspace}
\newcommand{\kdn}{\text{\normalfont K$^D_n$}\xspace}
\newcommand{\kutwo}{\text{\normalfont K$^U_2$}\xspace}
\newcommand{\sfivec}{\text{\normalfont S5$^C_2$}\xspace}

\newcommand{\RNum}[1]{\textnormal{\uppercase\expandafter{\romannumeral #1}\relax}}

\title[Epistemic Skills: Reasoning about knowledge and oblivion]{Epistemic Skills\texorpdfstring{\\}{}Reasoning about knowledge and oblivion}
\titlecomment{{\lsuper*}This is an extension of the paper \cite{LW2024} accepted to GandALF 2024.}
\author[X.~Liang]{Xiaolong Liang\lmcsorcid{0009-0007-9079-7752}}[a]
\author[Y.N.~W\'{a}ng]{Y\`{i} N. W\'{a}ng\lmcsorcid{0000-0002-0650-4993}}[b]

\address{School of Philosophy, Shanxi University, Taiyuan, Shanxi, China}
\email{lianghillon@gmail.com}

\address{School of Philosophy and Social Development, Shandong University, Jinan, Shandong, China}
\email{ynw@xixilogic.org, corresponding author}

\begin{document}

\begin{abstract}
This paper presents a class of epistemic logics that captures the dynamics of acquiring knowledge and descending into oblivion, while incorporating concepts of group knowledge. The approach is grounded in a system of weighted models, introducing an ``epistemic skills'' metric to represent the epistemic capacities tied to knowledge updates. Within this framework, knowledge acquisition is modeled as a process of upskilling, whereas oblivion is represented as a consequence of downskilling. The framework further enables exploration of ``knowability'' and ``forgettability,'' defined as the potential to gain knowledge through upskilling and to lapse into oblivion through downskilling, respectively. Additionally, it supports a detailed analysis of the distinctions between epistemic de re and de dicto expressions. The computational complexity of the model checking and satisfiability problems is examined, offering insights into their theoretical foundations and practical implications.
\end{abstract}

\maketitle

\section{Introduction}

Epistemic logic has served as a cornerstone of applied modal logic since its foundational development in formal epistemology \cite{Wright1951,Hintikka1962} and its subsequent integration into computer science \cite{FHMV1995,MvdH1995}.
Building upon this tradition, \emph{Weighted epistemic logic} \cite{DLW2021,LW2022,LW2022b} develops over weighted model logic \cite{LM2014,HLMP2018} by incorporating the concept of \emph{epistemic skills}---broadly defined as any agent capability that enables knowledge updates.

In this framework, the weights assigned to model edges represent the skills that failed to distinguish between pairs of possible worlds, established by a similarity measure. These epistemic skills can be \emph{explicitly} formalized within the language; for instance, the formula $K_a^r \phi$ denotes that agent $a$ knows $\phi$ utilizing a degree $r$ of skills. This aligns with epistemic logics that employ similarity or distance metrics \cite{NT2015,DLW2021,LW2024c}. Formally, in a weighted Kripke model $M = (W, E, V)$ designed for explicit epistemic skills, the standard binary accessibility relation for each agent is replaced by an edge function $E: \ag \to (W \times W) \to \mbR$. Here, \ag represents the set of agents, and $E(a)(w, w')$ represents the degree of \emph{similarity} (or \emph{distance}, depending on the model's configuration) between worlds $w$ and $w'$ for agent $a$. The explicit knowledge formula is interpreted as follows:
\[
M, w \models K_a^r \phi \iff \text{for all $w'$: if $r \leq E(a)(w, w')$, then $M, w' \models \phi$.}
\]
Intuitively, $K_a^r \phi$ holds in world $w$ if and only if $\phi$ is true in all worlds that agent $a$ deems \emph{similar} to $w$---specifically, those worlds where the required skill degree $r$ does not exceed the similarity degree between $w$ and $w'$.
Alternatively, an agent's skills may be \emph{implicitly} defined in the model. In this case, the unindexed epistemic formula $K_a \phi$ expresses that agent $a$ knows $\phi$ based on their model-defined skill set \cite{LW2022b,LW2022,LW2024,LW2024b}. Formally, a weighted Kripke model for implicit epistemic skills, $M = (W, E, C, V)$, extends the model for explicit epistemic skills by including a \emph{capability function} $C : \ag \to \mbR$, where $C(a)$ represents agent $a$'s skill degree. The interpretation of the implicit knowledge formula is:
\[
M, w \models K_a \phi \iff \text{for all $w'$: if $C(a) \leq E(w, w')$, then $M, w' \models \phi$.}
\]
Here, $K_a \phi$ is true in world $w$ if and only if $\phi$ is true in all worlds $w'$ where agent $a$'s capability $C(a)$ is no higher than the similarity between $w$ and $w'$. In this implicit version, the edge function $E$ is often treated as agent-independent, meaning $E(a)(w, w') = E(b)(w, w')$ for all agents $a$ and $b$. Although initially defined as degrees, epistemic skills can be generalized to systems lacking a linear order, such as abstract skill sets, which is the approach adopted in this paper.

A central theme in epistemic logic is the elucidation of various forms of \emph{group knowledge}, with \emph{mutual knowledge} (what all agents know), \emph{common knowledge}, and \emph{distributed knowledge} standing out as well-recognized concepts. While these concepts have been extensively studied in classical epistemic logic \cite{DHK2008}, their theories and applications within weighted epistemic logic have only recently garnered attention \cite{LW2024,LW2024b}. The weighted framework presented here seamlessly integrates \emph{field knowledge} while retaining classical notions of \emph{mutual}, \emph{common}, and \emph{distributed knowledge}.

Classical epistemic logic has also spurred \emph{dynamic} explorations into knowledge-altering actions, such as public announcements, birthing the subfield of dynamic epistemic logic \cite{DHK2008}. This discipline enriches its language with update modalities to depict evolving epistemic states. Prominent frameworks like Public Announcement Logic \cite{Plaza1989} and Action Model Logic \cite{BMS1998}---the former a subset of the latter's broader scope---exemplify this approach. Extensions incorporating the concept of \emph{knowability} have since gained traction \cite{BBDHHL2008,ABDS2010}, illuminating the potential for knowledge acquisition in dynamic informational contexts.

Parallel efforts have tackled the elusive phenomenon of \emph{forgetting}, spanning classical and non-classical logics. Two distinct strategies dominate: syntactical methods, such as the AGM paradigm \cite{AGM1985}, which excise formulas from an agent's knowledge base akin to belief contraction, and semantical methods, which reinterpret knowledge through techniques like erasing propositional truth values \cite{LR1994,LLM2003,DHLM2009,ZZ2009} or redefining agents' awareness scopes \cite{FH1988}. These varied approaches underscore the difficulty of modeling oblivion.

This paper develops a \emph{unified} logical framework for modeling group knowledge, knowledge updates, knowability, and forgettability based on weighted epistemic logic with \emph{implicit} skill sets. In this framework, each agent's skill set is implicitly defined in the model, and \emph{update modalities} drive the representation of knowledge acquisition, the descent into oblivion, and epistemic revision.
%---achieved through direct assignment or adoption of another agent's skills.
These processes are formalized as \emph{upskilling} ($(+S)_a \phi$), \emph{downskilling} ($(-S)_a \phi$), \emph{reskilling} ($(=S)_a \phi$), and \emph{learning} ($(\equiv_b)_a \phi$). Semantically, these operators function as model updates that modify the capability function $C$: expanding agent $a$'s skill set by $S$ (upskilling), reducing it by $S$ (downskilling), resetting it to $S$ (reskilling), or adopting agent $b$'s skills (learning).

By focusing on skill-modifying operations, our analysis extends to \emph{knowability} and \emph{forgettability}, quantifying the potential updates that lead to knowledge or oblivion. Drawing on \cite{BBDHHL2008} (titled ```knowable' as `known after an announcement'\,''), we posit that: the knowable reflects what becomes known via upskilling, while the forgettable captures what fades into the unknown via downskilling. Formally, utilizing \emph{quantifying modalities} for arbitrary upskilling ($\boxplus_a$), downskilling ($\boxminus_a$), and reskilling ($\Box_a$), the statement ``$\phi$ is knowable by agent $a$'' is expressed as $\neg {\boxplus_a} \neg K_a \phi$ or $\neg \Box_a \neg  K_a \phi$, and ``$\phi$ is forgettable by agent $a$'' as $\neg {\boxminus_a} K_a \phi$ or $\neg \Box_a K_a \phi$.

This approach also refines the distinction between \emph{de re} and \emph{de dicto} epistemic expressions \cite{Wright1951,Quine1956}, associating de re knowledge ($\exists x K_a \phi(x)$) with ``knowing how'', and de dicto knowledge ($K_a \exists x \phi (x)$)  with ``knowing that'' (see Section~\ref{sec:de-re}). Through these mechanisms, the proposed framework captures the dynamics of acquiring knowledge and descending into oblivion, as well as the potentials for knowability and forgettability.

We also analyze the computational complexity of these logics. Model checking for logics without quantifying modalities is shown to be in P, while logics with quantifying modalities are PSPACE complete. The satisfiability problem presents greater challenges: logics lacking common knowledge, update or quantifying modalities are PSPACE complete; those including common knowledge in addition, but without update or quantifying modalities, become EXPTIME complete. The complexity of satisfiability for logics with update or quantifying modalities remains an open question.

This paper extends \cite{LW2024} that explored basic epistemic language under various similarity/distance measures and introduced group knowledge, update and quantifying modalities into weighted epistemic logic and provided computational complexity results of the model checking problems. Here, we further introduce two variants of the logic using generalized skill sets (fuzzy sets or lattice structures; see Section~\ref{sec:variants}) to enhance versatility, analyze the complexity of satisfiability problems for logics without update or quantifying modalities (Section~\ref{sec:sat}), and provide a comprehensive revision of the earlier study.

The paper is organized as follows: Section~\ref{sec:logics} defines the formal syntax and semantics of the proposed logics (\ref{sec:syntax}--\ref{sec:semantics}), illustrates model abstraction from datasets and gives examples of model checking of formulas (\ref{sec:rough}--\ref{sec:repres}), generalizes skill sets to fuzzy sets and lattice structures (\ref{sec:variants}), and examines epistemic de re and de dicto expressions (\ref{sec:de-re}). Sections~\ref{sec:mc} and \ref{sec:sat} analyze the computational complexity of model checking and satisfiability problems, respectively. The paper concludes with Section~\ref{sec:conclusion}, offering final remarks. 

\section{Logics}
\label{sec:logics}

Classical epistemic logic \cite{FHMV1995,MvdH1995} is extended in this study through the integration of epistemic skills into the models. An \emph{epistemic skill} is conceptualized broadly here, transcending the conventional notion of a skill. It may encompass a profession inherently tied to specific abilities or a set of skills, as well as a position or privilege that provides resources for acquiring knowledge. For instance, an individual with access to the JFK Assassination Records possesses such an epistemic skill. More generally, any capacity that enhances knowledge can be classified as an epistemic skill. This extension, detailed in this section, offers a unified framework for modeling knowledge and oblivion, alongside diverse forms of group knowledge---namely, \emph{mutual}, \emph{common}, \emph{distributed}, and \emph{field knowledge}.

\begin{conv}[Parameters of the logics]
Four sets, three of which are primitive, are defined as parameters prior to defining the formal languages:
\begin{itemize}
\item \pr: the \emph{set of atomic propositions};
\item \ag: the \emph{set of agents};
\item $\gr \subseteq \wp (\ag)$: the \emph{set of finite, nonempty groups}, where $\wp(\ag)$ is the power set of $\ag$;
\item \sk: the \emph{set of epistemic skills} (e.g., capabilities, professions, or privileges).
\end{itemize}
For simplicity, the sets \pr, \ag and \sk are assumed to be countably infinite throughout this paper, implying that \gr is also countably infinite. These sets are fixed as parameters across all languages considered herein. Alternatively, these sets may be treated as having arbitrary cardinality or as adjustable parameters tailored to specific languages, provided their cardinality is sufficient to support the required expressive power and practical application.
\end{conv}

\subsection{Syntax}
\label{sec:syntax}

The most expressive language introduced here, denoted \langGUQ, has its grammar defined as follows:
\begin{align*}
\phi ::= &\ p \mid \neg \phi \mid (\phi \ra \phi) \mid K_a \phi \mid C_G \phi \mid D_G \phi \mid E_G \phi \mid F_G \phi \mid \\
&\ (+_S)_{a} \phi \mid (-_S)_{a} \phi \mid ({=}_S)_{a} \phi \mid ({\equiv}_b)_{a} \phi \mid
\boxplus_{a} \phi \mid \boxminus_{a} \phi \mid \Box_a \phi
\end{align*}
where $p \in \pr$, $a, b \in \ag$, $G \in \gr$, and $S \subseteq \sk$.

This language subsumes multiple sublanguages of interest. The basic language, \lang, is constructed recursively from atomic propositions using Boolean operators (negation and implication as primitives) and the modal operator $K_a$ ($a \in \ag$), which expresses \emph{individual knowledge}. Thus, \lang serves as the formal language of classical multi-agent epistemic logic, providing a baseline for further extensions.

Four types of \emph{group-knowledge modalities} are incorporated: $C_G$ for \emph{common knowledge}, $D_G$ for \emph{distributed knowledge}, $E_G$ for \emph{mutual knowledge}, and $F_G$ for \emph{field knowledge}, where $G \in \gr$ is a group of agents. Intuitions of these modalities are found in \cite{LW2024b}.

Four types of \emph{update modalities} are introduced to express skill-based epistemic dynamics: $(+_S)_{a}$, $(-_S)_{a}$, $(=_S)_a$ and $(\equiv_b)_a$, where $a, b \in \ag$ are agents and $S \subseteq \sk$ is a skill set. These operators represent, respectively, agent $a$'s \emph{upskilling} (augmenting skills by $S$), \emph{downskilling} (removing skills $S$), \emph{reskilling} (replacing the skill set with $S$), and \emph{learning} (adopting agent $b$'s skill set%
\footnote{Alternative learning operators could be defined, such as $(+_b)_a$ (adding $b$'s skills to $a$'s) or $(-_b)_a$ (removing $b$'s skills from $a$'s), but such extensions are omitted here to avoid unnecessary complexity.}%
). These operators are self-dual, a property verifiable once semantics is introduced.

Additionally, three \emph{quantifying modalities}, or \emph{quantifiers}, are included: $\boxplus_a$, $\boxminus_a$ and $\Box_a$, representing agent $a$'s ability to add, subtract, and modify an arbitrary skill set, respectively. Their duals, $\diamondplus_a$, $\diamondminus_a$ and $\Diamond_a$, are non-primitive and defined accordingly.

Languages extending \lang are named using combinations of subscripts $C$, $D$, $E$, $F$, $+$, $-$, $=$, $\equiv$, $\boxplus$, $\boxminus$ and $\Box$ to indicate the inclusion of specific types of group-knowledge, update or quantifying modalities. For instance, \langdm denotes the extension of \lang with distributed ($D_G$) and field ($F_G$) knowledge modalities, while $\lang_{C + \boxplus}$ extends \lang with common knowledge modality ($C_G$), upskilling modality ($(+_S)_{a}$), and the quantifier for arbitrary upskilling ($\boxplus_a$), applicable for any $a \in \ag$, $G \in \gr$ and $S \subseteq \sk$.

This produces $2^{11} = 2048$ distinct languages extending \lang, determined by the presence or absence of each operator type---four group-knowledge modalities, four update modalities, and three quantifiers---though not all combinations are highlighted here. Additional Boolean operators, such as conjunction and disjunction, follow classical definitions. A \emph{formula} refers to an element of one of these languages, with its specific language determined by context unless specified otherwise.

\subsection{Semantics}
\label{sec:semantics}\label{sec:models}

A class of models is introduced to interpret the languages defined previously.

\begin{defi}\label{def:models}
A \emph{weighted Kripke model} (\emph{model} for short) is a quadruple $(W, E, C, \beta)$, where:
	\begin{itemize}
	\item $W$ is a nonempty set of (possible) worlds or states;
	\item $E : W \times W \to \wp(\sk)$ is an \emph{edge function}, assigning a skill set to each pair of worlds;
	\item $C: \ag \to \wp(\sk)$ is a \emph{capability function} that assigns a skill set to each agent;
	\item $\beta: W \to \wp(\pr)$ is a valuation, mapping each world to a set of true atomic propositions.
	\end{itemize}
The model satisfies two constraints in addition:%
\footnote{They are the constraints of a similarity metric. More or less constraints can be enforced if needed.}
	\begin{itemize}
		\item Positivity (i.e., congruence implies equality): for all $w, u \in W$, if $E(w, u) = \sk$, then $w = u$;
		\item Symmetry: for all $w, u \in W$, $E(w, u) = E(u, w)$.
	\end{itemize}
The tuple $(W, E)$, i.e., the pair of the first two elements of a model, is called a \emph{frame}.
\end{defi}
In this definition, the edge function $E$ specifies the skills ineffective for distinguishing between worlds: for any pair $(w, u)$, an agent can differentiate $w$ from $u$ only if her skill set, as assigned by $C$, contains at least one skill not in $E(w, u)$. The positivity condition ensures that if $E(w, u) = \sk$---implying no skill enables discernment---the worlds $w$ and $u$ are identical. Symmetry, meanwhile, guarantees that the epistemic accessibility relation remains symmetric.

Given a capability function $C: \ag \to \wp(\sk)$, agents $a, b, x \in \ag$ and a skill set $S \subseteq \sk$, the following modified capability functions are defined:
\begin{multicols}{2}
$\begin{array}{lll}
C^{a \cup S} (x) &=& \left \{ \begin{tabular}{ll} $C(a) \cup S$,& if $x = a$, \\ $C(x)$,& if $x \neq a$; \end{tabular} \right.
\\[2ex]
C^{a \setminus S} (x) &=& \left \{ \begin{tabular}{ll} $C(a) \setminus S$,& if $x = a$, \\ $C(x)$,& if $x \neq a$; \end{tabular} \right.
\end{array}$

$\begin{array}{lll}
C^{a=S} (x) &=& \left \{ \begin{tabular}{ll} $S$,& if $x = a$, \\ $C(x)$,& if $x \neq a$; \end{tabular} \right.
\\[2ex]
C^{a \equiv b} (x) &=& \left \{ \begin{tabular}{ll} $C(b)$,& if $x = a$, \\ $C(x)$,& if $x \neq a$. \end{tabular} \right.
\\
\end{array}$
\end{multicols}

\noindent Here, $C^{a \cup S}$ denotes a capability function identical to $C$ except at agent $a$, whose skill set is expanded by $S$ (upskilling). Similarly, $C^{a \setminus S}$ reduces $a$'s skill set by $S$ (downskilling), $C^{a = S}$ sets $a$'s skill set to $S$ (reskilling), and $C^{a \equiv b}$ aligns $a$'s skill set with $b$'s (learning). An additional variant, $C^{a \cap S}$, where $a$'s skill set becomes $C(a) \cap S$, is not explicitly included but can be expressed as $C^{a \setminus (\sk \setminus S)}$, consistent with the definition of set intersection through set difference.

The satisfaction criteria for formulas are defined as follows.

\begin{defi}\label{def:semantics}
Given a formula $\phi$, a model $M = (W, E, C, \beta)$, and a world $w \in W$, the notation $M, w \models \phi$ indicates that $\phi$ is \emph{true} or \emph{satisfied} at $w$ in $M$. This relation is defined inductively by the following conditions:
\begingroup
\addtolength{\jot}{-.6ex}
\begin{alignat*}{3}
&M, w \models p & \iff\  & p \in \beta(w)
\\
&M, w \models \neg \psi& \iff\  & \text{not } M, w \models \psi
\\
&M, w \models (\psi \ra \chi)& \iff\  & \text{if $M, w \models \psi$, then $M, w \models \chi$}
\\
&M, w \models K_a \psi & \iff\  & \text{for all $u \in W$, if $C(a) \subseteq E(w, u)$ then $M, u \models \psi$}
\\
&M, w \models E_G \psi& \iff\  & M, w \models K_a\psi \text{ for all $a\in G$}
\\
&M, w \models C_G \psi & \iff\  & \text{for all positive integers $n$, $M, w \models E_G^n \psi$,}
\\
&& & \text{ where $E_G^1\psi := E_G \psi$ and $E_G^n\psi := E^1_G E_G^{n-1}\psi$}
\\
&M, w \models D_G \psi& \iff\  & \text{for all $u \in W$, if $\textstyle\bigcup_{a\in G} C(a) \subseteq E(w, u)$ then $M, u \models \psi$}
\\
&M, w \models F_G \psi & \iff\  & \text{for all $u \in W$, if $\textstyle\bigcap_{a\in G} C(a) \subseteq E(w, u)$ then $M, u \models \psi$}
\\
&M, w \models (+_S)_{a} \psi& \iff\  & M^{a \cup S}, w \models \psi, \text{ where } M^{a \cup S} = (W, E, {C^{a \cup S}}, \beta)
\\
&M, w \models (-_S)_a \psi& \iff\  & M^{a \setminus S}, w \models \psi, \text{ where } M^{a \setminus S} = (W, E, C^{a \setminus S}, \beta)
\\
&M, w \models ({=}_S)_a \psi& \iff\  & M^{a = S}, w \models \psi, \text{ where } M^{a = S} = (W, E, C^{a = S}, \beta)
\\
&M, w \models ({\equiv}_b)_a \psi& \iff\  & M^{a \equiv b}, w \models \psi, \text{ where } M^{a \equiv b} = (W, E, C^{a \equiv b}, \beta)
\\
&M, w \models \boxplus_a \psi& \iff\  & \text{for all $S \subseteq \sk$}, M, w \models (+_S)_a \psi
\\
&M, w \models \boxminus_a \psi& \iff\  & \text{for all $S \subseteq \sk$}, M, w \models (-_S)_a \psi
\\
&M, w \models \Box_a \psi& \iff\  & \text{for all $S \subseteq \sk$}, M, w \models {({=}_{S})_a} \psi.
\end{alignat*}
\endgroup
A formula $\phi$ is \emph{valid} if $M, w \models \phi$ holds for all models $M$ and all worlds $w$, and \emph{satisfiable} if $M, w \models \phi$ holds for some model $M$ and some world $w$.
\end{defi}

Given that $G$ is a finite group, the formula $E_G \psi$ is logically equivalent to $\bigwedge_{a \in G} K_a \psi$. This equivalence suggests that its inclusion in the language is not strictly necessary, serving primarily to ensure comprehensiveness. While $G$ could be allowed to be infinite, the present framework adheres to classical epistemic logic, where groups are conventionally finite (see, e.g., \cite{FHMV1995}). Nevertheless, this equivalence potensionally influences the language's succinctness, preventing $E_G \psi$ from being treated as a simple syntactic shorthand for $\bigwedge_{a \in G} K_a \psi$ in such analyses.

For a group $G \in \gr$, a \emph{$G$-path in a model $M = (W, E, C, \beta)$ from a world $w$ to a world $u$} is a finite sequence $\langle w_0, w_1, \dots, w_n \rangle$ such that $w_0 = w$, $w_n = u$, and for all $i$ where $1 \leq i \leq n$, there exists an agent $a_i \in G$ satisfying $C(a_i) \subseteq E (w_{i-1}, w_{i})$. We denote $w \reach^M_G u$ if there exists a $G$-path from $w$ to $u$ in $M$; omitting the superscript $M$ when the model is clear from context.  The semantics of $C_G \psi$ is equivalently expressed as:
\[ M, w \models C_G \psi \iff \text{for all $u \in W$, if $w \reach_G u$ then $M, u \models \psi.$}\]

Formulas such as $({=}_\emptyset)_{a} \phi$, where agent $a$ is assigned an empty skill set, are permissible. This could alternatively be expressed without an empty set: $({=}_\emptyset)_{a} \phi$ is equivalent to $({=}_S)_a (-_{S})_a \phi$ for any $S \subseteq \sk$. Additionally, both $(+_\emptyset)_a \phi$ and $(-_\emptyset)_a \phi$ are equivalent to $\phi$, as verified through the semantics.% Nonetheless, explicitly allowing empty skill sets alters the complexity of the satisfiability problem for these logics, as detailed in Section~\ref{sec:sat}.

%Note that although $({=}_S)_{a} \phi$ is not a legal formula when $S$ is the empty set $\emptyset$, we can regard it as a defined formula, i.e., $({=}_\emptyset)_{a} \phi := ({=}_S)_a (-_{S})_a \phi$ (for any qualified set $S$). In the mean time, it is not hard to verify that both $(+_\emptyset)_a \phi$ and $(-_\emptyset)_a \phi$, if allowed, are equivalent to $\phi$, so there is no need to worry about the cases with empty sets.

A logic is defined over a given formal language, consisting of the set of valid formulas under the specified semantics. Each logic adopts the naming convention of its corresponding formal language but is denoted in upright Roman typeface, e.g., \l, $\l_{F+\boxplus}$ and $\l_{CDEF+-=\equiv\boxplus\boxminus\Box}$.

\subsection{Abstracting a model from a dataset}
\label{sec:rough}

A primary motivation for the development of weighted epistemic logic stems from the theoretical correspondence between classical modal logic and rough set theory, a connection established in the 1980s and 1990s \cite{OP1984,Vakarelov1989,Orlowska1990,YL1996}. Within this context, the universe of an \emph{approximation space} (a dataset comprising a set of objects) corresponds to the universe of a Kripke model, while a category (a subset of the universe) is characterized by a proposition, akin to the extension (truth set) of that proposition within the model. Consequently, the \emph{upper} and \emph{lower approximations} of a proposition $p$ correspond to the modal formulas $\Diamond p$ and $\Box p$, respectively, enabling the description of approximation spaces through modal logic.

Weighted Kripke models extend this correspondence by interpreting the attributes of an approximation space as epistemic skills. This conceptual shift permits operations over attributes that classical modal logic cannot express without introducing modal operators for each attribute set. Furthermore, the update mechanisms inherent in weighted Kripke models enhance the expressive power of rough sets, as studied in recent work \cite{LW2025}.

To illustrate this framework, we consider the solar system example adapted from \cite[Example~1]{Orlowska1988}, which categorizes planets based on specific attributes (see Table~\ref{tbl:planets}).%
\footnote{The object Pluto is removed due to the redefinition of the term \emph{planet} by the International Astronomical Union in 2006.}
We analyze this dataset using weighted epistemic logic, distinguishing our approach from standard rough set analysis.
\begin{table}[h]
\caption{\label{tbl:planets}Attributes of Solar System Planets (adapted from \cite[Example~1]{Orlowska1988}).}
\medskip
\begin{center}
\begin{tblr}{
	hline{1,10} = {1pt, solid},
	hline{2} = {.5pt, solid},
%	rows = {rowsep=1pt},
}
Planet & Size $(1)$ & Distance ($2$) & Moon ($3$) \\
Mercury ($m_e$) & small & near & no \\
Venus ($v$) & small & near & no \\
Earth ($e$) & small & near & yes \\
Mars ($m_a$) & small & near & yes \\
Jupiter ($j$) & large & far & yes \\
Saturn ($s$) & large & far & yes \\
Uranus ($u$) & medium & far & yes \\
Neptune ($n$) & medium & far & yes \\
\end{tblr}
\end{center}
\end{table}

In this model, the planets Mercury, Venus, Earth, Mars, Jupiter, Saturn, Uranus, and Neptune are denoted by $m_e$, $v$, $e$, $m_a$, $j$, $s$, $u$, and $n$, respectively. The attribute set $\sk = \{ 1, 2, 3 \}$, where $1$ represents ``size,'' $2$ represents ``distance from the sun,'' and $3$ represents ``possession of a moon.''

We derive a frame $(W, E)$ from the data in Table~\ref{tbl:planets}, where the set of worlds is $W = \{m_e, v, e, m_a, j, s, u, n \}$ and the edge function $E$ is detailed in Table~\ref{tbl:edge}. Additionally, we assume a capability function $C$ and a valuation $\beta$ defined as follows:
\begin{itemize}
\item $C(a) = \{1, 2\}$: Agent $a$ has access to attributes $1$ (size) and $2$ (distance); 
\item $C(b) = \{1, 3\}$: Agent $b$ has access to attributes $1$ (size) and $3$ (moon);
\item $C(c) = \{2, 3\}$: Agent $c$ has access to attributes $2$ (distance) and $3$ (moon);
\item The proposition $p$ (representing ``looks beautiful'') holds exactly at the worlds $e$, $m_a$, $j$, and $s$. Formally, $p \in \beta(e) \cap \beta(m_a) \cap \beta(j) \cap \beta(s)$.
\end{itemize}

\begin{table}
\centering
\parbox{.8\textwidth}{
\caption{\label{tbl:edge}Edge Function $E$ Derived from Table~\ref{tbl:planets}.\\
Note: $E$ is symmetric; blank cells mirror the diagonal.}
}
\small
\begin{tblr}{
	hline{2} = {.5pt, solid},
	hline{3-9} = {.2pt, dashed},
	vline{2} = {.5pt, solid},
	columns = {c},
	rows = {mode=math},
}
E & m_e & v & e & m_a & j & s & u & n \\
m_e & \{1, 2, 3\} & \{1, 2, 3\} & \{1, 2\} & \{1, 2\} & \emptyset & \emptyset & \emptyset & \emptyset \\
v & & \{1, 2, 3\} & \{1, 2\} & \{1, 2\} & \emptyset & \emptyset & \emptyset & \emptyset \\
e & & & \{1, 2, 3\} & \{1, 2, 3\} & \{3\} & \{3\} & \{3\} & \{3\} \\
m_a & & & & \{1, 2, 3\} & \{3\} & \{3\} & \{3\} & \{3\} \\
j & & & & & \{1, 2, 3\} & \{1, 2, 3\} & \{2, 3\} & \{2, 3\} \\
s & & & & & & \{1, 2, 3\} & \{2, 3\} & \{2, 3\} \\
u & & & & & & & \{1, 2, 3\} & \{1, 2, 3\} \\
n & & & & & & & & \{1, 2, 3\} \\
\end{tblr}
\end{table}

This configuration yields the model $M = (W, E, C, \beta)$. We can verify the following properties within this model:
\begin{itemize}
\item Individual knowledge:
	\begin{itemize}
	\item $M, e \not\models K_a p$: Agent $a$ does not know that Earth looks beautiful. Intuitively, because $C(a) = \{1, 2\}$, agent $a$ cannot distinguish Earth ($e$) from other planets like Mercury ($m_e$) or Venus ($v$) that share similar size and distance attributes. Since $p$ is false at $m_e$ (Mercury does not look beautiful in this valuation), agent $a$ does not know $p$ at $e$.
	\item $M, e \models \neg K_a \neg p$: Agent $a$ considers it \emph{possible} that Earth looks beautiful.
	\item $M, e \models K_b p \wedge K_c p$: Both agents $b$ and $c$ know that Earth looks beautiful.
	\end{itemize}
\item Group knowledge (see \cite{LW2024b} for a detailed explanation of these operators):
	\begin{itemize}
	\item It follows that $E_{\{a, b\}} p$, $E_{\{a, c\}} p$ and $E_{\{a, b, c\}} p$ are all false (at $e$ in $M$), implying that mutual knowledge of $p$ does not exist for these groups. Consequently, common knowledge $C_{\{a, b\}} p$, $C_{\{a, c\}} p$ and $C_{\{a, b, c\}} p$ are also false.
	\item $M, e \models C_{\{b, c\}} p$: It is common knowledge between $b$ and $c$ that Earth looks beautiful. In this specific model, $E^n_{\{b, c\}} p$ is true at both $e$ and $m_a$ for all $n \in \mbN$ (in particular, $E_{\{b, c\}} p$ is true).	
	\item $M, e \models D_{\{a, b\}} p$: Agents $a$ and $b$ possess distributed knowledge of $p$. By pooling their skills, the group capability becomes $\{1, 2, 3\}$. With this complete skill set, the only planets indistinguishable from Earth is itself and Mars ($m_a$). Since $p$ is true at both Earth and Mars, the group \emph{distributedly knows} $p$. Similarly, distributed knowledge of $p$ holds for groups $\{a, c\}$, $\{b, c\}$, and $\{a, b, c\}$.
	\item $M, e \not\models F_{\{a, b\}} p$: Field knowledge of $p$ fails for group $\{a, b\}$. The shared skill set is $\{1\}$ (size). Based on size alone, the set of similar planets is $\{m_e, v, e, m_a\}$. Since $p$ is false at $m_e$, the group does not possess filed knowledge of $p$. For similar reasons, field knowledge of $p$ fails for groups $\{a, c\}$, $\{b, c\}$, and $\{a, b, c\}$.
	\end{itemize}
\item Update modalities (to see the following, one just check carefully the skill set of each agent):
	\begin{itemize}
	\item Upskilling: $M, e \not\models (+_{\{1, 2\}})_a K_a p$, but $M, e \models (+_{\{3\}})_a K_a p$;
	\item Downskilling: $M, e \models (-_{\{2\}})_b K_b p$, but $M, e \not \models (-_{\{1\}})_b K_b p$ and $M, e \not \models (-_{\{3\}})_b K_b p$;
	\item Reskilling: $M, e \models (=_{\{1, 3\}})_c K_c p$ and $M, e \models (=_{\{2, 3\}})_c K_c p$, but $M, e \not\models (=_{\{1, 2\}})_c K_c p$;
	\item Learning: $M, e \models (\equiv_{b})_a K_a p$ and $M, e \models (\equiv_{c})_a K_a p$, but $M, e \not\models (\equiv_{a})_b K_b p$.
	\end{itemize}
\item Quantifiers:
	\begin{itemize}
	\item Knowability: From the above we conclude that proposition $p$ is knowable by all agents;
	\item Forgettability: Similarly, proposition $p$ is forgettable by all agents;
	\item Arbitrary updates: $M, e \models \Box_a p$, but $M, e \not\models \Box_a K_a p$ and $M, e \not\models \Box_a \neg K_a p$.
	\end{itemize}
\end{itemize}

\subsection{A sophisticated model: representation and model checking}
\label{sec:repres}

This section presents a more complex formal model to illustrate the expressive power of the epistemic logic framework, showcasing several formulas that hold within it. The model is adapted from \cite[Example~3.2]{LW2022b}, which offers an intuitive interpretation of its structure. For a comprehensive illustration of group knowledge in the context of weighted models, we refer to \cite{LW2024b}. Recent work establishes connections between rough sets theory and weighted modal logic \cite{LW2025}, enabling applications of this framework in data science, particularly for modeling knowledge and uncertainty in complex datasets.

\subsubsection{Model definition}
Let $s_1, s_2, s_3, s_4 \in \sk$ denote epistemic skills and $a, b, c \in \ag$ represent agents. The model is defined as $M = (W, E, C, \beta)$ with the following components:
\begin{itemize}
\item $W = \{ w_1, w_2, w_3, w_4, w_5 \}$ constitutes the set of possible worlds.
\item $E : W \times W \to \wp(\sk)$, the edge function, is defined by:
	\begin{itemize}
	\item $E(w_1,w_1) = E(w_2,w_2) = E(w_3,w_3) = E(w_4,w_4) = E(w_5,w_5) = \{ s_1, s_2, s_3, s_4 \}$,
	\item $E(w_1, w_2) = E(w_2, w_1) = E(w_3, w_5) = E(w_5, w_3) = \{ s_1, s_4\}$,
	\item $E(w_1, w_3) = E(w_2, w_5) = E(w_3, w_1) = E(w_5, w_2) = \{ s_1, s_2, s_3\}$,
	\item $E(w_1, w_4) = E(w_4, w_1) = \emptyset$,
	\item $E(w_1, w_5) = E(w_2, w_3) = E(w_3, w_2) = E(w_5, w_1) = \{ s_1 \}$,
	\item $E(w_2, w_4) = E(w_4, w_2) = \{ s_2, s_3 \}$,
	\item $E(w_3, w_4) = E(w_4, w_3) = \{ s_4 \}$,
	\item $E(w_4, w_5) = E(w_5, w_4) = \{ s_2, s_3, s_4 \}$.
	\end{itemize}
\item $C: \ag \to \wp(\sk)$, the capability function, assigns skill sets to agents $a$, $b$ and $c$:
	\begin{itemize}
	\item $C(a) = \{s_1, s_2, s_3\}$,
	\item $C(b) = \{s_2, s_3, s_4\}$,
	\item $C(c) = \{s_4\}$.
	\end{itemize}
\item $\beta: W \to \wp(\pr)$, the valuation function, assigns proposition sets to each world:
	\begin{itemize}
	\item $\beta(w_1) = \{p_1, p_2\}$
	\item $\beta(w_2) = \{p_1, p_3\}$
	\item $\beta(w_3) = \{p_1, p_2, p_4\}$
	\item $\beta(w_4) = \{p_3, p_4\}$
	\item $\beta(w_5) = \{p_1,p_3,p_4\}$.
	\end{itemize}
\end{itemize}
That $M$ satisfies the model conditions---positivity and symmetry---can be readily confirmed. Representing $M$ diagrammatically often aids understanding (see Figure~\ref{fig:sim-model}). In such a diagram, nodes correspond to worlds, and undirected edges indicate accessibility relations, labeled with the skill sets from $E$ that define indistinguishability between worlds. An edge labeled with $\emptyset$, as between $w_1$ and $w_4$, signifies that all agents can distinguish the pair except for totally incompetent agents (i.e., agents with an empty skill set), and such edges are typically omitted from the diagram. This visualization clarifies the model's structure and connectivity.

\begin{figure}
\centering
\parbox{.66\textwidth}{%
\begin{tikzpicture}[modal, node distance=1.2cm and 2cm, world/.append style={minimum width=1cm, minimum height=.5cm}, inner sep=.5ex]
\scriptsize
\node[world, ellipse split] (w1) {$w_1$ \nodepart{lower} $p_1, p_2$};
\node[world, ellipse split] (w2) [right=of w1] {$w_2$ \nodepart{lower} $p_1, p_3$};
\node[world, ellipse split] (w3) [below=of w2] {$w_3$ \nodepart{lower} $p_1, p_2, p_4$};
\node[world, ellipse split] (w4) [right=of w3] {$w_4$ \nodepart{lower} $p_3, p_4$};
\node[world, ellipse split] (w5) [right=of w2] {$w_5$ \nodepart{lower} $p_1, p_3, p_4$};

\path (w1) edge [reflexive above=3] node[above] {\tiny $s_1, s_2, s_3, s_4$} (w1);
\path (w1) edge node[above] {\tiny $s_1, s_4$} (w2);
\path (w1) edge node[below left] {\tiny $s_1, s_2, s_3$}(w3);
\path (w1) edge [bend left = 36] node[above] {\tiny $s_1$} (w5);
\path (w2) edge [reflexive above=3] node[above] {\tiny $s_1, s_2, s_3, s_4$} (w2);
\path (w2) edge node [left] {\tiny $s_1$} (w3);
\path (w2) edge node [left, xshift=-3ex, yshift=1.5ex] {\tiny $s_2, s_3$} (w4);
\path (w2) edge node [above] {\tiny $s_1, s_2, s_3$} (w5);
\path (w3) edge [reflexive below=3] node[below] {\tiny $s_1, s_2, s_3, s_4$} (w3);
\path (w3) edge node[below] {\tiny $s_4$} (w4);
\path (w3) edge node [right, xshift=3ex, yshift=1.5ex] {\tiny $s_1, s_4$} (w5);
\path (w4) edge [reflexive below=3] node[below] {\tiny $s_1, s_2, s_3, s_4$} (w4);
\path (w4) edge node[right] {\tiny $s_2, s_3, s_4$} (w5);
\path (w5) edge [reflexive above=3] node[above] {\tiny $s_1, s_2, s_3, s_4$} (w5);
\end{tikzpicture}
}
\parbox{.3\textwidth}{%
$\footnotesize
\begin{array}{l}
\bigskip\bigskip\\
C(a) = \{s_1, s_2, s_3\}\\
C(b) = \{s_2, s_3, s_4\}\\
C(c) = \{s_4\}\bigskip\\
C^{a \cup \{s_4\}}(a)=\{s_1, s_2, s_3, s_4\}\\
C^{a \setminus \{s_2, s_3\}}(a)=\{s_1\}\\
C^{c=\{s_2\}}(c)=\{s_2\}\\
C^{b\equiv c}(b)=\{s_4\}\\
\end{array}$
}

\caption{Illustration of the model $M$. Curly brackets are omitted from set labels for brevity. Edges labeled with the empty set, such as between $w_1$ and $w_4$, indicate universal distinguishability---except by totally incompetent agents (those with an empty skill set)---and are not depicted in the diagram.}\label{fig:sim-model}
\end{figure}

\subsubsection{Model checking results}
\label{subsubsec:mc}

Given the model $M = (W, E, C, \beta)$ defined above, the following properties hold.

\paragraph{\mdseries\itshape Individual knowledge}
\begin{enumerate}
\item (Knowledge of a proposition) $M, w_2 \models K_a p_3$: In world $w_2$, agent $a$ knows proposition $p_3$.
\item (Uncertainty) $M, w_4 \models \neg K_b p_1 \wedge \neg K_b \neg p_1$: Agent $b$ does not know whether $p_1$ is true or false at $w_4$, reflecting uncertainty about $p_1$.
\item (Meta-knowledge) $M, w_3 \models K_c (K_a p_3 \vee K_a \neg p_3)$: In world $w_3$, agent $c$ knows that agent $a$ knows the truth value of $p_3$, even if $c$ does not know $p_3$ themselves.
\end{enumerate}

\paragraph{\mdseries\itshape Group knowledge}
\begin{enumerate}[resume]
\item (Mutual knowledge) $M, w_4 \models E_{\{a,b\}} (p_3 \wedge p_4)$: In world $w_4$, agents $a$ and $b$ mutually know both $p_3$ and $p_4$.
\item (Common knowledge) $M, w_5 \models (\neg C_{\{a,c\}} p_1 \wedge \neg C_{\{a,c\}} \neg p_1) \wedge (\neg C_{\{a,c\}} p_2 \wedge \neg C_{\{a,c\}} \neg p_2)$: In world $w_5$, neither $p_1$ nor $p_2$, nor their negations, constitute common knowledge between agents $a$ and $c$.
\item (Distributed knowledge) $M, w_4 \models D_{\{a,b\}} (\neg p_1 \wedge p_4)$: In world $w_4$, by pooling their skills $(C(a) \cup C(b))$, the group $\{a, b\}$ can distinguish enough worlds to know that $p_1$ is false and $p_4$ is true.
\item (Field knowledge) $M, w_4 \models \neg F_{\{a,b\}} \neg p_1 \wedge \neg F_{\{a,b\}} p_4$: In world $w_4$, neither $\neg p_1$ nor $p_4$ qualifies as field knowledge for agents $a$ and $b$, which relies on the intersection of skills $(C(a) \cap C(b))$.
\end{enumerate}

\paragraph{\mdseries\itshape Epistemic updates (dynamics)}
\begin{enumerate}[resume]
\item (Upskilling) $M, w_5 \models \neg K_a p_4 \land (+_{\{s_4\}})_a K_a p_4$: In world $w_5$, agent $a$ does not initially know $p_4$, but would know it upon acquiring skill $s_4$.
\item (Downskilling) $M, w_2 \models K_a p_3 \land (-_{\{s_2, s_3\}})_a \neg K_a p_3$: In world $w_2$, agent $a$ knows $p_3$, but would lose this knowledge if skills $s_2$ and $s_3$ were removed.
\item (Reskilling) $M, w_1 \models E_{\{a, b\}} (\neg K_c p_2 \land (=_{\{s_2\}})_c K_c p_2))$: In world $w_1$, it is mutual knowledge between agents $a$ and $b$ that agent $c$ does not know $p_2$, but $c$ would know it if their skill set were reset strictly to ${s_2}$.
\item (Learning) $M, w_1 \models (\equiv_c)_b \bigwedge_{p\in \{p_1,\dots,p_4\}}(F_{\{b,c\}} p \lra K_b p)$: In world $w_1$, if agent $b$ adopts agent $c$'s skill set, $b$'s individual knowledge aligns with the field knowledge shared between $b$ and $c$ for propositions $p_1$ through $p_4$.
\end{enumerate}

\paragraph{\mdseries\itshape Quantifiers and knowability}
\begin{enumerate}[resume]
\item\label{it:dicto1} $M, w_5 \not \models K_a \diamondplus_a p_4$: In world $w_5$, agent $a$ does not know that there exists an upskilling under which $p_4$ holds (in particular, $M, w_2 \not\models \diamondplus_a p_4$).
\item\label{it:imp-re1} $M, w_5 \models \diamondplus_a K_a p_4$: In world $w_5$, there exists a skill addition (upskilling) under which agent $a$ come to know $p_4$ (agent $a$ knows implicitly how to achieve $p_4$; cf. Section~\ref{sec:de-re}).
\item\label{it:exp-re1} $M, w_5 \not\models (\equiv_a)_c\diamondplus_c K_a (\equiv_c)_a p_4$: In world $w_5$, agent $a$ does not know explicitly how to achieve $p_4$ by upskilling (cf. Section~\ref{sec:de-re}).
\item\label{it:dicto2} $M, w_5 \models K_a \diamondplus_a ( (p_4 \wedge \neg K_a p_4) \vee (\neg p_4 \wedge K_a \neg p_4) )$: In world $w_5$, agent $a$ knows that there exists a skill addition under which $\big((p_4 \wedge \neg K_a p_4) \vee (\neg p_4 \wedge K_a \neg p_4)\big)$ holds.
\item\label{it:imp-re2} $M, w_5 \models \diamondplus_a K_a ( (p_4 \wedge \neg K_a p_4) \vee (\neg p_4 \wedge K_a \neg p_4) )$: In world $w_5$, there exists a skill addition under which agent $a$ come to know $\big((p_4 \wedge \neg K_a p_4) \vee (\neg p_4 \wedge K_a \neg p_4)\big)$ (agent $a$ knows implicitly how to achieve $\big((p_4 \wedge \neg K_a p_4) \vee (\neg p_4 \wedge K_a \neg p_4)\big)$).
\item\label{it:exp-re2} $M, w_5 \not \models (\equiv_a)_c \diamondplus_c K_a (\equiv_c)_a ( (p_4 \wedge \neg K_a p_4) \vee (\neg p_4 \wedge K_a \neg p_4) )$. Denote $\psi = (p_4 \wedge \neg K_a p_4) \vee (\neg p_4 \wedge K_a \neg p_4)$. The leftmost $K_a$ enforces that the statement is true if and only if $(\equiv_c)_a \psi$ is true in both $w_2$ and $w_5$. However, this is impossible. $(\equiv_a)_c \Diamond_c$ restricts $c$'s potential upskilling to only two possibilities (becoming $\{s_1, s_2, s_3\}$ or $\{s_1, s_2, s_3, s_4\}$), which is learnt by $a$ for verification of the truth of $\psi$. In case $a$'s skill set is $\{s_1, s_2, s_3\}$, $\psi$ is false in $w_2$; If $a$'s skills set is $\{s_1, s_2, s_3, s_4\}$, then $\psi$ is false in $w_5$.
\item $M, w_3 \models \diamondminus_b \bigwedge_{p\in \{p_1,\dots,p_4\}} (\neg C_{\{a, b\}} p \land \neg C_{\{a, b\}} \neg p)$: In world $w_3$, some downskilling of agent $b$ could result in a world where none of the propositions $p_1$ through $p_4$, nor their negations, are common knowledge between agents $a$ and $b$.
\item $M, w_2 \models K_c p_1\land \neg K_c p_3\land \Diamond_c (\neg K_c p_1\land K_c p_3)$: In world $w_2$, agent $c$ knows $p_1$ but not $p_3$, yet there exists a skill modification (reskilling) under which $c$ would cease to know $p_1$ while coming to know $p_3$.
\end{enumerate}

\subsection{Variants}
\label{sec:variants}

In this paper, epistemic skills are represented using abstract skill sets $S \subseteq \sk$, or more formally, as the ordered set $(\wp(\sk), \subseteq)$, where the subset relation serves to compare skill sets implicitly. Alternatively, other structures can be adopted: real numbers, offering a more concrete representation, or a partial order, providing a more generalized approach, to indicate degrees of skill proficiency, as explored in \cite{LW2022}. Furthermore, the ordering of skill sets can be extended to structures such as fuzzy sets or a lattice, thereby broadening the framework's adaptability.

\paragraph{Fuzzy skill sets}

Each $X \in \wp(\sk)$ can be generalized to a fuzzy skill set $X = (\sk, \mu_X)$, where $\mu_X: \sk \to [0, 1]$ is a membership function assigning each skill $s \in \sk$ a value between 0 and 1, representing its degree of membership in $X$. For two fuzzy skill sets $S = (\sk, \mu_S)$ and $T = (\sk, \mu_T)$, the subset relation, union, intersection, and difference operations are defined as follows:
\[
\begin{array}{ccl}
S \subseteq T & \Lra & \forall s \in \sk : \mu_S (s) \leq \mu_T (s) \\
S \cup T & = & ( \sk, \max (\mu_S, \mu_T)) \\
S \cap T & = & ( \sk, \min (\mu_S, \mu_T)) \\
S \setminus T & = & S \cap \bar{T}, \\
\end{array}
\]
where $\max(\mu_S, \mu_T)$ maps each $s \in \sk$ to $\max(\mu_S(s), \mu_T(s))$, $\min(\mu_S, \mu_T)$ maps each $s \in \sk$ to $\min(\mu_S(s), \mu_T(s))$, and $\bar{T} = (\sk, \bar{\mu}_T)$ with $\bar{\mu}_T(s) = 1 - \mu_T(s)$ for all $s \in \sk$. These definitions adhere to standard fuzzy set theory, enabling the logic's language to be interpreted within this generalized structure without altering its core semantics.

\paragraph{Skills as a lattice}

Let $(L, \leq)$ be a lattice, defined as a partially ordered set where every two-element subset $\{x, y\} \subseteq L$ has a \emph{join} (\emph{supremum} or \emph{least upper bound}), denoted $x \sqcup y$, and a \emph{meet} (\emph{infimum} or \emph{greatest lower bound}), denoted $x \sqcap y$. A \emph{model over a lattice $(L, \leq)$} is a quadruple $(W, E, C, \beta)$, differing from the standard model introduced in Section~\ref{sec:models} in the following respects:
\begin{itemize}
\item The edge function $E: W \times W \to L$ assigns each pair of worlds an element in the lattice.
\item The capability function $C: \ag \to L$ assigns each agent an element of the lattice.
\end{itemize}
The lattice structure is incorporated into the semantics by reinterpreting the following operators:
\begingroup
\addtolength{\jot}{-.6ex}
\begin{alignat*}{3}
&M, w \models K_a \psi &\ \iff\ & \text{for all $u \in W$, if $C(a) \leq E(w, u)$, then $M, u \models \psi$}
\\
&M, w \models D_G \psi &\ \iff\ & \text{for all $u \in W$, if $\textstyle\bigsqcup_{a\in G} C(a) \leq E(w, u)$, then $M, u \models \psi$}
\\
&M, w \models F_G \psi &\ \iff\ & \text{for all $u \in W$, if $\textstyle\bigsqcap_{a\in G} C(a) \leq E(w, u)$, then $M, u \models \psi$}
\\
&M, w \models (+_S)_{a} \psi &\ \iff\ & (W, E, {C^{a \sqcup S}}, \beta), w \models \psi
\\
&M, w \models (-_S)_a \psi &\ \iff\ & (W, E, C^{a \sqcap S}, \beta), w \models \psi
\end{alignat*}
\endgroup
where:
\[
\begin{array}{l@{\qquad}l}
C^{a \sqcup S} (x) = \left \{ \begin{tabular}{ll} $C(a) \sqcup S$,& if $x = a$, \\ $C(x)$,& if $x \neq a$; \end{tabular} \right.
&
C^{a \sqcap S} (x) = \left \{ \begin{tabular}{ll} $C(a) \sqcap S$,& if $x = a$, \\ $C(x)$,& if $x \neq a$. \end{tabular} \right.
\end{array}
\]
The class of $\subseteq$-ordered skill sets, whether classical or fuzzy, constitutes a special case of a lattice. Each lattice element can be regarded as a skill set, with the $\leq$ order generalizing the subset relation, and the join and meet operations corresponding to union and intersection, respectively. Notably, a general lattice lacks a natural notion of complement unless it is a complemented lattice. Consequently, the semantics of $(-_S)_a \psi$ shifts here, utilizing $C^{a \sqcap S}$ as a generalization of $C^{a \cap S}$ rather than directly mirroring set difference.

\subsection{Enriching epistemic de re and de dicto}
\label{sec:de-re}

The distinction between epistemic \emph{de re} and \emph{de dicto} modalities, first articulated in \cite{Wright1951}, differentiates whether a modality pertains to a specific entity possessing or lacking a property (\emph{de re}) or to the truth or falsity of a proposition (\emph{de dicto}). For instance, in the sentence ``Ralph knows that someone is a spy,'' the \emph{de re} and \emph{de dicto} readings diverge based on whether ``someone'' refers to a specific individual known to Ralph. As noted in \cite{Quine1956}, this contrast becomes clearer in formal languages with quantifiers over terms. In epistemic logic, a \emph{de re} statement is typically expressed as ``there exists a term $x$ such that an agent knows or does not know that $x$ has or lacks a certain property'' (e.g., $\exists x K_a \phi (x)$ or $\exists x \neg K_a \phi (x)$). In contrast, a \emph{de dicto} statement takes the form ``an agent knows or does not know that there exists a term possessing or lacking a property'' (e.g., $K_a \exists x \phi (x)$ or $\neg K_a \exists x \phi (x)$).'' Technically, the \emph{de re} / \emph{de dicto} distinction hinges on the scope of epistemic modalities and quantifiers.

In dynamic epistemic logic, the distinction between \emph{knowing de re} and \emph{knowing de dicto} is enriched through the integration of quantifiers over update operations, encompassing both quantifiers over public announcements \cite{BBDHHL2008,ABDS2010} and those over skill modifications as introduced in this paper. This approach sharpens the differentiation between \emph{de re} and \emph{de dicto} knowledge while resonating with philosophical inquiries into \emph{knowing that} (propositional knowledge, linked to $K_a \exists x \phi (x)$) versus \emph{knowing how} (procedural or skill-based knowledge, linked to $\exists x K_a \phi (x)$).

The logics presented in this paper not only distinguish between \emph{de re} and \emph{de dicto} modalities but also identify two distinct types of \emph{de re} knowledge (cf. Group Announcement Logic \cite[Section~6]{ABDS2010}, which discusses only one type of \emph{de re} knowledge):
\begin{itemize}
\item \emph{Knowing de dicto}: ``Agent $a$ knows, with her current skills, that there exists a skill set $S$ such that, with $S$ in addition, $\phi$ holds.''

Formally: $(\forall u \in W) [C(a) \subseteq E(w,u) \Ra (\exists S \subseteq \sk)\ (W, E, C^{a \cup S}, \beta), u \models \phi]$.
\item \emph{Implicitly knowing de re}: ``There exists a skill set $S$ such that agent $a$, upon adding $S$ to her skill set, knows that $\phi$ holds.''

Formally: $(\exists S\subseteq\sk)(\forall u\in W)[C^{a \cup S}(a)\subseteq E(w,u) \Ra (W,E,C^{a \cup S},\beta),u\models \phi]$.
\item \emph{Explicitly knowing de re}: ``There exists a skill set $S$ such that agent $a$ knows, with her current skills, that with $S$ in addition, $\phi$ holds.''

Formally: $(\exists S\subseteq\sk)(\forall u\in W)[C(a)\subseteq E(w,u) \Ra (W,E,C^{a \cup S},\beta),u\models \phi]$.
\end{itemize}
The distinction between \emph{de dicto} and \emph{de re} knowledge remains evident, while the subtle difference between \emph{implicit} and \emph{explicit de re} knowledge lies in whether the skill set $S$ is part of the agent's current capabilities when formulating her knowledge.

These distinctions illuminate the intricate relationship between knowledge and capabilities in dynamic epistemic contexts, revealing subtle variations in how agents process information based on their skill sets and the form of their knowledge. All three types---de dicto, explicit de re, and implicit de re---are expressible within the formal languages introduced in this paper. Their representations are formalized as follows:

\begin{prop}\label{prop:de-re-de-dicto}\ 
\begin{enumerate}
\item\label{it:de-dicto} De dicto knowledge is expressed by the formula $K_a\diamondplus_a\phi$;
\item\label{it:knowable} Implicit de re knowledge is expressed by the formula $\diamondplus_a K_a\phi$;
\item\label{it:de-re} Explicit de re knowledge is expressed by the formula $(\equiv_a)_c\diamondplus_c K_a (\equiv_c)_a \phi$, where $c$ is an agent not occurring in $\phi$.
\end{enumerate}
\end{prop}
\begin{proof}
The validity of statements (\ref{it:de-dicto}) and (\ref{it:knowable}) follows directly from the semantics. The focus here is on statement (\ref{it:de-re}), where $c$ denotes an agent not appearing in $\phi$:
\[
\begin{array}[b]{cl}
& (\exists S\subseteq\sk)(\forall u\in W)\ C(a)\subseteq E(w,u) \Ra (W,E,C^{a \cup S},\beta),u\models \phi \\
\iff & (\exists S\subseteq\sk)(\forall u\in W)\ C(a)\subseteq E(w,u) \Ra (W,E,((C^{c\equiv a})^{c+S})^{a\equiv c},\beta),u\models \phi \\
\iff & (\exists S\subseteq\sk)(\forall u\in W)\ C(a)\subseteq E(w,u) \Ra (W,E,(C^{c\equiv a})^{c+S}, \beta), u\models (\equiv_c)_a\phi \\
\iff & (\exists S\subseteq\sk)(W,E,(C^{c\equiv a})^{c+S},\beta), w \models K_a (\equiv_c)_a \phi \\
\iff & (W,E,C^{c\equiv a},\beta),w\models \diamondplus_c K_a(\equiv_c)_a\phi \\
\iff & (W,E,C,\beta),w\models (\equiv_a)_c\diamondplus_c K_a(\equiv_c)_a\phi.
\end{array}
\qedhere
\]
\end{proof}

The statements (\ref{it:dicto1})--(\ref{it:exp-re2}) presented in Section~\ref{subsubsec:mc} illustrate the differences between de dicto, implicit de re, and explicit de re knowledge. Specifically, statements (\ref{it:dicto1})--(\ref{it:exp-re1}) demonstrate that implicit de re knowledge is distinct from both de dicto and explicit de re knowledge. Furthermore, statements (\ref{it:dicto2})--(\ref{it:exp-re2}) establish that explicit de re knowledge is distinct from both de dicto and implicit de re knowledge.

For simplicity, the definitions of \emph{de dicto knowledge}, \emph{implicit de re knowledge}, and \emph{explict de re knowledge} have been presented above primarily in terms of the individual knowledge operator $K_a$ and the quantifier $\boxplus_a$ over upskilling actions. These concepts can be readily extended to encompass:
\begin{itemize}
\item \emph{Group knowledge}, employing operators such as $C_G$, $D_G$, $E_G$ and $F_G$,
\item \emph{Quantifiers over downskilling and reskilling actions}, represented by $\boxminus_a$ and $\Box_a$, respectively.
\end{itemize}

For instance, the formula $D_G {\diamondplus_a} {\diamondminus_b} \phi$ expresses: ``It is distributed knowledge among group $G$ that, with the addition of certain skills by agent $a$, it becomes possible that, even after the loss of certain skills by agent $b$, $\phi$ remains true.'' This constitutes an \emph{epistemic de dicto} statement.
The formula $(\equiv_a)_c \Diamond_c K_a (\equiv_c)_a \phi$ (where $c$ does not occur in $\phi$) conveys: ``There exists a skill set such that agent $a$ knows, with precisely this skill set, that $\phi$ is true.'' This represents \emph{explicit de re knowledge}.
The formula $\Diamond_a K_a\phi$ indicates: ``There exists an update to agent $a$'s skill set through which she knows that $\phi$ is true.'' This exemplifies \emph{implicit de re knowledge}.

Nested quantifiers further enrich these distinctions. For example, the formula $F_{G}\,\diamondminus_{a_1}\diamondplus_{a_2}\Diamond_{a_3}\diamondminus_{a_4}\phi$ articulates an \emph{epistemic de dicto} statement, combining field knowledge ($F_G$) with a sequence of actions---downskilling ($\boxminus_{a_1}$), upskilling ($\boxplus_{a_2}$), reskilling ($\Box_{a_3}$) and further downskilling ($\boxminus_{a_4}$)---across multiple agents. 
Similarly, the formula $(\equiv_{d_1})_{c_1}\diamondplus_{c_1}(\equiv_{d_2})_{c_2}\diamondplus_{c_2}(\equiv_{d_3})_{c_3}\diamondplus_{c_3} E_{I}(\equiv_{c_1})_{d_1}(\equiv_{c_2})_{d_2}(\equiv_{c_3})_{d_3}\phi$ captures \emph{explicit de re knowledge}, involving nested quantifiers and multiple agents tied to mutual knowledge.
Likewise, the formula $\diamondplus_{b_1}\Diamond_{b_2}\diamondminus_{b_3}D_{H}\phi$ illustrates \emph{implicit de re knowledge}, integrating a sequence of updates with distributed knowledge.
In these examples, the agents $a_1, a_2, a_3, a_4, b_1, b_2, b_3, c_1, c_2, c_3, d_1, d_2, d_3$ are not constrained to be within or outside the groups $G$, $H$ or $I$. This flexibility enables broad applicability across diverse contexts and group dynamics, extending beyond a mere distinction between \emph{knowing that} and \emph{knowing how}.

\section{Complexity of Model Checking}
\label{sec:mc}

This section investigates the computational complexity of the model checking problem for the logics introduced in the previous section. The \emph{model checking problem} for a logic is to determine whether a given formula $\phi$ is true in a specified finite model $M$ at a designated world $w$---formally, whether $M, w \models \phi$.

\begin{conv}
\label{sec:inputsize}
The measure of the input is defined as follows. The \emph{length} of a formula $\phi$, denoted $|\phi|$, represents the number of symbols in $\phi$ (including brackets), consistent with \cite[Section 3.1]{FHMV1995}. More precisely, it is defined inductively based on the structure of $\phi$:
\begin{itemize}
	\item Atomic proposition $p$: $|p| = 1$;
	\item Negation $\neg\psi$: $| \neg \psi | = |\psi| +1$;
	\item Implication $(\psi \ra \chi)$: $| (\psi \ra \chi) | = |\psi| + |\chi| + 3$;
	\item Individual knowledge $K_a \psi$: $|K_a\psi| = |\psi| + 2$;
	\item Group knowledge: $|C_G\psi| = |\psi| + 2|G| + 2$, with analogous definitions for $D_G \psi$, $E_G \psi$ and $F_G\psi$; e.g., $| (p \ra C_{\{a, b, c\}} q) | = 13$;
	\item Update modality: $|(+_S)_a \psi| = 2|S| + |\psi| + 5$, similarly for $(-_S)_a \psi$ and $(=_S)_a \psi$, and $|(\equiv_b)_a \psi| = |\psi| + 5$;
	\item Quantifier: $|\boxplus_a \psi| = |\psi| + 2$, likewise for $\boxminus_a \psi$ and $\Box_a \psi$.
\end{itemize}

The \emph{size} of a finite model $M = (W,E,C,\beta)$, denoted $|M|$, is  the sum of the following components:
\begin{itemize}
\item $|W|$: the cardinality of the domain;
\item $|E|$: the size of $E$, which comprises triples $(w, u, S)$ where $w, u \in W$ and $S \subseteq \sk$, measured by the number of symbols required to represent this set;
\item $|C|$: the size of $C$, comprising pairs $(a, S)$ where $a \in \ag$ and $S \subseteq \sk$, measured by the total number of symbols required to represent it;%
\footnote{Theoretically, $C$ maps a possibly infinite set of agents to skill sets, each of which may also be infinite. However, practical model checking necessitates a finite input. Thus, the set of agents and the cardinality of each skill set must be finite and restricted to those occurring in the formula under consideration.}
\item $|\beta|$: the size of $\beta$, comprising pairs $(w, \Phi)$ where $w \in W$ and $\Phi \subseteq \pr$, determined by the number of symbols needed to represent this set.
\end{itemize}

For a formula $\phi$ and a model $M$ (with a designated world $w$), the \emph{size of the input} is defined as $|\phi| + |M| + 3$.
\end{conv}

\subsection{Model checking for logics without quantifiers: in P}

This section begins by presenting a polynomial-time algorithm to determine the truth of classical epistemic formulas in a specified world within a given model, addressing the model checking problem for $\l$. The algorithm is then extended to accommodate group knowledge modalities, establishing that the model checking problem for $\l_{CDEF}$ lies within the complexity class P. This upper bound is then broadened to encompass update modalities, covering the model checking problems for $\l_{CDEF+-=\equiv}$ and all its sublogics.

\subsubsection{Model checking in \l}

Given a model $M = (W, E, C, \beta)$, a world $w \in W$ and a formula $\phi$, the task is to decide whether $M, w \models \phi$. To this end, an algorithm (Algorithm~\ref{alg:val}) is introduced for computing $Val(M,\phi)$, the \emph{truth set} of $\phi$ in $M$, i.e., $\{ x \in W \mid M, x \models \phi\}$. The question of whether $M, w \models \phi$ holds is thus reduced to testing membership in $Val (M, \phi)$, which requires at most $|W|$ steps beyond the computation of $Val (M, \phi)$.
	
\begin{algorithm}
\caption{Function $Val (M, \phi)$: Computing the Truth Set for Basic Formulas}\label{alg:val}
\footnotesize
\begin{algorithmic}[1]
\Require model $M = (W, E, C, \beta)$ and formula $\phi$
\Ensure $\{ x \mid M, x \models \phi \}$
\Initialize{$tmpVal \gets \emptyset$}

\If{$\phi=p$} \Return $\{ x \in W \mid p \in \beta(x)\}$

\ElsIf{$\phi=\neg\psi$} \Return $W \setminus Val(M,\psi)$

\ElsIf{$\phi=\psi\to\chi$}
	\State\Return $(W \setminus Val(M,\psi)) \cup Val(M,\chi)$

\ElsIf{$\phi=K_a\psi$}
\ForAll{$x \in W$}
\Initialize{$n \gets \True$}
\ForAll{$y \in W$}
\If{{\footnotesize $C(a) \subseteq E (x, y)$} \algand {\footnotesize~$y \notin Val (M, \psi)$}} $~~~~n \gets \False$
\EndIf
\EndFor
\If{$n = \True$} $tmpVal \gets tmpVal \cup \{x\}$
\EndIf
\EndFor
\State \Return $tmpVal$ \Comment{This returns $\{x\in W \mid \forall y \in W: C(a)\subseteq E (x, y) \Rightarrow y \in Val (M,\psi)\}$}
\EndIf
\end{algorithmic}
\end{algorithm}
	
It is not hard to verify that $Val (M, \phi)$ accurately represents the set of worlds in $M$ where $\phi$ is true. In particular, for the $K_a$ operator, the following equivalence is established:
\[
\begin{array}{llll}
M, w \models K_a \psi & \iff & \forall y \in W: C(a) \subseteq E (w, y) \Rightarrow M, y \models \psi \\
& \iff & \forall y \in W: C(a) \subseteq E (w, y) \Rightarrow y \in Val (M, \psi) \hfill \text{(by IH)}\\
& \iff & w \in \{ x \in W \mid \forall y \in W: C(a) \subseteq E (x, y) \Rightarrow y \in Val(M,\psi)\} 
\end{array}
\]

The computation of $Val(M, \phi)$ operates in polynomial time. For the case of $K_a \psi$---the most computationally intensive scenario---two nested loops iterate over $W$, with the check $C(a) \subseteq E (x, y)$ requiring at most $|C| \cdot |E|$ steps, and the membership test $y \notin Val(M,\psi)$ (assuming $Val(M,\psi)$ is precomputed) taking at most $|W|$ steps. Thus, this case has a time complexity of at most $|W|^2 \cdot (|C| \cdot |E| + |W|)$. The algorithm recursively computes $Val(M, \phi)$ for subformulas of $\phi$, with the maximum recursion depth bounded by $|\phi|$, the length of $\phi$. Consequently, the total time complexity for computing $Val(M, \phi)$ is $|W|^2 \cdot (|C| \cdot |E| + |W|) \cdot |\phi|$. Relative to the input size $n = |\phi| + |M| + 3$, where $|M| = |W| + |E| + |C| + |\beta|$, this is bounded by $O(n^5)$, leading to the following lemma:

\begin{lem}\label{lem:mc-el}
The model checking problem for \l is in P.
\end{lem}

\subsubsection{Model checking group knowledge}

Building on the previous result, this section extends the analysis to incorporate group knowledge scenarios. To support this extension, a definition and supporting propositions are introduced below.

\begin{defi}\label{def:trans-e}
For a formula $\phi$, let $\ab_\phi = \{G \mid \text{``$E_G$'' or ``$C_G$'' appears in $\phi$}\}$. For a model $M = (W, E, C, \beta)$,
\begin{itemize}
\item For all worlds $w, u \in W$, define $E_\phi (w, u) = E (w, u) \cup \{G \in A_\phi \mid (\exists a\in G)\ C(a) \subseteq E(w,u) \}$,
\item For all worlds $w, u \in W$, define $E_{\phi}^+(w,u) = E_\phi(w,u) \cup \{G\in A_\phi \mid (\exists n \geq 1)(\exists w_0,\dots,w_n\in W)\ w_0=w \text{ and } w_n=u \text{ and } G\in \bigcap_{0\leq i< n}E_\phi(w_i, w_{i+1})\}$,
\end{itemize}
where it is assumed, without loss of generality, that $A_\phi \cap \ag = \emptyset$. The notation $M_\phi^+$ is used to denote $(W, E^+_\phi, C, \beta)$.
\end{defi}

It should be noted that this definition involves a notational simplification by treating groups of agents as skills. To maintain formal rigor, a mapping can be established from each element of $A_\phi$ to a distinct new skill in $\sk$.

\begin{prop}
For any model $M$ and any formula $\phi$, $M^+_\phi$ is a model.
\end{prop}

\begin{lem}\label{lem:trans-e}
Given formulas $\phi$ and $\chi$, a group $G$, a model $M$ and a world $w$ of $M$:
\begin{enumerate}
\item\label{it:trans-e1} $M, w \models \phi$ if and only if $M^+_{\chi}, w \models \phi$;
\item\label{it:trans-e3} If ``$C_G$'' appears in $\chi$, then $M, w \models C_G\phi$ if and only if $M, u \models \phi$ for every world $u$ such that $G\in E^+_{\chi}(w, u)$.
\end{enumerate} 
\end{lem}
\begin{proof}
(\ref{it:trans-e1}) For any agent $a$, formula $\chi$ and worlds $w, u$, it holds that $C(a) \subseteq E(w, u)$ iff $C(a) \subseteq E_\chi(w, u)$ iff $C(a) \subseteq E^+_{\chi}(w, u)$. This follows because $E(w, u) \subseteq E_\chi(w, u) \subseteq E_\chi^+(w, u)$, and $C(a)$ contains only individual skills, not groups from $A_\chi$, which are disjoint from $\ag$ by Definition~\ref{def:trans-e}. Consequently, the satisfaction of any formula $\phi$ remains unchanged between $(M, w)$ and $(M_\chi^+, w)$.
	
(\ref{it:trans-e3}) The proof proceeds by establishing the base case for $E_G \phi$ and then extending it to $C_G \phi$:
\[
\begin{array}{lll}
& M, w \models E_G \phi \\
\iff & \text{for any $a \in G$, $M, w\models K_a\phi$} & \\
\iff & \text{for any $a\in G$ and $u\in W$, $C(a)\subseteq E(w, u)$ implies $M,u \models \phi$} & \\
\iff & \text{for any $u\in W$ and $a\in G$, $C(a)\subseteq E(w, u)$ implies $M,u \models \phi$} & \\
\iff & \text{for any $u\in W$, $M, u \models \phi$ if $C(a) \subseteq E(w,u)$ for some $a\in G$} & \\
\iff & \text{for any $u\in W$, $G\in E_\chi(w,u)$ implies $M, u \models \phi$} & \\
\iff & \text{$M, u\models \phi$ for any world $u$ such that $G\in E_\chi(w, u)$} &\\
\end{array}
\]
and so
\[
\begin{array}{lll}
& M, w \models C_G \phi \\
\iff & \text{$M, w\models E^k_G\phi$ for all $k\in \mbN^+$} & \\
\iff & \text{$M, u\models \phi$ for any world $u$ such that $G\in E^+_{\chi}(w, u)$.} & (*)
\end{array}
\]
To justify $(*)$, suppose $M, w \not \models E^n_G \phi$ for some $n\in\mbN^+$. Then by induction on $n$, there  exist worlds $w_1, \dots, w_n \in W$ such that $M, w_n \not\models \phi$ and $G \in E_\chi(w, w_1) \cap \bigcap_{1\leq i< n} E_\chi(w_i, w_{i+1})$. Hence $M, w_n\not\models\phi$ and $G\in E^+_\chi(w, w_n)$. Suppose $M, u\not\models \phi$ for a world $u$ such that $G\in E^+_\chi(w, u)$, w.l.o.g, assume that there exist $w_0,\dots, w_n\in W$ such that $w_0=w$, $w_n=u$, $G\in \bigcap_{0\leq i< n} E_\chi(w_i, w_{i+1})$ and $M, w_n \not\models \phi$. Applying the above result $n$ times, it follows that $M, w \not \models E^n_G \phi$.
\end{proof}

\begin{lem}\label{thm:complexity}
The model checking problem for $\l_{CDEF}$, and thus for all its sublogics, is in P.
\end{lem}
\begin{proof}
To establish this result, it suffices to provide a polynomial-time algorithm for formulas of the form $C_G \psi$, $D_G \psi$, $E_G \psi$ and $F_G \psi$. The extended algorithm is detailed in Algorithm~\ref{alg:val-dcx}.
\begin{algorithm}
\caption{Function $Val(M,\phi)$ Extended: Cases with Group Knowledge Operators}\label{alg:val-dcx}
\footnotesize
\begin{algorithmic}[1]
	\Initialize{$temVal \gets \emptyset$}
	\If{...} ... \Comment{Same as in Algorithm~\ref{alg:val}}
	\ElsIf{$\phi = C_G\psi$}
	\ForAll{$x \in W$}
	\Initialize{$n \gets \True$}
	\ForAll{$y \in W$}
	%   \State {{\bf initialize} $z=E(t,u)$}
	\If{$G\in E^+_{\phi} (x, y)$ \algand $y \notin Val (M, \psi)$}
		\State $n \gets \False$
	%      \Comment{Checking $C(a)\subseteq z$ costs at most $|\ab|$ steps}
	\EndIf
	\EndFor
	\If{$n = \True$}
		\State $tmpVal \gets tmpVal \cup \{x\}$
	\EndIf
	\EndFor
	\State \Return $tmpVal$ \quad \Comment{Returns $\{x \in W\mid \forall y \in W: G \in E^+_{\phi}(x, y) \Rightarrow y \in Val(M,\psi)\}$}
	\ElsIf{$\phi=D_G\psi$}
	\ForAll{$x \in W$}
	\Initialize{$n \gets \True$}
	\ForAll{$y \in W$}
	%   \State {{\bf initialize} $z=E(t,u)$}
	\If{$\bigcup_{a\in G} C(a) \subseteq E(x,y)$ \algand $y \notin Val(M,\psi)$}
		\State {$n \gets \False$}
	%      \Comment{Checking $C(a)\subseteq z$ costs at most $|\ab|$ steps}
	\EndIf
	\EndFor
	\If{$n = \True$}
		\State $tmpVal \gets tmpVal \cup \{x\}$
	\EndIf
	\EndFor
	\State \Return $tmpVal$ \quad \Comment{Returns $\{x \in W\mid \forall y \in W: \bigcup_{a\in G} C(a) \subseteq E (x, y) \Rightarrow y \in Val(M,\psi)\}$}

	\ElsIf{$\phi=E_G\psi$}
	\ForAll{$x \in W$}
	\State {{\bf initialize} $n \gets \True$}
	\ForAll{$y \in W$}
	%\State {{\bf initialize} $z=E(t,u)$}
	\If{$G\in E_\phi (x, y)$ \algand $y \notin Val(M,\psi)$}
		\State {$n \gets \False$}
	% \Comment{Checking $C(a)\subseteq z$ costs at most $|\ab|$ steps}
	\EndIf
	\EndFor
	\If{$n = \True$} {$tmpVal \gets tmpVal \cup \{ x \}$}
	\EndIf
	\EndFor
	\State \Return $tmpVal$ \quad \Comment{Returns $\{x \in W \mid \forall y \in W: G \in E_\phi (x, y) \Ra y \in Val(M, \psi) \}$}

	\ElsIf{$\phi = F_G\psi$}
	\ForAll{$x \in W$}
	\Initialize{$n \gets \True$}
	\ForAll{$y \in W$}
	%   \State {{\bf initialize} $z=E(t,u)$}
	\If{$\bigcap_{a\in G}C(a)\subseteq E (x, y)$ \algand $y \notin Val(M,\psi)$}
		\State {$n \gets \False$}
	%      \Comment{Checking $C(a)\subseteq z$ costs at most $|\ab|$ steps}
	\EndIf
	\EndFor
	\If{$n = \True$} {$tmpVal \gets tmpVal \cup \{ x \}$}
	\EndIf
	\EndFor
	\State \Return $tmpVal$ \quad \Comment{Returns $\{x \in W\mid \forall y \in W: \bigcap_{a\in G} C(a) \subseteq E (x, y) \Rightarrow y \in Val(M,\psi)\}$}
	\EndIf
\end{algorithmic}
\end{algorithm}
As in the proof of Lemma \ref{lem:mc-el}, checking $C(a) \subseteq E(t, u)$ costs at most $|C| \cdot |E|$ steps, here we furthermore need to calculate the cost caused by group knowledge operators.
	
For $D_G$ and $F_G$, notice that the number of agents in any group $G$ that appears in $\phi$ is less than $|\phi|$, so checking $\bigcup_{a \in G} C(a)\subseteq E(t, u)$ and $\bigcap_{a \in G} C(a)\subseteq E(t,u)$ costs at most $|C| \cdot |E| \cdot |\phi|$ steps. Thus for the logics extended with these operators, the complexity for model checking would not go beyond P.
	
For $E_G$ and $C_G$, the computation of $E_\phi(w, u)$ and $E^+_{\phi}(w, u)$ must be polynomial. By Definition~\ref{def:trans-e} and Lemma~\ref{lem:trans-e}, computing the set $A_\phi$ costs at most $|\phi|$ steps, since there are at most $|\phi|$ modalities appearing in $\phi$; moreover, the size of $G$ is at most $|\phi|$. To compute $E_\phi(w, u)$ for any given $w$ and $u$, it costs at most $|E|$ steps to compute $E(w, u)$ and at most $|\phi| ^ 2 \cdot |C| \cdot |E|$ steps to check for every $G \in A_\phi$ whether there exists $a \in G$ such that $C(a) \subseteq E(w, u)$. So the cost of computing the whole function $E_\phi$ can be finished in at most $|W|^2 \cdot (|E| + |\phi| ^ 2 \cdot |C| \cdot |E|)$ steps. Now consider the computation of $E^+_{\phi}$. Assume that there is a string that describes $E_\phi$, then check for all pairs $(x, y), (y, z)\in W^2$ whether there exists a ``$G$'' appearing in $\phi$ such that $G\in E_\phi (x, y) \cap E_\phi (y, z)$; if it is, add $G$ as a member of $E_\phi (x, z)$. Keep doing this until $E_\phi$ does not change any more. Every round of checking takes at most $2|\phi|^2 \cdot |W|^3$ steps, and it will be stable in at most $|\phi| \cdot |W|^2$ rounds. Then the function $E^+_{\phi}$ is achieved. Every membership checking for $G \in E^+_{\phi}(w, v)$ is finished in polynomial steps. So the whole process remains in P.
\end{proof}

\subsubsection{Model checking formulas with update modalities}
\label{sec:MC-U}

This section addresses the model checking problem for formulas involving update modalities. Consider a model $M = (W, E, C, \beta)$, a world $w \in W$, and the formulas $(+_S)_a \psi$, $(-_S)_a \psi$, $(=_S)_a \psi$ and $(\equiv_b)_a \psi$. According to the semantics in Definition~\ref{def:semantics},
$$M, w \models (+_S)_a \psi \iff M^{a \cup S}, w \models \psi$$
where $M^{a \cup S} = (W, E, C^{a \cup S}, \beta)$, and $C^{a \cup S}$ updates $C(a)$ to $C(a) \cup S$ while leaving other agents' skill sets unchanged. Consequently, verifying whether \( M, w \models (+_S)_a \psi \) reduces to checking \( M^{a \cup S}, w \models \psi \), effectively eliminating the outermost update modality. An algorithm that invokes the existing model checking procedure (e.g., Algorithm~\ref{alg:val-dcx}) on $M^{a \cup S}$ and $\psi$ operates efficiently: constructing $M^{a \cup S}$ from $M$ requires at most $|C(a)| \cdot |S|$ steps to compute the union, where $|S| \leq |\phi|$ since $S$ is specified in the formula, and $\psi$ is a subformula of the original input. Given that model checking for $\logic{CDEF}$ is in P (Lemma~\ref{thm:complexity}), this additional step introduces only polynomial overhead, maintaining the total complexity within polynomial bounds.

The cases for $(-_S)_a \psi$, $(=_S)_a \psi$, and $(\equiv_b)_a \psi$ proceed similarly, each requiring a distinct model transformation:
\begin{itemize}
\item For $(-_S)_a \psi$, the model becomes $M^{a \cap S} = (W, E, C^{a \cap S}, \beta)$, where $C^{a \cap S}(a) = C(a) \cap S$, computed in at most $|C(a)| \cdot |S|$ steps.
\item For $(=_S)_a \psi$, the model is $M^{a = S} = (W, E, C^{a = S}, \beta)$, where $C^{a = S}(a) = S$, requiring at most $|S|$ steps to assign $S$ directly.
\item For $(\equiv_b)_a \psi$, the model is $M^{a \equiv b} = (W, E, C^{a \equiv b}, \beta)$, where $C^{a \equiv b}(a) = C(b)$, taking at most $|C(b)|$ steps to copy $C(b)$.
\end{itemize}
Each transformation modifies $C$ in polynomial time relative to the input size, as $|S| \leq |\phi|$ (since $S$ is specified in the formula), and $|C(a)|$ and $|C(b)|$ are bounded by the model's finite representation. The subsequent recursive check on the transformed model and subformula $\psi$, using the procedure for $\l_{CDEF}$ (e.g., Algorithm~\ref{alg:val-dcx}), remains in P per Lemma~\ref{thm:complexity}. Consequently, the total complexity for these cases remains polynomial, establishing the following theorem:

\begin{thm}
The model checking problems for $\l_{CDEF+-=\equiv}$ and all its sublogics are in P.
\end{thm}

\subsection{Model checking quantified formulas: PSPACE complete}

The PSPACE hardness of model checking for logics with quantified modalities---specifically $\logic{\boxplus}$, $\logic{\boxminus}$ and $\logic{\Box}$---is achieved by a polynomial-time reduction from the problem of undirected edge geography (UEG), a variant of generalized geography \cite{Schaefer1978,LS1980} known to be PSPACE complete for determining a winning strategy, as established in \cite{FSU1993}. The PSPACE upper bound is established via a polynomial-space algorithm, extending the algorithms from the prior section.

Consider an undirected graph $G = (D, R)$, where $D$ is a finite nonempty set of nodes, and $R \subseteq D \times D$ is a symmetric and irreflexive relation. For a node $d \in D$, the pair $(G, d)$ is termed a \emph{rooted undirected graph}.
The undirected edge geography (UEG) game on $(G, d)$ is a two-player game processing as follows:
\begin{enumerate}
\item Player I's Move: Player I starts by selecting edge $\{ d, d_1 \} \in R$. If no such edge exists, the game ends and Player II wins as Player I cannot make a valid move.
\item Player II's Move: After Player I''s move selecting an edge $\{ d_i, d_{i+1} \}$, Player II must choose an edge $\{ d_{i+1}, d_{i+2} \}$ that has not been chosen in previous moves. If Player II cannot make such a move, the game ends and Player I wins.
\item Alternating Turns: After Player II's move selecting an edge $\{ d_j, d_{j+1} \}$, it is Player I's turn again to choose an edge $\{ d_{j+1}, d_{j+2} \}$ not previously chosen. If Player I cannot make such a move, the game ends and Player II wins.
\item Repeat Step 2: The game continues by alternating turns following the process described in step 2.
\end{enumerate}

Alternatively, UEG game on $(G, d)$  can be recursively defined by modifying the graph after each move:
\begin{itemize}
\item The current player selects an edge $\{ d, d' \} \in R$; if no such edge exists, the player loses, and the game terminates.
\item Upon a successful move, the game proceeds with the opposing player on the updated graph $(G', d')$, where $G' = (D, R \setminus \{ \{ d, d' \} \} )$.
\end{itemize}
Play alternates between Player I (starting at $d$) and Player II until a player cannot move.

As an example,
$G_0 = ( \{d_1, d_2, d_3, d_4 \}, \{ \{d_1, d_3\}, \{d_1, d_4\}, \{d_2, d_4\}, \{d_3, d_4\} \} )$
represents an undirected graph. An illustration of $G_0$ is given in Figure~\ref{fig:UEG}.

\begin{figure}
\centering
\begin{tikzpicture}[modal, node distance=1.5cm and 1.8cm, world/.append style={minimum width=.8cm, minimum height=.5cm}, inner sep=.5ex]
\scriptsize
\node[world] (d1) {$d_1$};
\node[world] (d2) [right=of d1] {$d_2$};
\node[world] (d3) [below=of d1] {$d_3$};
\node[world] (d4) [right=of d3] {$d_4$};

\path (d1) edge (d3);
\path (d1) edge (d4);
\path (d2) edge (d4);
\path (d3) edge (d4);
\end{tikzpicture}
\caption{\label{fig:UEG}Illustration of
$G_0 = ( \{d_1, d_2, d_3, d_4 \}, \{ \{d_1, d_3\}, \{d_1, d_4\}, \{d_2, d_4\}, \{d_3, d_4\} \} )$.}
\end{figure}

The \emph{UEG problem} asks whether Player I has a winning strategy, i.e., can force a win regardless of Player II's moves.

\begin{defi}[Induced model]
Given an undirected graph $G = (D, R)$, assign:
\begin{itemize}
\item To each edge $\{x, y\} \in R$, a unique epistemic skill $s_{\{x, y\}} \in \sk$, such that $s_{\{x', y'\}} \neq s_{\{x'', y''\}}$ for distinct unordered pairs $\{x', y'\}$ and $\{x'', y''\}$,
\item To each node $x \in D$, a unique atomic proposition $p_x \in \pr$, such that $p_{x'} \neq p_{x''}$ for distinct nodes $x'$ and $x''$.
\end{itemize}

The \emph{induced model} $M_G$ is defined as the tuple $(D, E, C, \beta)$, where:
\begin{itemize}
\item $E: D \times D \to \wp(\sk)$, with $E(x, y) = \{ s_{\{x, y\}} \}$ if $\{x, y\} \in R$, and $E (x, y) = \emptyset$ otherwise;
\item $C: \ag \to \wp(\sk)$, with $C(a) = \emptyset$ for all $a \in \ag$;
\item $\beta: D \to \wp(\pr)$, with $\beta(x) = \{p_x\}$ for each $x \in D$.
\end{itemize}
\end{defi}

The model $M_G$ is well-defined and succinctly encodes the structure and properties of $G$. The size of $E$ is $O(|D|^2)$, reflecting pairwise edge relations, while the size of $\beta$ is $O(|D|)$, corresponding to one proposition per node. The size of $C$ is $O(|R|)$, given that only a limited number of agents are relevant, as clarified in the definition of the size of the input and the subsequent definition.

As an example, the model $M_{G_0}$ induced from the undirected graph $G_0$ satisfies the following:
\begin{itemize}
\item $W = \{d_1, d_2, d_3, d_4\}$;
\item For all $m, k \in \{1, 2, 3, 4\}$, $E(d_m, d_k)=\{s_{\{d_m, d_k\}}\}$ whenever $\{d_m, d_k \} \in G_0$;
\item $C(a_1) = C(a_2) = C(a_3) = C(a_4) = \emptyset$\ \  ($a_i$ is the player who performs the $i$-th move)j
\item $V(d_j) = \{ p_{j}\} $ for $1\leq j\leq 4$.
\end{itemize}
An illustration is given in Figure~\ref{fig:UEG-model}.

\begin{figure}
\centering
\parbox{.56\textwidth}{
\begin{tikzpicture}[modal, node distance=1.6cm and 3.6cm, world/.append style={minimum width=1cm, minimum height=.3cm}, inner sep=.5ex]
\scriptsize
\node[world, ellipse split] (d1) {$d_1$ \nodepart{lower} $p_1$};
\node[world, ellipse split] (d2) [right=of d1] {$d_2$ \nodepart{lower} $p_2$};
\node[world, ellipse split] (d3) [below=of d1] {$d_3$ \nodepart{lower} $p_3$};
\node[world, ellipse split] (d4) [right=of d3] {$d_4$ \nodepart{lower} $p_4$};

\path (d1) edge node[left] {$s_{\{d_1, d_3\}}$} (d3);
\path (d1) edge node[above right] {$s_{\{d_1, d_4\}}$} (d4);
\path (d2) edge node[right] {$s_{\{d_2, d_4\}}$} (d4);
\path (d3) edge node[below] {$s_{\{d_3, d_4\}}$} (d4);
\end{tikzpicture}
}
\quad\parbox{.2\textwidth}{\small
$\begin{array}{l}
C(a_1) = \emptyset\\
C(a_2) = \emptyset\\
C(a_3) = \emptyset\\
C(a_4) = \emptyset\\
\end{array}$}
\caption{\label{fig:UEG-model}Illustration of the induced model $M_{G_0}$, with $G_0$ illustrated in Figure~\ref{fig:UEG}.}
\end{figure}

\begin{defi}[Induced formula]\label{def:indf}
Given an undirected graph $G = (D, R)$, let $n$ be the smallest even positive integer greater than or equal to $|R|$. Select distinct agents $a_1, a_2, \ldots, a_n \in \ag$. For each $i$ where $1 \leq i \leq n$, define:
$$\begin{array}{lll}
\psi_{i} & := & \neg K_{a_i} \bot \wedge \bigvee_{x\in D} K_{a_i} p_x, \\
\chi_{i} & := & \bigvee_{x, y\in D,\, x \neq y,\, 1 \leq j < i} (p_x \wedge \hat K_{a_j} p_y \wedge K_{a_i} p_y), \\[.5ex]
\multicolumn{3}{l}{\text{and for even $i$:}}\\
\phi_{i} & :=  & \begin{array}[t]{l}
\diamondplus_{a_1} ( \psi_{1} \wedge \neg \chi_{1} \wedge K_{a_1}\boxplus_{a_2}( \neg \psi_{2} \vee \chi_{2} \vee 
\\
\qquad\hat{K}_{a_2}\diamondplus_{a_3}(\psi_{3} \wedge \neg\chi_{3} \wedge K_{a_3}\boxplus_{a_4}(\neg \psi_{4} \vee \chi_{4} \vee
\\
\qquad\qquad\hat{K}_{a_4}\diamondplus_{a_5}(\psi_{5} \wedge \neg \chi_{5} \wedge K_{a_5}\boxplus_{a_6} (\neg \psi_{6} \vee \chi_{6} \vee
\\
\qquad\qquad\qquad\dots
\\
\qquad\qquad\qquad\qquad\hat{K}_{a_{i-2}}\diamondplus_{a_{i-1}}(\psi_{i-1} \wedge \neg \chi_{i-1} \wedge K_{a_{i-1}}\boxplus_{a_i}(\neg\psi_{i} \vee \chi_{i}))\cdots)))))).
\end{array}
\end{array}$$
where $\hat{K}_a$ is the dual of $K_a$ (i.e., $\hat{K}_a \psi = \neg K_a \neg \psi$). The \emph{induced formula} $\phi_G$ for $G$ is defined as $\phi_n$.
\end{defi}

Table~\ref{tbl:inF} details the process of generating the induced formula $\phi_{G_0}$ from the undirected edge graph given above.

\begin{table}
\caption{\label{tbl:inF}Process of generating the induced formula $\phi_{G_0}$, where $G_0$ is given in Figure~\ref{fig:UEG}.}
\medskip
\begin{tblr}{
	hline{1,9} = {1pt, solid},
	column{1,2} = {leftsep=0pt, rightsep=0pt},
	rows = {mode=math},
}
\psi_{i} & := & \neg K_{a_i}\bot \wedge (K_{a_i} p_1 \vee K_{a_i} p_2 \vee K_{a_i} p_3 \vee K_{a_i} p_4) \\
\chi_{1} & := & \bot\\
\chi_{2} & := & \begin{array}[t]{@{}l@{}}
(p_2 \wedge \hat K_{a_1} p_1 \wedge K_{a_2} p_1) \vee (p_3 \wedge \hat K_{a_1} p_1 \wedge K_{a_2} p_1) \vee (p_4 \wedge \hat K_{a_1} p_1 \wedge K_{a_2} p_1) \\
\vee (p_1 \wedge \hat K_{a_1} p_2 \wedge K_{a_2} p_2) \vee (p_3 \wedge \hat K_{a_1} p_2 \wedge K_{a_2} p_2) \vee (p_4 \wedge \hat K_{a_1} p_2 \wedge K_{a_2} p_2)\\
\vee (p_1 \wedge \hat K_{a_1} p_3 \wedge K_{a_2} p_3) \vee (p_2 \wedge \hat K_{a_1} p_3 \wedge K_{a_2} p_3) \vee (p_4 \wedge \hat K_{a_1} p_3 \wedge K_{a_2} p_3)\\
\vee (p_1 \wedge \hat K_{a_1} p_4 \wedge K_{a_2} p_4) \vee (p_2 \wedge \hat K_{a_1} p_4 \wedge K_{a_2} p_4) \vee (p_3 \wedge \hat K_{a_1} p_4 \wedge K_{a_2} p_4)\\
\end{array}\\
& & \text{I.e., } \chi_{2} = \bigvee_{x \neq y \in \{1,2,3,4\}} (p_x \wedge \hat K_{a_1} p_y \wedge K_{a_2} p_y) \\
\chi_{3} & := & \begin{array}[t]{@{}l@{}}
\bigvee_{x \neq y \in \{1,2,3,4\}} (p_x \wedge \hat K_{a_1} p_y \wedge K_{a_2} p_y) \\
\vee \bigvee_{x \neq y \in \{1,2,3,4\}} (p_x \wedge \hat K_{a_1} p_y \wedge K_{a_3} p_y)\\
\vee \bigvee_{x \neq y \in \{1,2,3,4\}} (p_x \wedge \hat K_{a_2} p_y \wedge K_{a_3} p_y)\\
\end{array}
\\
& & \text{I.e., } \chi_{3} = \bigvee_{1 \leq j < i \leq 3} \Big( \bigvee_{x \neq y \in \{1,2,3,4\}} (p_x \wedge \hat K_{a_j} p_y \wedge K_{a_i} p_y) \Big)
\\
\chi_{4} & := & \bigvee_{1 \leq j < i \leq 4} \Big( \bigvee_{x \neq y \in \{1,2,3,4\}} (p_x \wedge \hat K_{a_j} p_y \wedge K_{a_i} p_y) \Big)
\\
\phi_{G_0} & :=  &
\diamondplus_{a_1} ( \psi_{1} \wedge \neg \chi_{1} \wedge K_{a_1}\boxplus_{a_2}( \neg \psi_{2} \vee \chi_{2} \vee \hat{K}_{a_2}\diamondplus_{a_3}(\psi_{3} \wedge \neg\chi_{3} \wedge K_{a_3}\boxplus_{a_4}(\neg \psi_{4} \vee \chi_{4} ))))
\end{tblr}
\end{table}

To elucidate the induced formula $\phi_G$ for an undirected graph $G = (D, R)$, consider its role in encoding the UEG game. Each agent $a_i$ corresponds to the player making the $i$-th move, with $i$ ranging from 1 to $n$, where $n$ is the smallest even integer at least $|R|$. The subformulas are interpreted as follows:
\begin{itemize}
\item $\psi_i = \neg K_{a_i} \bot \wedge \bigvee_{x \in D} K_{a_i} p_x$ ensures that player $a_i$, at the $i$-th move, selects exactly one edge from the current node. In $M_G$, where $C(a_i)$ starts as $\emptyset$, $\neg K_{a_i} \bot$ holds trivially, and $\bigvee_{x \in D} K_{a_i} p_x$ requires $a_i$ to ``know'' one node's proposition.
\item $\chi_i = \bigvee_{x, y \in D,\, x \neq y,\, 1 \leq j < i} (p_x \wedge \hat{K}_{a_j} p_y \wedge K_{a_i} p_y)$ identifies invalid moves by detecting if $a_i$'s chosen edge (leading to $y$) was previously selected by some $a_j$ (where $j < i$), as $\hat{K}_{a_j} p_y$ indicates the edge $\{x, y\}$ was selected earlier.
\item The conjunction $\psi_i \wedge \neg \chi_i$ enforces a valid move: $a_i$ picks a new, unvisited edge from the current node.
\end{itemize}

As for complexity, the length of $\psi_{i}$ is in $O(|D|)$, due to the disjunction over $|D|$ nodes. The length of $\chi_i$ is $O(|D|^2 \cdot i)$, as it involves pairs $(x, y)$ and prior moves $j < i$; since $i \leq n = O(|R|)$, this is $O(|D|^2 \cdot |R|)$. The formula $\phi_G = \phi_n$ has $n = O(|R|)$ nested modalities, each adding $\psi_i$ and $\chi_i$, yielding a total length of $O(|R| \cdot |D|^2 \cdot |R|) = O(|D|^2 \cdot |R|^2)$.

The structure of $\phi_G$ mirrors UEG gameplay:
\begin{itemize}
\item $\diamondplus_{a_1}$ allows player $a_1$ (Player I) to upskill, adding a skill (edge) to $C(a_1)$, representing a move choice;
\item $\psi_1 \wedge \neg \chi_1$ ensures $a_1$ selects a new edge from the root $d$, valid at the game's start;
\item $K_{a_1} \boxplus_{a_2}$ asserts that, after $a_1$'s move, for all possible upskillings by $a_2$ (Player II), the subformula $\neg \psi_2 \vee \chi_2 \vee \hat{K}_{a_2} \diamondplus_{a_3} (\cdots)$ holds:
	\begin{itemize}
	\item $\neg \psi_2$ means $a_2$ cannot select a node (no edges remain), ending the game with $a_1$ winning.
	\item $\chi_2$ indicates $a_2$ repeats an edge (invalid), also favoring $a_1$.
	\item $\hat{K}_{a_2} \diamondplus_{a_3} (\psi_3 \wedge \neg \chi_3 \wedge \cdots)$ allows $a_2$ a valid move, shifting play to $a_3$ (Player I again), recursively continuing the game.
	\end{itemize}
\end{itemize}
This nested, alternating structure captures the strategic interplay of UEG, where each move constrains the opponent's options, modeling game states as nodes in $M_G$ and moves as skill updates, within a framework tailored to $\lang_{\boxplus}$'s quantified modalities.

A lemma is now presented that establishes a formal correspondence between the undirected edge geography problem and the epistemic logics developed herein, specifically those incorporating quantified modalities.

\begin{lem}
For any rooted undirected graph $(G, d)$, Player I has a wining strategy in the UEG game on $(G,d)$, if and only if $M_{G}, d \models \phi_{G}$.
\end{lem}
\begin{proof}
The proof proceeds by induction on $|R|$, the number of edges in $G$.

\underline{Base case: $|R| = 0$}. Here, $n=2$, and $R = \emptyset$, so no edges exist. Player I loses immediately, unable to move from $d$. In the induced model $M_G = (D, E, C, \beta)$, $E(x, y) = \emptyset$ for all $x, y\in D$. We show $M_G, d \not \models \phi_G$, where $\phi_G = \phi_{2} = \diamondplus_{a_1}(\psi_{1}\land\neg\chi_{1}\land K_{a_1}\boxplus_{a_2}(\neg\psi_{2}\lor\chi_{2}))$, with $\psi_1=\neg K_{a_1}\bot \land \bigvee_{x\in D}K_{a_1} p_x$, $\chi_1 = \bot$, $\psi_2 = \neg K_{a_2}\bot \land \bigvee_{x\in D}K_{a_2} p_x$, and $\chi_2 = \bigvee_{x \neq y \in D} (p_x \wedge \hat K_{a_1} p_y \wedge K_{a_2} p_y)$. For any finite nonempty $S \subseteq \sk$, consider the model $M' = (D, E, C^{a_1+S}, \beta)$. Since $E(d, y) = \emptyset$ for all $y$, $M', d \models K_{a_1} \bot$ (no worlds are accessible), so $M', d \not\models \psi_1$. Thus, $M', d \not\models \psi_1 \wedge \neg \chi_1 \wedge K_{a_1} \boxplus_{a_2} (\neg \psi_2 \vee \chi_2)$. As $S$ is arbitrary, $M_G, d \not\models \phi_G$.

\underline{Base case: $|R| = 1$}. Let $R = \{ \{ d, d'\} \}$, so $n=2$. (The case where $d$ is not a vertex of $R$, i.e., $R = \{ \{ x, y \} \}$ with $x \neq d$ and $y \neq d$, follows similarly to the case where $|R| = 0$.) Player I wins by choosing $\{d, d'\}$, leaving Player II with no moves. In $M_G = (D, E, C, \beta)$, $E(d, d') = E(d', d) = \{s_{\{d, d'\}}\}$, and $E(x, y)=\emptyset$ otherwise. We show $M_G, d \models \phi_G$, with $\phi_G = \phi_2$ as above. Take $S=\{s_{\{d, d'\}}\}$ and $M' = (D, E, C^{a_1+S}, \beta)$:
\begin{itemize}
\item $\psi_1 = \neg K_{a_1}\bot \land \bigvee_{x\in D}K_{a_1} p_x$
\hfill ($M', d \models \psi_1$, for $M', d \models \neg K_{a_1} \bot \wedge K_{a_1} p_{d'}$)
\item $\chi_1 = \bot$
\hfill ($M', d \models \neg \chi_1$)
\item $\psi_2 = \neg K_{a_2} \bot \land \bigvee_{x\in D}K_{a_2} p_x$
\item $\chi_2 = (p_d \land \hat{K}_{a_1} p_{d'} \land K_{a_2} p_{d'}) \vee (p_{d'} \wedge \hat{K}_{a_1} p_{d} \land K_{a_2} p_{d}) \vee \bigvee_{x \neq y \in D \setminus \{d, d'\}} (p_x \wedge \hat K_{a_1} p_y \wedge K_{a_2} p_y)$.
\end{itemize}
For any finite nonempty $S' \subseteq \sk$, let $M'' = (D, E, (C^{a_1+S})^{a_2+S'}, \beta)$, we have one of the following cases:
\begin{enumerate}[label={(\arabic*)}]
\item $S' \not\subseteq S$, then $\forall x \in D$, $(C^{a_1+S})^{a_2+S'}(a_2)\nsubseteq E(d,x)$, hence $M'', d' \models \neg \psi_2$, for $M'', d'\models K_{a_2}\bot$.
\item $S'\subseteq S$, then $M'', d' \models p_{d'} \wedge \hat K_{a_1} p_{d} \wedge K_{a_2} p_d$. Thus, $M'', d'\models \chi_2$ for its middle disjunct is satisfied. 
\end{enumerate}
In both case $M'', d' \models \neg \psi_2 \vee \chi_2$, and so $M', d' \models \boxplus_{a_2} (\neg \psi_2 \vee \chi_2)$, and $M', d \models K_{a_2} \boxplus_{a_2} (\neg \psi_2 \vee \chi_2)$. Together with the verifications above, we have $M_G, d \models \phi_G$.

\underline{Inductive step: $|R| = k \geq 1$}. Assume the lemma holds for all graphs with fewer than $k$ edges.
Left to right. Suppose Player I has a winning strategy, choosing $\{d, d'\}$ as the first move. For the induced model $M_G = (D, E, C, \beta)$, we show $M_G, d \models \phi_G$, where $\phi_G= \diamondplus_{a_1}(\psi_1 \land \neg \chi_1 \land K_{a_1}\phi_{G,\boxplus_{a_2}})$, in which $\phi_{G,\boxplus_{a_2}}$ is the subformula of $\phi_G$ beginning with $\boxplus_{a_2}$ (see Definition~\ref{def:indf}). Take $S=\{s_{\{d,d'\}}\}$ and $M' = (D, E, C^{a_1+S}, \beta)$:
\begin{itemize}
\item $\psi_1=\neg K_{a_1}\bot \land \bigvee_{x\in D}K_{a_1}p_x$
\hfill ($M', d \models \psi_1$, for $M',d\models \neg K_{a_1}\bot \land K_{a_1} p_{d'}$)
\item $\chi_1=\bot$
\hfill ($M', d \models \neg \chi_1$)
\end{itemize}
Now we show $M',d\models K_{a_1}\phi_{G,\boxplus_{a_2}}$; namely, $M',d'\models \phi_{G,\boxplus_{a_2}}$, where $\phi_{G,\boxplus_{a_2}}=\boxplus_{a_2}(\neg \psi_2 \lor \chi_2 \lor \hat{K}_{a_2}\phi_{G,\diamondplus_{a_3}})$ in which $\phi_{G,\diamondplus_{a_3}}$ is the subformula of $\phi_G$ beginning wtih $\diamondplus_{a_3}$. For any finite nonempty $S'\subseteq\sk$, let $M'' = (D, E, (C^{a_1+S})^{a_2+S'}, \beta)$, and it suffices to show that 
\begin{align}\tag{\dag}
M'', d' \models \neg \psi_2 \lor \chi_2 \lor \hat{K}_{a_2}\phi_{G,\diamondplus_{a_3}},
\end{align}
where $\psi_2=\neg K_{a_2}\bot \land \bigvee_{x\in D}K_{a_2}p_x$ and $\chi_2 = \bigvee_{x\neq y\in D}  (p_x \land \hat{K}_{a_1} p_y \land K_{a_2} p_y)$. Consider the possible cases:
\begin{enumerate}[label=(\arabic*)]
\item There does not exist $x \in D$ such that $S' \subseteq E(d',x)$, or
\item There exists $d'' \in D$ such that $S' \subseteq E(d', d'')$ (note that $S'$ must be singleton).
\end{enumerate}
In case (1), $M'',d'\models K_{a_2}\bot$, so $M'',d'\models \neg\psi_2$, hence $(\dag)$ holds. In case (2), 
Player I has a winning strategy in the continued game on $(G_2, d'')$ with $G_2=(D, R \setminus \{ \{d, d'\}, \{d',d''\} \})$ (note that $d''$ cannot be $d$ or $d'$). It suffices to show the following result:
\begin{align}\tag{\ddag}
M'',d''\models \phi_{G,\diamondplus_{a_3}} \iff M_{G_2},d''\models \phi_{G_2}.
\end{align}
Since $M_{G_2}, d'' \models \phi_{G_2}$ holds by the induction hypothesis, by $(\ddag)$, we have $M'',d''\models \phi_{G,\diamondplus_{a_3}}$. This makes the rightmost disjunct of $(\dag)$ true in $M'', d'$, and completes the whole proof.

Let $M_{G_2}=(D,E_2,C,\beta)$. To see $(\ddag)$,
\begin{enumerate}
\item[]$M'',d''\models \phi_{G,\diamondplus_{a_3}}$, i.e., $(D, E, (C^{a_1+S})^{a_2+S'}, \beta),d''\models \phi_{G,\diamondplus_{a_3}}$
\item[iff]
$(D,E_2,(C^{a_1+S})^{a_2+S'},\beta),d''\models \phi'_{G,\diamondplus_{a_3}}$, where $\phi'_{G,\diamondplus_{a_3}}$ is adapted from $\phi_{G,\diamondplus_{a_3}}$ by the following:
\begin{itemize}
\item Delete all occurrences of $\bigvee_{x\neq y\in D}  (p_x \land \hat{K}_{a_1} p_y \land K_{a_i} p_y)$ from $\phi_{G,\diamondplus_{a_3}}$
\item Delete all occurrences of $\bigvee_{x\neq y\in D}  (p_x \land \hat{K}_{a_2} p_y \land K_{a_i} p_y)$ from $\phi_{G,\diamondplus_{a_3}}$
\end{itemize}
(This equivalence holds since $E_2(d, d') = E_2(d', d) = \emptyset$, which implies that any formulas $\hat K_{a_1} \phi$ and $\hat K_{a_2} \phi$ are false in any world $x$ of model $(D,E_2,C',\beta)$, where $C'$ is any capability function updated from $(C^{a_1+S})^{a_2+S'}$ without changing the capabilities of $a_1$ and $a_2$.)
\item[iff] 
$(D,E_2,C,\beta),d''\models \phi''_{G,\diamondplus_{a_3}}$, where $\phi''_{G,\diamondplus_{a_3}}$ a variant of $\phi'_{G,\diamondplus_{a_3}}$ by replacing any $a_{i+2}$ with $a_i$,\\
 (This holds since $(C^{a_1+S})^{a_2+S_2}(a_{i+2})=C(a_i)=\emptyset$; note that $a_1$ and $a_2$ does not exist in $\phi'_{G,\diamondplus_{a_3}}$.)% \hfill (2)
\item[iff] 
$M_{G_2},d''\models\phi_{G_2}$, i.e.,
$(D,E_2,C,\beta),d''\models \phi_{G_2}$.
\hfill (since $\phi_{G_2}=\phi''_{G,\diamondplus_{a_3}}$)% \hfill (3)
\end{enumerate}

Right to left: Suppose Player I has no winning strategy in the UEG game on $(G, d)$, where $G = (D, R)$. We must show that $M_G, d \not\models \phi_G$, with $M_G$ be $(D,E,C,\beta)$ as the induced model. Since Player I lacks a winning strategy, one of two cases holds:
\begin{enumerate}[label=(\alph*)]
\item No $x\in D$ exists such that $\{d, x\} \in R$, so Player I loses immediately.
\item For every $d' \in D \setminus \{d\}$ with $\{d, d'\} \in R$, Player I has no winning strategy after choosing $\{d, d'\}$.
\end{enumerate}

Case (a): If $R$ contains no edges incident to $d$, then $E(d, x)=\emptyset$ for all $x \in D$. We get $M_G, d \not \models \phi_G$ in a way similar to the case when $|R| = 0$.

Case (b): Assume $\{d, d'\} \in R$ exists, but no initial move $\{d, d'\}$ yields a winning strategy for Player I. For any finite nonempty $S \subseteq \sk$, consider $M' = (D, E, C^{a_1+S}, \beta)$ and two subcases:
\begin{enumerate}[label=(\arabic*)]
\item $S\nsubseteq E(d,x)$ for all $x\in D$,
\item Theres exists $d' \in D$ such that $S \subseteq E(d, d')$ (note that $d'$ cannot be $d$).
\end{enumerate}
We need to show $M_G, d \not \models \phi_G$ where $\phi_G$ is given in Definition~\ref{def:indf}. Let $M' = (D, E, C^{a_1+S}, \beta)$. In subcase (1), since $M' ,d \models K_{a_1} \bot$, $M', d \not \models \psi_1$ (with $\psi_1=\neg K_{a_1}\bot \land \bigvee_{x\in D}K_{a_1}p_x$), and so $M ,d \not\models \phi_G$. 

In subcase (2) (under the case (b)), there must exist $d''\in D\setminus\{d,d'\}$ such that Player I does not have a winning strategy in the game on $(G_2, d'')$ where $G_2 = (D, R \setminus \{ \{d, d'\}, \{d',d''\} \})$; for otherwise Player I has a winning strategy (this is also the case when there is no such a $d''$), leading to a contradiction. Let $S'=\{s_{\{d', d''\}} \}$, then $S'\subseteq E(d',d'')$. Let $M'' = (D, E, (C^{a_1+S})^{a_2+S'}, \beta)$. It suffices to show that
\begin{align}\tag{*}
M'', d' \not \models \neg \psi_2 \lor \chi_2 \lor \hat{K}_{a_2}\phi_{G,\diamondplus_{a_3}},
 \end{align}
Consider $\psi_2=\neg K_{a_2}\bot \land \bigvee_{x\in D}K_{a_2}p_x$. Since $M'',d'\models \neg K_{a_2}\bot \land K_{a_2} p_{d''}$, we have $M'',d' \not\models \neg \psi_2$.
As for $\chi_2 = \bigvee_{x\neq y\in D}  (p_x \land \hat{K}_{a_1} p_y \land K_{a_2} p_y)$, since $M'',d'\models \hat{K}_{a_1}p_y \land K_{a_2} p_y$ implies $y=d\neq d''=y$, we have $M'',d' \not\models \chi_2$.
Finally we show that $M'', d' \not \models \hat{K}_{a_2}\phi_{G,\diamondplus_{a_3}}$. Since there is exact one $x \in D$ (which must be $d''$ by the definition of $S'$) such that $S' \subseteq E(d',x)$, it suffices to prove $M'',d'' \not\models \phi_{G,\diamondplus_{a_3}}$. Note that $(\ddag)$ from the proof of the converse direction can also be shown here, it suffices to show that $M_{G_2}, d'' \not\models \phi_{G_2}$, and this holds by the induction hypothesis.
\end{proof}

\begin{cor}\label{lem:red-gg2cua}
The undirected edge geography (UEG) problem is polynomial-time reducible to the model checking problem for $\l_{\boxplus}$.
\end{cor}

\begin{rem}\label{rem:red}
The reduction outlined in the preceding lemma relies solely on the modalities $\boxplus$ and $\diamondplus$. An alternative reduction can be formulated using only $\Box$ and $\Diamond$, mirroring the original structure but substituting $\boxplus$ with $\Box$. Similarly, a reduction employing exclusively $\boxminus$ and $\diamondminus$ is viable, replacing $\boxplus$ with $\boxminus$ and adjusting the skill assignment in the induced model $M_G$ such that $C(a_i) = \{ s_{\{w, v\}} \mid w, v \in D \}$ for each agent $a_i$. Consequently, the model checking problems for any logic (extending \l) incorporating at least one of the quantifying modalities $\boxplus$, $\boxminus$, $\Box$, $\diamondplus$, $\diamondminus$, or $\Diamond$ are PSPACE hard, even when additional modalities---such as group knowledge operators and update modalities---are excluded from the logic.
\end{rem}

\begin{lem}\label{thm:complexity-cua}
The model checking problem for $\l_{CDEF+-=\equiv\boxplus\boxminus\Box}$ is in PSPACE.
\end{lem}
\begin{proof}
Given Algorithm~\ref{alg:val} for model checking in the basic logic \l, Algorithm~\ref{alg:val-dcx} for group knowledge operators, and an argument for reducing update modalities in Section~\ref{sec:MC-U}, it suffices to extend with a polynomial-space algorithm for formulas of the form $\boxplus_a \psi$, $\boxminus_a \psi$, and $\Box_a \psi$. This extension is provided in Algorithm~\ref{alg:val-box}.

\begin{algorithm}
\caption{Function $Val((W,E,C,\beta),\phi)$ Extended: Cases with Quantifiers}
\label{alg:val-box}
	\footnotesize
	\begin{algorithmic}[1]
	\Initialize{$temVal \gets \emptyset$}
	\Initialize{$S_1 \gets (\bigcup_{w,v\in W}E(w,v))\cup(\bigcup_{a\text{ appears in }\phi}C(a))$} 
	\Initialize{$S_2 \gets S_1\cup \{s\}$} \Comment{Here $s\in\sk$ is new for $S_1$}
	\If{...} ... \Comment{Same as in Algorithm~\ref{alg:val-dcx}}
	\ElsIf{$\phi=\boxplus_a\psi$}
	\ForAll{$t\in W$}
	\Initialize{$n \gets \True$}
	\ForAll{$S \subseteq S_2$}
	\If{$t \notin Val((W,E,C^{a \cup S},\beta),\psi)$} {$n \gets \False$}
	\EndIf
	\EndFor
	\If{$n = \True$} $tmpVal \gets tmpVal \cup \{t\}$
	\EndIf
	\EndFor
	\State \Return $tmpVal$ \quad \Comment{Returns $\{t \in W \mid \forall S \subseteq S_1: t\in Val((W,E,C^{a \cup S},\beta),\psi)\}$}
	\ElsIf{$\phi=\boxminus_a\psi$}
	\ForAll{$t \in W$}
	\Initialize{$n \gets \True$}
	\ForAll{$S \subseteq S_2$}
	\If{$t\notin Val((W,E,C^{a \setminus S},\beta),\psi)$} {$n \gets \False$}
	\EndIf
	\EndFor
	\If{$n = \True$} $tmpVal \gets tmpVal \cup \{t\}$
	\EndIf
	\EndFor
	\State \Return $tmpVal$ \quad \Comment{Returns $\{t\in W\mid \forall S\subseteq S_1: t\in Val((W,E,C^{a \setminus S},\beta),\psi)\}$}
	\ElsIf{$\phi=\Box_a\psi$}
	\ForAll{$t\in W$}
	\Initialize{$n \gets \True$}
	\ForAll{$S \subseteq S_2$}
	\If{$t \notin Val((W,E,C^{a=S},\beta),\psi)$} {$n \gets \False$}
	\EndIf
	\EndFor
	\If{$n = \True$} $tmpVal \gets tmpVal \cup \{t\}$
	\EndIf
	\EndFor
	\State \Return $tmpVal$ \quad \Comment{Returns $\{t\in W\mid \forall S\subseteq S_1: t\in Val((W,E,C^{a=S},\beta),\psi)\}$}
	\EndIf
	\end{algorithmic}
\end{algorithm}

To confirm the space complexity, consider the resource usage of $Val((W, E, C, \beta), \phi)$. The space cost of checking $Val((W,E,C,\beta),\phi)$ is in $O(|M| \cdot |\phi|)$, polynomial in the input size.  Since Algorithm~\ref{alg:val-dcx} is in PSPACE and the extension for $\boxplus$, $\boxminus$, and $\Box$ operates in polynomial space, the model checking problem for $\lang_{CDEF+-=\equiv\boxplus\boxminus\Box}$ is in PSPACE.
\end{proof}

The following result is derived from Corollary~\ref{lem:red-gg2cua} and Remark~\ref{rem:red}, which together establish a polynomial-time reduction from the PSPACE-complete undirected edge geography (UEG) problem to the model checking problems for $\logic{\boxplus}$, $\logic{\boxminus}$, and $\logic{\Box}$, and from Lemma~\ref{thm:complexity-cua}, which demonstrates that the model checking problem for $\logic{CDEF+-=\equiv\boxplus\boxminus\Box}$ is in PSPACE.

\begin{thm}\label{lem:mc-lcua}
The model checking problem for any logic that extends the base logic \l by including at least one quantifier modality from $\{\boxplus, \boxminus, \Box\}$ is PSPACE complete.
\end{thm}

\section{Complexity of the Satisfiability Problem}
\label{sec:sat}

This section examines the computational complexity of the satisfiability problem for some of the logics introduced in earlier sections. The \emph{satisfiability problem} for a logic is about determining whether a given formula $\phi$ is satisfiable---that is, whether there exists a model $M$ and a world $w$ within that model such that $M, w \models \phi$. The \emph{size} of the input formula $\phi$ is defined as its \emph{length}, denoted $|\phi|$, which is defined in the previous section.

\subsection{Satisfiability for logics without common knowledge, update and quantifying modalities: PSPACE complete}
\label{sec:sat-pspace}

The complexity of satisfiability for the logics under consideration is established through reductions to and from known results, summarized in Figure~\ref{fig:sat-psp}. These logics exclude common knowledge ($C_G$), update modalities ($(+_S)_a$, $(-_S)_a$, $(=_S)_a$, $(\equiv_b)_a$), and quantifying modalities ($\boxplus_a$, $\boxminus_a$, $\Box_a$), focusing on logics based on subsets of $\lang_{CDEF}$, such as $\l$, $\logic{D}$ and $\logic{DEF}$.

The results will be shown by reductions to and from known complexity results, and are summarized in Figure~\ref{fig:sat-psp}.

\begin{figure}[htbp]
\centering
\begin{tikzpicture}[
	modal,
	node distance=1cm and 2.3cm,
	world/.style={ellipse, draw, minimum width=1cm, minimum height=.5cm, inner sep=1ex},
 	rectworld/.style={rectangle, draw, minimum width=2cm, minimum height=.5cm, inner sep=1ex},
  	font=\footnotesize
]
\node[rectworld] (w1) {$\dfrac{\kbone}{\text{PSPACE complete}}$};
\node[world] (w2) [right=of w1] {\l};
\node[world] (w3) [right=of w2] {\logic{DEF}};
\node[world] (w4) [below=of w2] {\logic{D}};
\node[rectworld] (w5) [left=of w4] {$\dfrac{\kdn \ (n \geq 1)}{\text{PSPACE complete}}$};
\path[->, dashed] (w1) edge node[above, font=\tiny] {PTIME} node[below, font=\tiny] {(Lemma~\ref{lem:sat-KBtoL})} (w2);
\path[->, solid] (w2) edge (w3);
\path[->, dashed] (w3) edge node[above, sloped, font=\tiny] {PTIME} node[below, sloped, font=\tiny] {(Lemma~\ref{lem:red-toD})} (w4);
\path[->, dashed] (w4) edge node[above, font=\tiny] {PTIME} node[below, font=\tiny] {(Lemma~\ref{lem:sat-DtoKD})} (w5);
\end{tikzpicture}
\caption{Roadmap of proofs for the complexity of satisfiability problems for logics between $\l$ and $\logic{DEF}$. Logics under study are in elliptical frames, while known PSPACE-complete satisfiability problems are in rectangular frames. A solid arrow from one logic to another represents the satisfiability problem for the former logic as a subproblem of the satisfiability problem for the latter. A dashed arrow labeled ``PTIME'' from one logic to another indicates a polynomial-time reduction from the satisfiability problem for the former to that for the latter. References: \kdn from \cite[Section 3.5]{FHMV1995} (subscript denotes the number of agents); \kbone is folklore, with a proof in \cite{Sahlqvist1975} (named ``KB,'' citing a 1992 manuscript).
}\label{fig:sat-psp}
\end{figure}

\subsubsection{Reduction from \texorpdfstring{\kbone}{KB1} to \texorpdfstring{\l}{L}}

The satisfiability of any \lang-formula $\phi$ involving only one agent (let it be $a \in \ag$, the language hereafter referred to as ``single-agent \lang'') is shown to be equivalent in the logic \l and in \kbone, the classical mono-modal logic over symmetric frames. This equivalence is formalized in Lemma~\ref{lem:sat-KBtoL}. The satisfiability problem for \kbone is known to be PSPACE complete, as established in \cite{Sahlqvist1975} (denoted ``KB'' therein, with a proof attributed to a 1992 manuscript). Consequently, the satisfiability problem for \l is PSPACE hard.

Recall that an (epistemic) Kripke model is a triple $(W, R, V)$, where $W$ is a nonempty set of worlds, $R: \ag \to \wp(W \times W)$ assigns every agent a binary relation on $W$, and $V : W \to \wp(\pr)$ is a valuation. For a single-agent \lang-formula $K_a \phi$ , $M, w \models K_a \phi$ in a Kripke model $M = (W, R, V)$ if, for all $u \in W$, $(w, u) \in R(a)$ implies $M, u \models \phi$. A Kripke model $(W, R, V)$ is called \emph{symmetric} if $R$ is symmetric for all $a \in \ag$.

\begin{lem}\label{lem:sat-KBtoL}
\begin{enumerate}
\item Given a single-agent \lang-formula $\phi$, $\phi$ is \l-satisfiable if and only if $\phi$ is \kbone-satisfiable.
\item The satisfiability problem for \textrm{KB}$_1$ is polynomial-time reducible to that for $\l$.
\end{enumerate}
\end{lem}
\begin{proof}
(1) From left to right. Suppose $\phi$ is satisfied at a world $w$ in a model $M = (W, E, C, \beta)$. Construct a \kbone model $N = (W, R, V)$ where $R(a) = \{ (x, y) \in W \times W \mid C(a) \subseteq E(x, y) \}$ and $V = \beta$. By induction on the structure of $\lang$-formulas containing no agents other than $a$, it holds that for any such formula $\psi$ and any $x \in W$, $M, x \models_{\l} \psi$ iff $N, x \models_{\kbone} \psi$. Thus, $N, w \models_{\kbone} \phi$.

From right to left. Suppose $\phi$ is satisfied at a world $w$ in a \kbone model $N = (W, R, V)$, where $R$ is symmetric. For every agent $a$, let $s_a$ be a fixed skill uniquely associated with $a$, i.e., $s_a = s_b$ iff $a = b$ (this is possible since both the agent set $\ag$ and the skill set \sk are countably infinite). Construct a model $M = (W, E, C, \beta)$ where:
\begin{itemize}
\item $E: W \times W \to \wp(\ag)$ where for any $x, y \in W$, $E (x, y) = \{s_a \in \ag \mid (x, y) \in R(a) \}$,
\item $C: \ag \to \wp(\sk)$, with $C(b) = \{ s_b \}$ for all $b \in \ag$,
\item $\beta = V$.
\end{itemize}
Since $R$ is symmetric, $M$ is indeed a model. For any $x, y \in W$ and $b \in \ag$, $(x, y) \in R(b)$ iff $C(b) \subseteq E(x, y)$. By induction on $\lang$-formulas with only agent $a$, for any such $\psi$ and $x \in W$, $N, x \models_{\kbone} \psi$ iff $M, x \models_\l \psi$. Thus, $M, w \models_\l \phi$.

(2) Since \kbone is based on a mono-modal language that is a sublanguage of that of \l, following statement (1), satisfiability in \kbone reduces to that in \l by inclusion.
\end{proof}

\subsubsection{Reduction from \texorpdfstring{\ld}{LD} to \texorpdfstring{\kdn}{KDn}}

A transformation is proposed to rewrite any \langd-formula, satisfiable in the logic \ld, into an \langd-formula satisfiable in \kdn, the multi-agent epistemic logic with distributed knowledge. The complexity of the satisfiability problem for \kdn is known to be PSPACE complete \cite[Section 3.5]{FHMV1995}. Recall that \kdn employs classical Kripke semantics, where, for a Kripke model $N = (W, R, V)$ and world $w \in W$:
\[\begin{tabular}{l@{\ $\iff$\ }l}
$N, w \models_\kdn K_a \psi$ & for all $u \in W$, $(w,u) \in R(a) $ implies $N, u \models_\kdn \psi$ \\
$N, w \models_\kdn D_G\phi$ & for all $u \in W$, $(w, u) \in \bigcap_{a \in G} R(a) $ implies $N, u \models_\kdn \psi$.
\end{tabular}\]

The key distinction between \ld and \kdn is that \ld requires symmetry of relations, whereas \kdn does not. Our reduction transforms an \ld-formula into a \kdn-formula that enforces symmetry syntactically. 
Ideally, we would add the characterization scheme of symmetry as conjuncts to the original formula, but this would require infinitely many conjuncts. To ensure a finite formula—specifically, one of polynomial length relative to the original formula for a polynomial-time reduction—we focus on the \emph{closure} of the original formula, defined as follows:

\begin{defi}[Closure of a formula]
\label{def:closure}
For any formula $\phi$ in any language, the \emph{closure of $\phi$}, denoted $\cl(\phi)$, is the set $\{\neg \psi, \psi \mid \psi \text{ is subformula of } \phi \} \cup \{ \top, \bot \}$.
\end{defi}

\begin{defi}\label{def:rewrite-DtoKD}
Given an $\lang_D$-formula $\phi$, fix a fresh agent $c$ not appearing in $\phi$ which will be used to reduce the number of conjuncts. Define $\rho'(\phi)$ as the $\lang_{D}$-formula obtained by applying the following steps sequentially:
\begin{enumerate}
\item For each agent $a \in \ag$ where $a \neq c$, replace every occurrence of $K_a$ with $D_{\{a, c\}}$;
\item For each group $G \in \gr$, replace every occurrence $D_G$ with $D_{G \cup \{c\}}$.
\end{enumerate}
Define $\rho(\phi)$ as the $\lang_{D}$-formula:
\[ \rho'(\phi) \wedge \textstyle\bigwedge_{0 \leq i \leq |\phi|} K^{i}_c \big(\bigwedge_{\chi \in \mu(\phi)} \chi \big), \]
where $c$ is the fixed fresh agent, $K^0_c \chi := \chi$, $K^n_c \chi := K_c K^{n-1}_c \chi$ (for $n \geq 1$), and $\mu(\phi)$ is the set of formulas comprising, for all $\psi \in \cl(\phi)$, $a$ appearing in $\phi$ or $a=c$, and $G$ appearing in $\phi$:
	\begin{itemize}
	\item $(\rho'(\psi) \ra K_a \neg K_a \neg \rho'(\psi)) \wedge (\neg K_a \neg K_a \rho'(\psi) \ra \rho'(\psi))$, 
	\item $(\rho'(\psi) \ra D_G \neg D_G \neg \rho'(\psi)) \wedge (\neg D_G \neg D_G \rho'(\psi) \ra \rho'(\psi))$,
	\item $D_{\{a, c\}} \rho'(\psi) \lra K_a \rho'(\psi)$ and $D_{G \cup \{c\} } \rho'(\psi) \lra D_G \rho'(\psi)$.
	\end{itemize}
It follows that both $\rho (\phi)$ and $\rho' (\phi)$ are \langd-formulas if $\phi$ is.
\end{defi}

\begin{lem}\label{lem:sat-DtoKD}
\begin{enumerate}
\item Given an \langd-formula $\phi$, $\phi$ is \ld-satisfiable if and only if $\rho(\phi)$ is \kdn-satisfiable;
\item The satisfiability problem for $\l_D$ is polynomial-time reducible to that for \textrm{K}$^D_n$.
\end{enumerate}
\end{lem}
\begin{proof}
(1)
From left to right. Suppose $\phi$ is satisfied at a world $w$ in a model $M = (W, E, C, \beta)$. It can be shown by induction on $\phi$ that $M^{c=\emptyset}, w \models_\ld \rho (\phi)$: just to observe that for any $u, v \in W$, any agent $a$ and any $G$ appearing in $\phi$, $C(a) = C^{c=\emptyset}(c) \cup C^{c=\emptyset}(a)$ (hence $M, w \models_\ld K_a \psi \iff M^{c = \emptyset}, w \models_\ld D_{\{c, a\}} \psi$ for any $\psi$ such that $M, w \models_\ld \psi \iff M^{c = \emptyset}, w \models_\ld \psi$) and $\bigcup_{b \in G}C(b) = C^{c=\emptyset}(c) \cup \bigcup_{b \in G}C^{c=\emptyset}(b)$ (hence $M, w \models_\ld D_G \psi \iff M^{c = \emptyset}, w \models_\ld D_{G \cup \{c\}} \psi$ for any $\psi$ such that $M, w \models_\ld \psi \iff M^{c = \emptyset}, w \models_\ld \psi$), and that $M^{c = \emptyset}, w \models \bigwedge_{0 \leq i \leq |\phi|} K^{i}_c \big(\bigwedge_{\chi \in \mu(\phi)} \chi \big)$.
Let $N = (W, R, V)$ be a Kripke model such that $V = \beta$ and for every $a \in \ag$, $R(a) = \{ (x, y) \in W \times W \mid C^{c = \emptyset}(a) \subseteq E(x, y)\}$. For any $u, v \in W$ and $G \in \gr$, it follows that $(u, v) \in \bigcap_{a\in G} R(a)$ iff $\bigcup_{a \in G} C^{c=\emptyset}(a) \subseteq E(u, v)$. By induction, it can be shown that for any $\lang_D$-formula $\psi$ and any $x \in W$, $M^{c=\emptyset}, x \models_\ld \psi$ iff $N, x \models_\kdn \psi$. Thus, $N, w \models_\kdn \rho(\phi)$, and so $\rho(\phi)$ is \kdn-satisfiable.

From right to left. Suppose that $\rho(\phi)$ is satisfied at a world $w$ of a model $N = (W, R, V)$, i.e., $N, w \models_\kdn \rho(\phi)$. Define $W_0 = \{ (u, G) \mid G \in \gr,\, u \in W \text{ and } (w, u) \in R^+(c) \} \cup \{(w, \{c\})\}$, where $R^+_c$ is the transitive closure of $R(c)$. Let $W_1$ be the set of finite sequences of elements of $W_0$ starting with $(w, \{c\})$. An element $\sigma$ of $W_1$ is of the form $\langle (w, \{c\}), (w_1, G_1), \dots, (w_n, G_n) \rangle$. The first element of the tail of $\sigma$, i.e., $w_n$, which is a world, is denoted $\tail (\sigma)$.
Construct a model $M = (W_1, E, C, \beta)$, where:%
\footnote{\label{ft}Agents are treated as skills for convenience, which is permissible since both $\ag$ and $\sk$ are countably infinite. Alternatively, this can be achieved by associating each agent $a \in \ag$ with a unique skill $s_a \in \sk$, as used in the proof of Lemma~\ref{lem:sat-KBtoL}.}
\begin{itemize}
\item $E: W_1 \times W_1 \to \wp(\sk)$ where for any $\sigma, \sigma' \in W$,
\[
E(\sigma, \sigma') = \left\{
\begin{array}{ll}
G, & \text{if $(\dag_1)$ and $(\dag_2)$,}
\\
\emptyset, & \text{otherwise;}
\end{array}
\right.
\]
\begin{enumerate}
\item[$(\dag_1)$] Either $\sigma$ extends $\sigma'$ with $(\tail (\sigma),G)$ or $\sigma'$ extends $\sigma$ with $(\tail(\sigma),G)$;
\item[$(\dag_2)$] Either $(\tail(\sigma), \tail (\sigma')) \in \bigcap_{a\in G} R(a)$ or $(\tail (\sigma'), \tail (\sigma)) \in \bigcap_{a\in G} R(a)$;
\end{enumerate}
\item $C: \ag \to \wp(\sk)$, with $C(a) = \{ a \}$ for all $a \in \ag$;
\item $\beta: W_1 \to \wp(\pr)$ is defined as $\beta(\sigma) = V(\tail(\sigma))$ for any $\sigma \in W_1$.
\end{itemize}

By induction on $\psi \in \cl(\phi)$, for $\sigma \in W_1$ of length $n$ and $\psi$ of modal depth $k$ where $n + k \leq |\phi| $, it holds that $N, \tail(\sigma) \models_\kdn \rho'(\psi) \iff M, \sigma \models_\ld \psi$. Consequently, since $N, w \models_{\kdn} \rho(\phi)$ and $\rho(\phi)$ includes $\rho'(\phi)$, it follows that $M, \langle (w, {c}) \rangle \models_\ld \phi$, establishing that $\phi$ is \ld-satisfiable.
\begin{itemize}[wide]
\item Atomic and Boolean cases are easy to verify.

\item $\psi = K_a \chi$: $\rho'(\psi)= D_{\{a, c\}} \rho'(\chi)$.
\underline{Left to right.} Suppose $M, \sigma \not \models_\ld K_a \chi$, where $\sigma$ has length $n$ and $K_a \chi$ has modal depth $k$ with $n + k \leq |\phi|$. Then, there exists $\sigma' \in W_1$ such that $\{a\} \subseteq E(\sigma,\sigma')$ and $M, \sigma \not\models_\ld \chi$. Since either $\sigma'$ extends $\sigma$ with one pair, or $\sigma$ extends $\sigma'$ with one pair, $\sigma'$ has length $n+1$ or $n-1$, $\chi$'s modal depth is $k-1$, so the sum $\leq |\phi|$. By the induction hypothesis, $N, \tail(\sigma') \not\models_\kdn \rho'(\chi)$. Since $\{a\} \subseteq E(\sigma,\sigma')$, by the definition of $E$, there exists $a \in G \in \gr$ such that either $(\tail(\sigma), \tail(\sigma')) \in \bigcap_{a\in G} R(a)$ or $(\tail (\sigma'), \tail (\sigma)) \in \bigcap_{a\in G} R(a)$. In the former case, $N, \tail(\sigma) \not \models_\kdn K_a \rho'(\chi)$ by Kripke semantics. In the latter case, from $N, w \models_\kdn \rho(\phi)$ and Definition~\ref{def:rewrite-DtoKD}(3a), it follows that $N, w \models_\kdn \bigwedge_{0 \leq i \leq |\phi|} K^i_c (\neg K_a \neg K_a \rho'(\chi) \ra \rho'(\chi))$. Hence $N, \tail(\sigma') \not \models_\kdn \neg K_a \neg K_a \rho'(\chi)$, and so $N, \tail(\sigma') \models_\kdn K_a \neg K_a \rho'(\chi)$. Thus, $N, \tail(\sigma) \not \models_\kdn K_a \rho'(\chi)$. In both cases, from $N, w \models_\kdn \rho(\phi)$ and Definition~\ref{def:rewrite-DtoKD}(3c), it follows that $N, w \models_\kdn \bigwedge_{0 \leq i \leq |\phi|} K_c^i (D_{\{a, c\}} \chi \ra K_a \chi)$, and so $N, \tail(\sigma) \not \models_\kdn D_{\{a, c\}} \rho'(\chi)$.
\underline{Right to left.} Suppose $N, \tail(\sigma) \not \models_\kdn D_{\{a, c\}} \rho'(\chi)$, then there exists $u \in W$ such that $(\tail (\sigma), u) \in R(a) \cap R(c)$ and $N, u \not \models_\kdn \rho'(\chi)$. Clearly $(w, u) \in R^+_c$. Let $\sigma'$ extends $\sigma$ with $(u, \{a,c\})$. It follows that $\tail(\sigma') = u$, and by induction hypothesis, $M, \sigma' \not \models_\ld \chi$. By the definition of $E$, $\{a\} \subseteq E(\sigma,\sigma')$, and so $M, \sigma \not \models_\ld K_a \chi$.

\item $\psi = D_G \chi$: $\rho'(\psi)= D_{G \cup \{c\}} \rho'(\chi)$. Similar reasoning applies, using $G \subseteq E(\sigma, \sigma')$ and Definition~\ref{def:rewrite-DtoKD}(3b, 3c).
\end{itemize}

(2) The function $\rho$ operates in polynomial time: Steps (1) and (2) of Definition~\ref{def:rewrite-DtoKD} are linear in $|\phi|$, replacing $K_a$ and $D_G$. Step (3) adds $\mu(\phi)$ conjuncts (size $O(|\phi|)$ from $\cl(\phi)$), and $K_c^i$ conjuncts (size $O(|\phi|^2)$), totaling $O(|\phi|^3)$ time and size. Thus, \ld-satisfiability reduces to \kdn-satisfiability in polynomial time.
\end{proof}

\subsubsection{Reduction from \texorpdfstring{\ldef}{LDEF} to \texorpdfstring{\ld}{LD}}
\label{subsec:LDEF-LD}

A procedure is presented that transforms any formula in \langdef into a formula in \langd, preserving satisfiability through the transformation.

The concept of a formula's closure, as defined in Definition~\ref{def:closure}, will be employed in the subsequent text. Additionally, the following convention is adopted for clarity and consistency.

The initial approach to rewriting an $\lang_{DEF}$-formula into an $\lang_D$-formula involves encoding group knowledge operators $E_G$ and $F_G$ using $K_{f(E_G)}$ and $K_{f(F_G)}$, where $f(E_G)$ and $f(F_G)$ are designated agents for $E_G$ and $F_G$, respectively, and adding conjuncts to express the properties of these operators. However, such conjuncts grow exponentially with iterated group knowledge operators. To achieve a polynomial-size transformation, we introduce a fresh agent $c$ and rewrite $K_a$, $D_G$, $E_G$, and $F_G$ as $D_{\{c, f(K_a)\}}$, $D_{\{c, f(D_G)\}}$, $D_{\{c, f(E_G)\}}$, and $D_{\{c, f(F_G)\}}$, respectively. The operator $K_c$ is then used to reduce the number of conjuncts.

\begin{conv}
\label{conv:agents}
Each operator $K_a$, $D_G$, $E_G$ and $F_G$, where $a \in \ag$ and $G \in \gr$, is assigned a unique agent by an injective function $f$, resulting in $f(K_a)$, $f(D_G)$, $f(E_G)$ and $f(F_G)$, respectively.

For a given formula $\phi$:
\begin{itemize}
\item $S_\phi$ denotes the set of skills appearing in $\phi$;
\item $A_\phi$ denotes the set of agents appearing in $\phi$;
\item $G_\phi$ denotes the union of groups explicitly appearing in $\phi$ and singleton groups $\{a\}$ for each agent $a$ appearing in $\phi$, formally $G_\phi =\{ G \mid G \text{ appears in }\phi \} \cup \{ \{a\} \mid a \text{ appears in } \phi \}$.
\end{itemize} 
\end{conv}

\begin{defi}[Rewriting]
\label{def:re-toD}
For an \langdef-formula $\phi$,  the \langd-formula $\rho(\phi)$ is constructed by applying the following steps sequentially:
\begin{enumerate}
\item Transform $\phi$ into $\phi \wedge \bigwedge_{0 \leq i \leq |\phi|} K^{i}_c \big(\bigwedge_{\chi \in \mu(\phi)} \chi \big) $, where $c$ is a fresh agent not appearing in $\phi$ and distinct from $f(K_a)$, $f(D_G)$, $f(E_G)$ and $f(F_G)$ for all operators $K_a$, $D_G$, $E_G$ and $F_G$ in $\phi$, and $\mu(\phi)$ is the set of the following formulas (with $a \in A_\phi$, $G, H, I, J \in G_\phi$ and $\psi \in \cl(\phi)$):
	\begin{enumerate}
	\item $F_G \psi \ra K_a\psi$, for $a \in G$
	\item $K_a\psi \ra D_G\psi$, for $a \in G$
	\item $F_H\psi \ra F_G\psi$, for $G \subseteq H$
	\item $D_G \psi \ra D_H\psi$, for $G \subseteq H$
	\item $F_I\psi \ra D_J\psi$, for $I \cap J \neq \emptyset$
	\item $E_I \psi \leftrightarrow \bigwedge_{b\in I} K_b \psi$
	\item $(D_{\{a\}} \psi \lra K_a \psi) \wedge (E_{\{a\}} \psi \lra K_a \psi) \wedge (F_{\{a\}} \psi \lra K_a \psi)$
	\end{enumerate}

\item For each agent $a \in \ag$ distinct from $c$, replace every occurrence of $K_a$ with $D_{\{c, f(K_a)\}}$;

\item For each group $G \in \gr$, replace every occurrence of $D_G$ with $D_{\{c, f(D_G)\}}$, $E_G$ with $D_{\{c, f(E_G)\}}$, and $F_G$ with $D_{\{c,f(F_G)\}}$.
\end{enumerate}

Define $\rho_1(\phi)$ as the result of applying only Step (1), and $\rho_{23}(\phi)$ as the result of applying Steps (2) and (3) sequentially to $\phi$. Then, $\rho_1(\phi)$ is an \langdef-formula, while  $\rho(\phi)$ and $\rho_{23}(\phi)$ are \langd-formulas, with $\rho(\phi) = \rho_{23}(\rho_1(\phi))$.
\end{defi}

\begin{lem}[Invariance of rewriting]\label{lem:sat-toD}
For any \langdef-formula $\phi$, $\phi$ is satisfiable (in \logic{DEF}) if and only if $\rho(\phi)$ is satisfiable (in \logic{D}).
\end{lem}
\begin{proof}
The proof follows a structure similar to that of Lemma~\ref{lem:sat-DtoKD}, with some notations used without detailed explanation here; readers may refer to Lemma~\ref{lem:sat-DtoKD} for clarification.

Left to right. Suppose $\phi$ is satisfied at a world $w$ in a model $M = (W, E, C, \beta)$. First, verify that $M, w \models \rho_1(\phi)$. Without loss of generality, assume $C(c) = \emptyset$, which is permissible since $c$ is a fresh agent absent from $\phi$ and $\rho_1(\phi)$. The formulas in $\mu(\phi)$ (Definition~\ref{def:re-toD}(1)) are valid implications or equivalences by the semantics, making $\rho_1(\phi)$ true at $w$.

Construct a new model $M' = (W, E', C', \beta)$, where:
\begin{itemize}
\item $E': W \times W \to \wp(\sk)$, where $E'(u, v)$ is the minimal set satisfying all the following:
\begin{itemize}
\item $f(K_a) \in E'(u, v)$ iff $C(a) \subseteq E(u, v)$;
\item $f(D_G) \in E'(u, v)$ iff $\bigcup_{a \in G} C(a) \subseteq E(u, v)$;
\item $f(E_G) \in E'(u, v)$ iff there exists $a\in G$ such that $C(a) \subseteq E(u, v)$;
\item $f(F_G) \in E'(u, v)$ iff $\bigcap_{a \in G} C(a) \subseteq E(u, v)$;
\item $c \in E'(u, v)$;
\end{itemize}
\item $C':\ag \to \wp(\sk)$ with $C'(a) = \{a\}$ for all $a \in \ag$.
\end{itemize}

Treating agents as skills is justified by Footnote~\ref{ft}. For all $u, v \in W$, $E'(u, v) = E'(v, u)$ (symmetry holds by definition) and $E'(u, v) \neq \ag$ (as only finitely many operators appear in $\phi$), ensuring $M'$ is a model.

By induction on $\psi \in \langdef$, one can verify that $M, u \models \psi$ iff $M',u \models \rho_{23}(\psi)$ for all $u \in W$. Since $M, w \models \rho_1(\phi)$ and $\rho(\phi) = \rho_{23}(\rho_1(\phi))$, it follows that $M', w \models \rho(\phi)$, proving $\rho(\phi)$ is satisfiable.

Right to left. Suppose $\rho(\phi)$ is satisfied at a world $w$ in a model $M = (W, E, C, \beta)$, i.e., $M, w \models \rho(\phi)$. Let $W_0 = \{ (u, G, +) \mid u \in W,\, w \reach^M_{\{c\}} u \text{ and } G \in G_\phi \} \cup  \{ (u, G, -) \mid u \in W,\, w \reach^M_{\{c\}} u \text{ and } G \in G_\phi \} \cup \{(w, \{c\}, +)\}$. Define $W_1$ as the set of finite sequences of elements of $W_0$ starting with $(w, \{c\}, +)$. For any $\sigma \in W_1$, let $\tail(\sigma)$ denote the world component of the last element in $\sigma$ (e.g., $\tail( \langle (w, \{c\}, +), (u, G, -) \rangle ) = u$).

Construct $M' = (W_1, E', C', \beta')$, where:
\begin{itemize}
\item $E': W_1 \times W_1 \to \wp(\wp(A_{\phi}))$ is defined for all $\sigma, \sigma' \in W$ and $G \in \gr$ as:
\[
E' (\sigma, \sigma') = \left \{
\begin{array}{ll}
\{ H \subseteq A_{\phi} \mid H \in \gr \text{ and }H \cap G \neq \emptyset\}, & \text{if $(\dag_1)$ and $(\dag_2)$,}
\\
\{ H \subseteq A_{\phi} \mid H \in \gr \text{ and } G \subseteq H \}, & \text{if $(\dag_3)$ and $(\dag_4)$,}
\\
\emptyset, & \text{otherwise.}
\end{array}\right.
\]

\begin{enumerate}
\item[$(\dag_1)$] Either $\sigma$ extends $\sigma'$ with $(\tail (\sigma), G, +)$ or $\sigma'$ extends $\sigma$ with $(\tail (\sigma'), G, +)$;
\item[$(\dag_2)$] For all $\psi \in \cl(\phi)$, $M, \tail(\sigma) \models D_{\{c, f(D_G)\}} \rho_{23}(\psi)$ implies $M, \tail(\sigma') \models \rho_{23}(\psi)$, and $M, \tail(\sigma') \models D_{\{c, f(D_G)\}} \rho_{23}(\psi)$ implies $M, \tail(\sigma) \models \rho_{23}(\psi)$;
\item[$(\dag_3)$] Either $\sigma$ extends $\sigma'$ with $(\tail (\sigma), G, -)$ or $\sigma'$ extends $\sigma$ with $(\tail (\sigma'), G, -)$;
\item[$(\dag_4)$] For all $\psi \in \cl(\phi)$, $M,\tail(\sigma) \models D_{\{c, f(F_G)\}} \rho_{23}(\psi)$ implies $M, \tail(\sigma') \models \rho_{23}(\psi)$, and $M,\tail(\sigma') \models D_{\{c, f(F_G)\}} \rho_{23}(\psi)$ implies $M, \tail(\sigma) \models \rho_{23}(\psi)$.
\end{enumerate}

\item $C':\ag \to \wp(\wp(A_{\phi}))$ with $C'(a) = \{ G \subseteq A_{\phi} \mid a \in G \in \gr \}$ for all $a \in \ag$.

\item $\beta': W_1 \to \wp(\pr)$ with $\beta' (\sigma) = \beta(\tail(\sigma))$ for all $\sigma \in W_1$.
\end{itemize}
Here, finite groups of agents serve as skills, justified by Footnote~\ref{ft}, since $\wp(A_\phi)$ is finite (as $A_\phi$ is) and $\sk$ is countably infinite. To verify $M'$ is a model, note that $E'$ is symmetric (conditions are bidirectional).

We show the following by induction on $\psi$:
\begin{quote}
For all $\psi \in \cl(\phi)$ and all $\sigma \in W_1$, if $\sigma$ has length $n$ and $\psi$ has modal depth $k$ with $n + k \leq |\phi|$, then $M, \tail(\sigma) \models \rho_{23}(\psi) \iff M', \sigma \models \psi$.
\end{quote}
Since $M, w \models \rho(\phi)$ and $\rho(\phi)$ includes $\rho_{23}(\phi)$, if the claim holds, then $M', \langle (w, \{c\}, +) \rangle \models \phi$ (as $n = 1$ and $k \leq |\phi| - 1$), showing that $\phi$ is satisfiable.

$\bullet$ The base case (atomic propositions) and Boolean cases are straightforward and omitted. Here the focus is knowledge operators:

$\bullet$ Case $\psi = K_a \chi$: $\rho_{23}(\psi) = D_{\{c,f(K_a)\}} \rho_{23}(\chi)$.
\underline{Left to right.} Suppose $M', \sigma \not\models K_a \chi$. Then there exists $\sigma' \in W_1$ such that $C'(a) = \{ G \subseteq A_{\phi} \mid a \in G \in \gr \} \subseteq E'(\sigma,\sigma')$ and $M', \sigma' \not \models \chi$. By the definition of $E'$, one of two cases holds:
\begin{enumerate}
\item There exists $G \in \gr$ where: (i) either $\sigma$ extends $\sigma'$ with $(\tail (\sigma), G, +)$ or $\sigma'$ extends $\sigma$ with $(\tail (\sigma'), G, +)$, (ii) for all $\theta \in \cl(\phi)$, $M, \tail(\sigma) \models D_{\{c, f(D_G)\}} \rho_{23} (\theta)$ implies $M, \tail(\sigma') \models \rho_{23} (\theta)$, and (iii) $E' (\sigma, \sigma') = \{ H \subseteq A_\phi \mid H \in \gr \text{ and } H \cap G \neq \emptyset \}$;
\\
(In this case, $\{a\} \in C'(a) \subseteq E' (\sigma, \sigma')$, it follows that $\{a\} \cap G \neq \emptyset$, hence $a \in G$.)
\item There eixsts $G \in \gr$ such that: (i) either $\sigma$ extends $\sigma'$ with $(\tail (\sigma), G, -)$ or $\sigma'$ extends $\sigma$ with $(\tail (\sigma'), G, -)$, (ii) for all $\theta \in \cl(\phi)$, $M, \tail(\sigma) \models D_{\{c, f(F_{G})\}} \rho_{23} (\theta)$ implies $M, \tail(\sigma') \models \rho_{23} (\theta)$, and (iii) $E' (\sigma, \sigma') = \{ H \subseteq A_{\phi} \mid H \in \gr \text{ and } G \subseteq H \}$.
\\
(In this case, $\{a\} \in C'(a) \subseteq E' (\sigma, \sigma')$, it follows that $G \subseteq \{a\}$, hence $G = \{a\}$.)
\end{enumerate}
Since $M', \sigma' \not\models \chi$, by induction hypothesis (length of $\sigma' \leq n + 1$, modal depth of $\chi = k - 1$, and $(n + 1) + (k - 1) \leq |\phi|$), $M, \tail(\sigma') \not \models \rho_{23}(\chi)$. In case (1), $M, \tail(\sigma) \not\models D_{\{c,f(D_G)\}} \rho_{23}(\chi)$, and in case (2), $M,\tail(\sigma) \not\models D_{\{c,f(F_{\{a\}})\}} \rho_{23}(\chi)$. Since $M, w \models \rho(\phi)$, by Definition~\ref{def:re-toD}(1b, 1g), $M, w \models \bigwedge_{0 \leq i \leq |\phi|} K^i_c (D_{\{c,f(K_a)\}} \rho_{23}(\chi) \ra D_{\{c,f(D_G)\}} \rho_{23}(\phi))$ and $M, w \models \bigwedge_{0 \leq i \leq |\phi|}K^i_c (D_{\{c,f(K_a)\}} \rho_{23}(\chi) \ra D_{\{c,f(F_{\{a\}})\}} \rho_{23}(\phi))$. In both cases, $M, \tail(\sigma) \not \models D_{\{c,f(K_a)\}} \rho_{23}(\chi) = \rho_{23}(\psi)$.
\underline{Right to left.} Suppose $M, \tail(\sigma) \not \models D_{\{c,f(K_a)\}} \rho_{23}(\chi)$. Then there exists $u \in W$ such that $C'(c) \cup C'(f(K_a)) \subseteq E(\tail(\sigma), u)$ and $M, u \not\models \rho_{23}(\chi)$. Define $\sigma'$ as $\sigma$ extended with $(u, \{a\}, +)$. Here, $\sigma'$ has length $n+1$, $\chi$ has modal depth $k-1$, so the sum $\leq |\phi|$. By the induction hypothesis, $M', \sigma' \not\models \chi$. Check $C'(a) \subseteq E'(\sigma,\sigma')$ under $(\dag_1)$ and $(\dag_2)$. By semantics and $C'(c) \cup C'(f(K_a)) \subseteq E(\tail(\sigma), u)$, $M,\tail(\sigma) \models D_{\{c,f(K_a)\}} \rho_{23}(\theta)$ implies $M,\tail(\sigma') \models \rho_{23}(\theta)$ for all $\theta \in \cl(\phi)$. Since $M, w \models \rho(\phi)$, by Definition~\ref{def:re-toD}(1g), $M, w \models \bigwedge_{0 \leq i \leq |\phi|} K^i_c (D_{\{c, f(K_a)\}} \rho_{23}(\theta) \lra D_{\{c, f(D_{\{a\}})\}} \rho_{23}(\theta))$ for any $\theta \in \cl(\phi)$. It follows that $M, \tail(\sigma) \models D_{\{c, f(D_{\{a\}})\}} \rho_{23}(\theta) \Longrightarrow M, \tail(\sigma') \models \rho_{23}(\theta)$ for all $\theta \in \cl(\phi)$. Conversely, $M, \tail(\sigma') \models D_{\{c,f(D_{\{a\}})\}} \rho_{23}(\theta) \Longrightarrow M, \tail(\sigma) \models \rho_{23}(\theta)$ for all $\theta \in \cl(\phi)$; similar reasoning applies. Thus, $C'(a) \subseteq E'(\sigma,\sigma')$, and $M', \sigma \not \models K_a \chi$.

$\bullet$ Case $\psi = D_G \chi$: $\rho_{23}(\psi) = D_{\{c,f(D_G)\}} \rho_{23}(\chi)$.
The case when $|G| = 1$ mirrors the proof for $\psi = K_a \chi$ and is omitted. We consider only $|G| >1$.
\underline{Left to right.} Suppose $M', \sigma \not \models D_G \chi$. Then there exists $\sigma' \in W_1$ such that $\bigcup_{a \in G} C'(a) \subseteq E'(\sigma,\sigma')$ and $M', \sigma' \not \models \chi$, where $\bigcup_{a \in G} C'(a) = \{ H \subseteq A_\phi \mid H \in \gr \text{ and } H \cap G \neq \emptyset \}$ (since $C'(a) = \{ H \subseteq A_\phi \mid a \in H \in \gr \}$). By the definition of $E'$, one of two cases applies:
\begin{enumerate}
\item There exists $G' \in \gr$ such that: (i) either $\sigma$ extends $\sigma'$ with $(\tail (\sigma), G', +)$ or $\sigma'$ extends $\sigma$ with $(\tail (\sigma'), G', +)$, (ii) for all $\theta \in \cl(\phi)$, $M, \tail(\sigma) \models D_{\{c, f(D_{G'})\}} \rho_{23} (\theta)$ implies $M, \tail(\sigma') \models \rho_{23} (\theta)$, and (iii) $E' (\sigma, \sigma') = \{ H \subseteq A_\phi \mid H \in \gr \text{ and } H \cap G' \neq \emptyset \}$;
\\
(Since $\{\{a\} \mid a\in G\} \subseteq C'(a) \subseteq E' (\sigma, \sigma')$, implying $\{a\} \cap G' \neq \emptyset$ for all $a \in G$, hence $G \subseteq G'$.)
\item There exists $G' \in \gr$ such that: (i) either $\sigma$ extends $\sigma'$ with $(\tail (\sigma), G', -)$ or $\sigma'$ extends $\sigma$ with $(\tail (\sigma'), G', -)$, (ii) for all $\theta \in \cl(\phi)$, $M, \tail(\sigma) \models D_{\{c, f(F_{G'})\}} \rho_{23} (\theta)$ implies $M, \tail(\sigma') \models \rho_{23} (\theta)$, and (iii) $E' (\sigma, \sigma') = \{ H \subseteq A_{\phi} \mid H \in \gr \text{ and } G' \subseteq H \}$.
\\
(Since $\{\{a\} \mid a\in G\} \subseteq C'(a) \subseteq E' (\sigma, \sigma')$ and $|G| > 1$, $G \subseteq \{a\}$ for each $a \in G$ is impossible, so this case is infeasible.)
\end{enumerate}
Thus, only Case (1) holds. Since $M', \sigma' \not \models \chi$, by induction hypothesis (length of $\sigma' \leq n + 1$, modal depth of $\chi = k - 1$, and $(n + 1) + (k - 1) \leq |\phi|$), $M, \tail(\sigma') \not\models \rho_{23}(\chi)$. Therefore, $M,\tail(\sigma) \not \models D_{\{c,f(D_{G'})\}} \rho_{23}(\chi)$. Since $M, w \models \rho(\phi)$, by Definition~\ref{def:re-toD}(1d), $M, w \models \bigwedge_{0 \leq i \leq |\phi|} K^i_c (D_{\{c,f(D_G)\}} \rho_{23}(\chi) \ra D_{\{c,f(D_{G'})\}} \rho_{23}(\chi))$,  it follows that $M, \tail(\sigma) \not \models D_{\{c,f(D_G)\}} \rho_{23}(\chi)$.
\underline{Right to left.} Suppose $M, \tail(\sigma) \not \models D_{\{c,f(D_G)\}} \rho_{23}(\chi)$. Then there exists a world $u \in W$ such that: (i) $C(c) \cup C(f(D_G)) \subseteq E(\tail(\sigma), u)$ and (ii) $M, u \not\models \rho_{23}(\chi)$. Let $\sigma'$ be $\sigma$ extended with $(u, G, +)$. By (ii) and the induction hypothesis, $M',\sigma' \not \models \chi$. Verify $\bigcup_{a \in G} C'(a) \subseteq E'(\sigma,\sigma')$ under $(\dag_1)$ and $(\dag_2)$. By (i) and the semantics that $M,\tail(\sigma) \models D_{\{c, f(D_{G})\}} \rho_{23}(\theta) \Longrightarrow M,\tail(\sigma') \models \rho_{23}(\theta)$ for all $\theta \in \cl(\phi)$. Conversely, $M,\tail(\sigma') \models D_{\{c,f(D_{G})\}} \rho_{23}(\theta) \Longrightarrow M, \tail(\sigma) \models \rho_{23}(\theta)$ for all $\theta \in \cl(\phi)$. These enforce $(\dag_2)$ for $E' (\sigma, \sigma')$. By definition, the elements of $\bigcup_{a \in G} C'(a)$ are $H$'s that contains at least one element of $G$, thus $\bigcup_{a \in G} C'(a) = \{H \subseteq A_\phi \mid H\in \gr \text{ and } H \cap G \neq \emptyset \}$, it is $E'(\sigma,\sigma')$ under $(\dag_1)$ and $(\dag_2)$. Hence $\bigcup_{a \in G} C'(a) \subseteq E'(\sigma,\sigma')$, and so $M', \sigma \not \models D_G \chi$.

$\bullet$ For $\psi = E_G \chi$, where $\rho_{23}(\psi) = D_{\{c, f(E_G)\}} \rho_{23}(\chi)$, the proof resembles the $K_a \chi$ case, relying on Definition~\ref{def:re-toD}(1f, 1g).

$\bullet$ For $\psi = F_G \chi$, where $\rho_{23}(\psi) = D_{\{c, f(F_G)\}} \rho_{23}(\chi)$, the proof is analogous to the $D_G \chi$ case, leveraging Definition~\ref{def:re-toD}(1a, 1c, 1e, 1g). 
\end{proof}

\begin{lem}\label{lem:red-toD}
The satisfiability problem for $\logic{DEF}$ is polynomial-time reducible to the satisfiability problem for $\logic{D}$.
\end{lem}
\begin{proof}
Given an \langdef-formula $\phi$, Lemma~\ref{lem:sat-toD} establishes that, an \langd-formula $\rho(\phi)$ constructed per Definition~\ref{def:re-toD}, satisfies the property that $\phi$ is satisfiable if and only if $\rho(\phi)$ is satisfiable. Thus, the satisfiability problem for $\phi$ reduces to that for $\rho(\phi)$ in $\logic{D}$.

To confirm polynomial-time reducibility, it suffices to demonstrate that the procedure $\rho$ operates in polynomial time relative to the size of $\phi$, denoted $|\phi| = k$. The execution of the first step in computing $\rho(\phi)$ (Definition~\ref{def:re-toD}) is polynomial in $k$, as it merely involves listing the formulas in $\mu(\phi)$ $k$ times and binding them with conjunction. The size of $\mu(\phi)$ is polynomial, given that: (a) the number of subformulas of $\phi$ is at most $k$, (b) the number of modal operators present in $\phi$ is at most $k$, and (c) the size of any group appearing in $\phi$ is at most $k$. Steps (2) and (3) cost linear time with respect to the length of the formula obtained after Step (1).
\end{proof}

Following the establishment of Lemmas~\ref{lem:sat-KBtoL}, \ref{lem:sat-DtoKD}, and \ref{lem:red-toD}, the relationships depicted in Figure~\ref{fig:sat-psp} are now evident. These results enable the derivation of the following theorem, which applies to all logics ranging from $\l$ to $\logic{DEF}$.

\begin{thm}\label{thm:complexity-satpsp}
The satisfiability problems for any logic between \l and \ldef is PSPACE complete.
\end{thm}

\subsection{Satisfiability for logics with common knowledge but without update or quantifying modalities: EXPTIME complete}

Following the PSPACE completeness results for logics between \l and \ldef, we now examine logics incorporating common knowledge operators, excluding update and quantifying modalities.
To simplify the proofs, the \emph{universal modality}, denoted $U$, is introduced into the logics to express properties that hold across all worlds. This modality serves a role analogous to $K_c$ in the reduction from $\ldef$ to $\ld$ (Section~\ref{subsec:LDEF-LD}). Its semantics is defined as follows:
\[ M, w \models \univ \phi \iff \text{for all worlds $u$ of $M$, $M, u \models \phi$.}\]
The size of formulas containing the universal modality adheres to Convention~\ref{sec:inputsize}: each occurrence of $U$ increments the formula length by 1. Formally, the size of $U \phi$ is $|U \phi| = |\phi| + 1$.

Figure~\ref{fig:sat-C} delineates the proof strategy and complexity results for the satisfiability problems for logics incorporating common knowledge and the universal modality, without update or quantifying modalities, establishing their EXPTIME completeness. For those focused solely on the logics introduced in Section~\ref{sec:logics}, the roadmap can be streamlined by omitting the nodes for \kutwo and \lu, and replacing \lGA with \lG. This adjustment is viable since the universal modality remains invariant under the rewriting process, allowing the reduction from \lGA to \lcu to also serve as a reduction from \lG to \lcu. These additional results are included to provide a comprehensive analysis of related logics.

\begin{figure}
\begin{tikzpicture}[
	modal,
	node distance=1cm and 2.3cm,
	world/.style={ellipse, draw, minimum width=1cm, minimum height=.5cm, inner sep=1ex},
 	rectworld/.style={rectangle, draw, minimum width=2cm, minimum height=.5cm, inner sep=1ex},
  	font=\footnotesize
]
\node[rectworld] (w0) {$\dfrac{\kutwo}{\text{EXPTIME complete}}$};
\node[world] (w0b) [right=of w0] {\logic{\univ}};
\node[rectworld] (w1) [below=of w0] {$\dfrac{\sfivec}{\text{EXPTIME complete}}$};
\node[world] (w2) [right=of w1] {\lc};
\node[world] (w3) [right=of w2] {\lcdefu};
\node[world] (w4) [below=of w2] {\lcu};
\node[rectworld] (w5) [left=of w4] {$\dfrac{\cpdl}{\text{in EXPTIME}}$};
\path[->, dashed] (w0) edge node[above, font=\tiny] {PTIME} node[below, font=\tiny] {(Lemma~\ref{lem:red-to=})} (w0b);
\path[->, solid] (w0b) edge (w3);
\path[->, dashed] (w1) edge node[above, font=\tiny] {PTIME} node[below, font=\tiny] {(Lemma~\ref{lem:S5C-C})} (w2);
\path[->, solid] (w2) edge (w3);
\path[->, dashed] (w3) edge node[above, sloped, font=\tiny] {PTIME} node[below, sloped, font=\tiny] {(Lemma~\ref{lem:CDEF-CU})} (w4);
\path[->, dashed] (w4) edge node[above, font=\tiny] {PTIME} node[below, font=\tiny] {(Lemma~\ref{lem:CU-CPDL})} (w5);
\end{tikzpicture}
\caption{Roadmap of proofs for the complexity of satisfiability problems for logics with common knowledge, excluding update and quantifying modalities. Boxed nodes display known complexity results. A solid arrow from one logic to another indicates that the satisfiability problem for the former logic is a subproblem for the latter. A dashed arrow labeled ``PTIME'' denotes a polynomial-time reduction from the satisfiability problem for the source logic to that of the target logic. The EXPTIME completeness of \sfivec is from \cite[Section 3.5]{FHMV1995}. The EXPTIME upper bound for \cpdl is from \cite[Corallary 7.7]{PT1991}. The EXPTIME completeness of \kutwo is from \cite[Corallary 5.4.8]{Spaan1993}.
}\label{fig:sat-C}
\end{figure}

\subsubsection{Reduction from \texorpdfstring{\lGA}{LCDEFU} to \texorpdfstring{\lcu}{LCU}}

A procedure is introduced that transforms any formula in \langGA into a formula in \langcu, preserving satisfiability through the transformation.

The concept of a formula's closure, as introduced in Definition~\ref{def:closure}, and the convention of designated agents and skills, as established in Convention~\ref{conv:agents}, are utilized in the following discussion. The rewriting process presented below adapts techniques from Definitions~\ref{def:rewrite-DtoKD} and \ref{def:re-toD}, with a key simplification enabled by the common knowledge operators ($C_G$) and the universal modality ($U$), as detailed in the following definition.

\begin{defi}[Rewriting]
\label{def:rw-G-CA}
For an \langGA-formula $\phi$,  the \langcu-formula $\rho(\phi)$ is constructed by applying the following steps sequentially:
\begin{enumerate}
\item Transform $\phi$ into $\phi \wedge \univ \big(\bigwedge_{\chi \in \mu(\phi)} \chi \big) $, where $\mu(\phi)$ is the set of the following formulas (with $a \in A_\phi$, $G, H, I, J \in G_\phi$ and $\psi \in \cl(\phi) \cup \{ C_G \chi \mid \chi \in \cl(\phi) \text{ and } G \in G_\phi\}$):
	\begin{enumerate}
	\item $F_G \psi \ra K_a\psi$, for $a \in G$
	\item $K_a\psi \ra D_G\psi$, for $a \in G$
	\item $F_H\psi \ra F_G\psi$, for $G \subseteq H$
	\item $D_G \psi \ra D_H\psi$, for $G \subseteq H$
	\item $F_I\psi \ra D_J\psi$, for $I \cap J \neq \emptyset$
	\item $E_I \psi \leftrightarrow \bigwedge_{b\in I} K_b \psi$
	\item $(D_{\{a\}} \psi \lra K_a \psi) \wedge (E_{\{a\}} \psi \lra K_a \psi) \wedge (F_{\{a\}} \psi \lra K_a \psi)$
	\end{enumerate}

\item For each agent $a \in \ag$, replace every occurrence of $K_a$ with $K_{f(K_a)}$;

\item For each group $G \in \gr$, replace every occurrence of $D_G$ with $K_{f(D_G)}$, $E_G$ with $K_{f(E_G)}$, and $F_G$ with $K_{f(F_G)}$;

\item For each group $G \in \gr$, replace every occurrence of $C_G$ with $C_{f(C_G)}$, where $f(C_G) = \{ f(K_a) \mid a \in G \}$.
\end{enumerate}

Define $\rho_1(\phi)$ as the result of applying only Step (1), and $\rho_{234}(\phi)$ as the result of applying Steps (2)--(4) sequentially to $\phi$. Then, $\rho_1(\phi)$ is an \langG-formula, while  $\rho(\phi)$ and $\rho_{234}(\phi)$ are \langcu-formulas, with $\rho(\phi) = \rho_{234}(\rho_1(\phi))$.
\end{defi}

\begin{lem}[Invariance of rewriting]\label{lem:sat-toCA}\label{lem:CDEF-CU}
\begin{enumerate}
\item For any \langGA-formula $\phi$, $\phi$ is satisfiable (in \lGA) if and only if $\rho(\phi)$ is satisfiable (in \lcu).
\item The satisfiability problem for $\l_{CDEF}$ is polynomial-time reducible to that for $\l_{C\univ}$.
\end{enumerate}
\end{lem}
\begin{proof}
(1) The proof adapts the structure of Lemma~\ref{lem:sat-toD}, with some notations assumed familiar; readers may consult Lemma~\ref{lem:sat-toD} for additional details.

Left to right. Suppose $\phi$ is satisfied at a world $w$ in a model $M = (W, E, C, \beta)$. First, verify that $M, w \models \rho_1(\phi) = \phi \wedge U \big( \bigwedge_{\chi \in \mu(\phi)} \chi \big)$ (Definition~\ref{def:rw-G-CA}, Step (1)). The formulas in $\mu(\phi)$ are valid implications or equivalences by the semantics, so $\rho_1(\phi)$ is true at $w$.
Construct a model $M' = (W, E', C', \beta)$, adapting the model $M'$ introduced in the left-to-right direction of the proof of Lemma~\ref{lem:sat-toD} by deleting ``$c \in E'(u, v)$'' from the conditions of $E'$. By induction on $\psi \in \langG$, it can be shown that $M, u \models \psi$ iff $M',u \models \rho_{234}(\psi)$ for all $u \in W$. Case $\psi = C_G \chi$:
\[\begin{tabular}{@{}l@{\ \ }l@{}} 
& $M, u \not \models C_G \chi$ (let $u = u_0$)
\\
iff & There exist $u_0, \dots, u_n \in W$ and $a_1, \dots, a_n \in G$:\\
& for all $1 \leq i \leq n:  C(a_i) \subseteq E ( u_{i-1}, u_{i} )$ and $M, u_n \not \models \chi$
\\
iff & There exist $u_0, \dots, u_n \in W$ and $a_1, \dots, a_n \in G$:\\
& for all $1 \leq i \leq n:  f(K_{a_i}) \in E' ( u_{i-1}, u_{i} )$ and $M', u_n \not \models \rho_{234}(\chi)$
\\
iff & There exist $u_0, \dots, u_n \in W$ and $f(K_{a_1}), \dots, f(K_{a_n}) \in f(C_G)$:\\
& for all $1 \leq i \leq n:  C'(f(K_{a_i})) \subseteq E' ( u_{i-1}, u_{i} )$ and $M', u_n \not \models \rho_{234}(\chi)$
\\
iff & $M', u \not \models C_{f(C_G)} \rho_{234}(\chi)$
\\
iff & $M', u \not \models \rho_{234}(C_G \chi)$.
\end{tabular}\]
Given $M, w \models \rho_1(\phi)$ and $\rho(\phi) = \rho_{234}(\rho_1(\phi))$, it follows that $M', w \models \rho(\phi)$, proving $\rho(\phi)$ is satisfiable.

Right to left. Suppose $\rho(\phi)$ is satisfied at a world $w$ in a model $M = (W, E, C, \beta)$. Let:
\begin{itemize}
\item $W_0 = \{ (u, G, +) \mid u \in W \text{ and } G \in G_\phi \} \cup  \{ (u, G, -) \mid u \in W  \text{ and } G \in G_\phi \} \cup \{ (w, A_\phi, +) \}$;
\item $W_1$ be the set of finite sequences of elements of $W_0$ starting with $(w, A_\phi, +)$.
\end{itemize}
 For any $\sigma \in W_1$, let $\tail(\sigma)$ denote the world component of the last element in $\sigma$.
Construct $M' = (W_1, E', C', \beta')$, where:
\begin{itemize}
\item $E': W_1 \times W_1 \to \wp(\wp(A_{\phi}))$ is defined for all $\sigma, \sigma' \in W$ and $G \in \gr$ as:
\[
E' (\sigma, \sigma') = \left \{
\begin{array}{ll}
\{ H \subseteq A_{\phi} \mid H \in \gr \text{ and }H \cap G \neq \emptyset\}, & \text{if $(\dag_1)$ and $(\dag_2)$,}
\\
\{ H \subseteq A_{\phi} \mid H \in \gr \text{ and } G \subseteq H \}, & \text{if $(\dag_3)$ and $(\dag_4)$,}
\\
\emptyset, & \text{otherwise.}
\end{array}\right.
\]
\begin{enumerate}
\item[$(\dag_1)$] Either $\sigma$ extends $\sigma'$ with $(\tail (\sigma), G, +)$ or $\sigma'$ extends $\sigma$ with $(\tail (\sigma'), G, +)$;
\item[$(\dag_2)$] For all $\psi \in \cl(\phi) \cup \{ C_G \chi \mid \chi \in \cl(\phi) \text{ and } G \in G_\phi\}$, $M, \tail(\sigma) \models K_{f(D_G)} \rho_{234}(\psi) \Longrightarrow M, \tail(\sigma') \models \rho_{234}(\psi)$, and $M, \tail(\sigma') \models K_{f(D_G)} \rho_{234}(\psi) \Longrightarrow M, \tail(\sigma) \models \rho_{234}(\psi)$;
\item[$(\dag_3)$] Either $\sigma$ extends $\sigma'$ with $(\tail (\sigma), G, -)$ or $\sigma'$ extends $\sigma$ with $(\tail (\sigma'), G, -)$;
\item[$(\dag_4)$] For all $\psi \in \cl(\phi) \cup \{ C_G \chi \mid \chi \in \cl(\phi) \text{ and } G \in G_\phi\}$, $M, \tail(\sigma) \models K_{f(F_G)} \rho_{234}(\psi) \Longrightarrow M, \tail(\sigma') \models \rho_{234}(\psi)$, and $M,\tail(\sigma') \models K_{f(F_G)} \rho_{234}(\psi) \Longrightarrow M, \tail(\sigma) \models \rho_{234}(\psi)$.
\end{enumerate}

\item $C':\ag \to \wp(\wp(A_{\phi}))$ with $C'(a) = \{ G \subseteq A_{\phi} \mid a \in G \in \gr \}$ for all $a \in \ag$.

\item $\beta': W_1 \to \wp(\pr)$ with $\beta' (\sigma) = \beta(\tail(\sigma))$ for all $\sigma \in W_1$.
\end{itemize}
Finite groups of agents serve as skills, justified by Footnote~\ref{ft}, since $\wp(A_\phi)$ is finite (as $A_\phi$ is) and $\sk$ is countably infinite. To confirm $M'$ is a model, note that $E'$ is symmetric (conditions are bidirectional).

Establish the following by induction on $\psi$:
\begin{quote}
For all $\psi \in \cl(\phi)$ and all $\sigma \in W_1$, $M, \tail(\sigma) \models \rho_{234}(\psi) \iff M', \sigma \models \psi$.
\end{quote}
Since $M, w \models \rho(\phi)$ and $\rho(\phi)$ includes $\rho_{234}(\phi)$, if the claim holds, then $M', \langle (w, A_\phi, +) \rangle \models \phi$, showing that $\phi$ is satisfiable.

$\bullet$ The atomic and Boolean cases are straightforward and omitted.

$\bullet$ The cases for individual ($K_a$), distributed ($D_G$) and field ($F_G$) knowledge mirror the proof of Lemma~\ref{lem:sat-toD}. Here, we detail only the case $\psi = D_G \chi$ with $|G| > 1$ to highlight subtle differences, where $\rho_{234}(\psi) = K_{f(D_G)} \rho_{234}(\chi)$. 

\underline{Left to right.} Suppose $M', \sigma \not \models D_G \chi$. Then there exists $\sigma' \in W_1$ such that $\bigcup_{a \in G} C'(a) \subseteq E'(\sigma,\sigma')$ and $M', \sigma' \not \models \chi$, where $\bigcup_{a \in G} C'(a) = \{ H \subseteq A_{\phi} \mid H \in \gr \text{ and } H \cap G \neq \emptyset\} $. By the definition of $E'$, one of two cases applies:
\begin{enumerate}
\item There exists $G' \in \gr$ such that: (i) either $\sigma$ extends $\sigma'$ with $(\tail (\sigma), G', +)$ or $\sigma'$ extends $\sigma$ with $(\tail (\sigma'), G', +)$, (ii) $M, \tail(\sigma) \models K_{f(D_{G'})} \rho_{234} (\chi)$ implies $M, \tail(\sigma') \models \rho_{234} (\chi)$, and (iii) $E' (\sigma, \sigma') = \{ H \subseteq A_\phi \mid H \in \gr \text{ and } H \cap G' \neq \emptyset \}$;
\\
(In this case, $\{\{a\} \mid a\in G\} \subseteq C'(a) \subseteq E' (\sigma, \sigma')$, implying $\{a\} \cap G' \neq \emptyset$ for any $a \in G$, hence $G \subseteq G'$.)
\item There eixsts $G' \in \gr$ such that: (i) either $\sigma$ extends $\sigma'$ with $(\tail (\sigma), G', -)$ or $\sigma'$ extends $\sigma$ with $(\tail (\sigma'), G', -)$, (ii) $M, \tail(\sigma) \models K_{f(F_{G'})} \rho_{234} (\chi)$ implies $M, \tail(\sigma') \models \rho_{234} (\chi)$, and (iii) $E' (\sigma, \sigma') = \{ H \subseteq A_{\phi} \mid H \in \gr \text{ and } G' \subseteq H \}$.
\\
(In this case, $\{\{a\} \mid a\in G\} \subseteq C'(a) \subseteq E' (\sigma, \sigma')$, it follows that $G' \subseteq \{a\}$ for each $a \in G$, which is impossible since $|G| > 1$.)
\end{enumerate}
Thus, only Case (1) holds. Since $M', \sigma' \not \models \chi$, by induction hypothesis, $M, \tail(\sigma') \not \models \rho_{234}(\chi)$. Therefore, $M,\tail(\sigma) \not \models K_{f(D_{G'})} \rho_{234}(\chi)$. Since $M, w \models \rho(\phi)$, then by Definition~\ref{def:rw-G-CA}(1d), $M, w \models \univ (K_{f(D_{G})} \rho_{234}(\chi) \ra K_{f(D_{G'})} \rho_{234}(\chi))$,  it follows that $M, \tail(\sigma) \not \models K_{f(D_{G})} \rho_{234}(\chi)$.
\underline{Right to left.} Suppose $M, \tail(\sigma) \not \models K_{f(D_G)} \rho_{234}(\chi)$. Then there exists $u \in W$ such that: (i) $C(f(D_G)) \subseteq E(\tail(\sigma), u)$ and (ii) $M, u \not\models \rho_{234}(\chi)$. By (ii) and induction hypothesis, it follows that $M',\sigma' \not\models \chi$ when $\sigma'$ be $\sigma$ extended with $(u, G, +)$. Then by (i) and the semantics, $M, \tail(\sigma) \models K_{f(D_G)} \rho_{234}(\theta)$ implies $M, \tail(\sigma) \models \rho_{234}(\theta)$ for all $\theta \in \cl(\phi) \cup \{ C_G \chi \mid \chi \in \cl(\phi) \text{ and } G \in G_\phi\}$. Conversely, $M,\tail(\sigma') \models K_{f(D_G)} \rho_{234}(\theta)$ implies $M, \tail(\sigma) \models \rho_{234}(\theta)$ for all $\theta \in \cl(\phi) \cup \{ C_G \chi \mid \chi \in \cl(\phi) \text{ and } G \in G_\phi\}$. By definition, $\bigcup_{a \in G} C'(a) = \{H \subseteq A_\phi \mid H\in \gr \text{ and } H\cap G \neq \emptyset \}$. Thus, $\bigcup_{a \in G}C'(a) \subseteq E'(\sigma, \sigma')$, so $M', \sigma \not\models D_G \chi$.

$\bullet$ Case $\psi = C_G \chi$: $\rho_{234}(\psi) = C_{f(C_G)} \rho_{234}(\chi)$.
\underline{Left to right.} Suppose $M', \sigma \not \models C_G \chi$. Then there exist $\sigma_1, \dots, \sigma_n \in W_1$ and $a_1, \dots, a_n \in G$ such that: (i) $C'(a_1) \subseteq E'(\sigma,\sigma_1)$, $C'(a_2) \subseteq E'(\sigma_1, \sigma_2)$, \dots, $C'(a_n) \subseteq E'(\sigma_{n-1}, \sigma_n)$, and (ii) $M', \sigma_n \not\models \chi$. By (ii) and induction hypothesis, $M, \tail(\sigma_n) \not \models \rho_{234}(\chi)$. By an argument similar to the case for $K_a$, $M, \tail(\sigma_{n-1}) \not \models K_{f(K_a)} \rho_{234}(\chi)$, and so $M, \tail(\sigma_{n-1}) \not \models C_{f(C_G)} \rho_{234}(\chi)$. Do the inference again, $M, \tail(\sigma_{n-2}) \not \models K_{f(K_a)} C_{f(C_G)} \rho_{234}(\chi)$ and $M, \tail(\sigma_{n-2}) \not \models C_{f(C_G)} \rho_{234}(\chi)$. Repeating backwards, $M, \tail(\sigma) \not\models K_{f(K_a)} C_{f(C_G)} \rho_{234}(\chi)$ and $M, \tail(\sigma) \not \models C_{f(C_G)} \rho_{234}(\chi)$.
\underline{Right to left.} Suppose $M, \tail(\sigma) \not \models C_{f(C_G)} \rho_{234}(\chi)$. Then there exist $u_1, \dots, u_n \in W$ and $a_1, \dots, a_n \in G$ such that: $C(f(K_{a_1})) \subseteq E(\tail(\sigma),u_1)$, $C(f(K_{a_2})) \subseteq E(u_1,u_2)$, \dots, $C(f(K_{a_n})) \subseteq E(u_{n-1},u_n)$, and (ii) $M, u_n \not \models \rho_{234}(\chi)$. By (ii) and induction hypothesis, $M',\sigma_n \not \models \chi$. Let $\sigma_1$ extend $\sigma$ with $(u_1,\{a_1\},+)$, $\sigma_2$ extend $\sigma_1$ with $(u_2, \{a_2\}, +)$, \dots, $\sigma_n$ extend $\sigma_{n-1}$ with $(u_n, \{a_n\}, +)$. Similarly to case for $K_a$, $C'(a_1) \subseteq E'(\sigma,\sigma_1)$, $C'(a_2) \subseteq E'(\sigma_1,\sigma_2)$, \dots, $C'(a_n) \subseteq E'(\sigma_{n-1},\sigma_n)$. Hence $M', \sigma \not \models C_G \chi$.

$\bullet$ Case $\psi = \univ \chi$: $\rho_{234}(\psi) = \univ \rho_{234}(\chi)$. $M, \tail(\sigma) \not \models \univ \rho_{234}(\chi)$, iff there exists $u \in W$ such that $M, u \not\models \rho_{234}(\chi)$, iff there exists $\sigma' \in W_1$ such that $\tail(\sigma') = u$ and $M', \sigma' \not \models \chi$, iff $M',\sigma \not \models \univ \chi$.

(2) follows from (1) and the fact that $| \rho(\phi) |$ is polynomial in $| \phi |$, per Definition~\ref{def:rw-G-CA}.
\end{proof}

\subsubsection{Reduction from \texorpdfstring{\sfivec}{S5C2} to \texorpdfstring{\lc}{LC}}

The logic \sfivec is the two-agent epistemic logic with common knowledge, built upon the modal S5 system. It is based on the language \langc, restricted to only two agents (let them be $a, b \in \ag$; hereafter the language is referred to as ``two-agent \langc''), and interpreted over S5 models using standard Kripke semantics. It is established in \cite[Section 3.5]{FHMV1995} that the satisfiability problem for \sfivec is EXPTIME complete. In contrast, if the language \langc is interpreted over arbitrary Kripke models without S5 constraints, using standard Kripke semantics, the resulting logic is denoted \kc.

Recall that a Kripke model $(W, R, V)$ is an \emph{S5 model} if $R(a)$ is an equivalence relation---reflexive, symmetric and transitive---for all $a \in \ag$. For a group $G$, a \emph{classical $G$-path in a Kripke model $M = (W, R, V)$ from a world $w$ to a world $u$} is a finite sequence $\langle w_0, w_1, \dots, w_n \rangle$ such that $w_0 = w$, $w_n = u$, and for all $i$ where $1 \leq i \leq n$, there exists an agent $a_i \in G$ such that $(w_{i-1}, w_i) \in R(a_i)$. We write $w \creach^M_G u$ if there exists a classical $G$-path from $w$ to $u$ in $M$, omitting the superscript $M$ when the model is clear from context. For any agent $a$ and nonempty group $G$, the formulas $K_a \phi$ and $C_G \phi$ are interpreted at a world $w$ in a Kripke model $M = (W, R, V)$ as follows:
\[\begin{array}{lcl}
M, w \models K_a \phi &\iff& \text{for all $u \in W$, if $(w, u) \in R(a)$ then $M, u \models \phi$} \\
M, w \models C_G \phi &\iff&\text{for all $u \in W$, if $w \creach_G u$ then $M, u \models \phi$.}
\\
\end{array}\]

We propose a transformation that converts any two-agent \langc-formula satisfiable in \sfivec into an \langc-formula satisfiable in \lc. The concept of a formula's closure, as defined in Definition~\ref{def:closure}, will be employed in the subsequent text.

\begin{defi}[Rewriting]\label{def:rewrite-S5toC}
For a two-agent \langc-formula $\phi$, define
\[\textstyle \rho(\phi) = \phi \wedge \big(\bigwedge_{\chi \in \mu(\phi)} \chi \big) \wedge C_{\{a,b\}} (\bigwedge_{\chi \in \mu(\phi)}\chi),\]
where $\mu(\phi)$ is the collection of these formulas: (i) $K_i \psi \ra K_i K_i \psi$ and (ii) $K_i \psi \ra \psi$, where $i \in \{a, b\}$, $\psi \in \cl(\phi) \cup \{ C_G \chi \mid \chi \in cl(\phi) \text{ and } G \subseteq \{a, b\}\}$.
\end{defi}

It is clear that $\rho(\phi)$ remains a two-agent \langc-formula whenever $\phi$ is.

\begin{lem}[Invariance of rewriting]\label{lem:sat-S5toC}\label{lem:S5C-C}
\begin{enumerate}
\item For any two-agent \langc-formula $\phi$, $\phi$ is satisfiable in \sfivec if and only if $\rho(\phi)$ is satisfiable (in \lc);
\item The satisfiability problem for \sfivec is polynomial-time reducible to that for \lc.
\end{enumerate}
\end{lem}
\begin{proof}
Left to right. Suppose $\phi$ is satisfied at a world $w$ in an S5 model $N = (W, R, V)$, i.e., $N, w \models_{\sfivec} \phi$. It can be readily confirmed that $N, w \models_\sfivec \rho (\phi)$. Construct a model $M = (W, E, C, \beta)$ as follows:
\begin{itemize}
\item $E: W \times W \to \wp(\ag)$ with $E(u, v) = \{c \in \{a, b\} \mid (u, v) \in R(c) \}$ for all $u, v \in W$;
\item $C: \ag \to \wp(\ag)$ with $C(x) = \{x\}$ for all $x \in \ag$;
\item $\beta = V$.
\end{itemize}
Using agents as skills is justified by Footnote~\ref{ft}, and $M$ can be verified to be a model. 
For any $u, v \in W$ and $x \in \{a, b\}$, $(u, v) \in R(x) \iff C(x) \subseteq E(u, v)$. It can be shown by induction that for all two-agent \langc-formulas $\psi$ and all $u \in W$, $N, u \models_\sfivec \psi \iff M, u \models_\lc \psi$. Hence, $M, w \models_{\lc} \rho(\phi)$, proving $\rho(\phi)$ is satisfiable in \lc.

Right to left. Suppose $\rho(\phi)$ is satisfied at a world $w$ in a model $M = (W, E, C, \beta)$, i.e., $M, w \models \rho(\phi)$.
Construct a two-agent Kripke model $N = (W, R, V)$ where:
\begin{itemize}
\item For $x \in \{a, b\}$, $R(x) = \{ (u, v) \mid C(x) \subseteq E(u, v) \}$,
\item $V = \beta$.
\end{itemize}
It can be shown by induction that for all $u \in W$ and all two-agent \lc-formulas $\psi$, $N, u \models_\kc \psi \iff M, u \models_\lc \psi$. Consequently, $N, w \models_\kc \rho(\phi)$, so $N, w \models_\kc \phi$. 

Construct a two-agent S5 model $N^* = (W, R^*, V)$ where $R^*(a)$ and $R^*(b)$ are the reflexive and transitive closures of $R(a)$ and $R(b)$, respectively. We show the following by induction:
\begin{quote}
For all two-agent \langc-formulas $\psi \in \cl(\phi)$ and all $u \in W$ such that $u = w$ or $w \creach^N_{\{a, b\}} u$, $N, u \models_\kc \psi \iff N^*, u \models_\sfivec \psi$.
\end{quote}
Consequently, $N^*, w \models_\sfivec \phi$, proving $\phi$ is satisfiable in \sfivec.

$\bullet$ Atomic and boolean cases: straightforward.

$\bullet$ Case $\psi = K_a \chi$. If $N^*, u \models_{\sfivec} K_a \chi$, then for all $v$ with $(u, v) \in R^*(a)$, $N^*, v \models_{\sfivec} \chi$. Since $R(a) \subseteq R^*(a)$, this implies $N, u \models_{\kc} K_a \chi$. Conversely, suppose that $N^*, u \not \models_\sfivec K_a \chi$. Then there exists $v \in W$ with $(u, v) \in R^*(a)$ and $N^*, v \not\models_\sfivec \chi$. Two subcases arise:
\begin{itemize}[leftmargin=3em]
\item[(i)] $u = v$: Since $N, w \models_\kc \big(\bigwedge_{\theta \in \mu(\phi)} \theta \big) \wedge C_{\{a,b\}} (\bigwedge_{\theta \in \mu(\phi)} \theta)$, $N, v \models_\kc K_a \chi \to \chi$. By induction, $N, v \not\models_\kc \chi$, so $N, u \not \models_\kc K_a \chi$;
\item[(ii)] $u \neq v$ and $u \creach_{\{a\}} v$ in model $N$. Suppose towards a contradiction that $N, u \models_\kc K_a \chi$. Since $w \creach^N_{\{a, b\}} u$ and $N, w \models_\kc \big(\bigwedge_{\theta \in \mu(\phi)} \theta \big) \wedge C_{\{a,b\}} (\bigwedge_{\theta \in \mu(\phi)} \theta)$, $N, u \models_\kc (K_a \chi \ra K_a K_a \chi) \wedge C_{\{a, b\}} (K_a \chi \ra K_a K_a \chi)$. Thus, by the semantics, $N, u \models_\kc K_a^n \chi$ for any $n \geq 1$, implying $N, v \models \chi$ contradicting $N, v \not\models_\kc \chi$ (by induction). Hence, $N, u \not\models_\kc K_a \chi$.
\end{itemize}

$\bullet$ Case $\psi = K_b \chi$, $\psi = C_{\{a\}} \chi$ and $\psi = C_{\{b\}} \chi$ are similar.

$\bullet$ Case $\psi = C_{\{a,b\}} \chi$. First, observe that for all $v \in W$, $u \creach^{N^*}_{\{a,b\}} v$ iff $u = v$ or $u \creach^N_{\{a,b\}} v$. This holds because $R^*(x)$ extends $R(x)$ with reflexivity (adding $u = v$) and transitivity (already covered by the definition of an $\{a, b\}$-path in $N$). Thus:
\[\begin{tabular}{@{}l@{\ \ }l@{}} 
& $N, u \not \models_\kc C_{\{a,b\}} \chi$ $(*)$
\\
iff & $N, u \not \models_\kc \chi$, or there exists $v \in W$ with $u \creach^{N}_{\{a, b\}} v$ such that $N, v \not \models_\kc \chi$ $(\dag)$
\\
iff & There exists $v \in W$ with $u = v$ or $u \creach^{N}_{\{a, b\}} v$ such that $N, v \not \models_\kc \chi$
\\
iff & There exists $v \in W$ with $u \creach^{N^*}_{\{a, b\}} v$ such that $N^*, v \not \models_\sfivec \chi$
\\
iff & $N^*, u \not \models_\sfivec C_{\{a,b\}} \chi$.
\end{tabular}\]

$(*)$ to $(\dag)$ follows from the semantics. From $(\dag)$ to $(*)$, suppose $N, u \not \models_\kc \chi$, then similar to (i), $N, u \not \models_\kc K_a\chi$, so $N, u \not \models_\kc C_{\{a, b\}} \chi$.

(2) The rewriting $\rho(\phi)$ (Definition~\ref{def:rewrite-S5toC}) is computable in polynomial time, as $\mu(\phi)$ is linear in $|\cl(\phi)|$, and the reduction preserves satisfiability by statement (1).
\end{proof}

\subsubsection{Reduction from \texorpdfstring{\lcu}{LCU} to \texorpdfstring{\cpdl}{CPDL}}

We propose a transformation that converts any satisfiable \lcu-formula (in \lcu) into a satiafiable formula in Combinatory Propositional Dynamic Logic (\cpdl) introduced in \cite{PT1991}. The satisfiability problem for \cpdl is known to be in EXPTIME \cite[Corollary~7.7]{PT1991}. The syntax and semantics of \cpdl are briefly outlined below.

The syntax of $\cpdl$ comprises:
\begin{align*}
&\text{(Formulas)}& \phi ::= &\ p \mid \neg \phi \mid (\phi \ra \phi) \mid [\pi] \phi \\
&\text{(Programs)}& \pi ::= &\ a \mid (\pi;\pi) \mid (\pi \cup \pi) \mid \pi^* \mid \phi? \mid \nu
\end{align*}
where $p \in \pr$, $a \in \ag$ with $\ag$ treated as the set of \emph{atomic programs}, and $\nu \notin \ag$ is a distinguished \emph{universe} program. Formulas $\phi$ in this definition are called \cpdl-formulas, and $\pi$ are called \emph{programs}. The set of all programs is denoted $\Pi$.
A CPDL model is a Kripke model $N = (W, R, V)$ where $W$ is a nonempty set of worlds, $V: W \to \wp(\pr)$ is a valuation,  and $R: \Pi \to \wp(W \times W)$ assigns binary relations to programs $\pi \in \Pi$, satisfying:
\begin{itemize}
\item $R(\nu) = W \times W$ (universal relation);
\item $R((\pi_1 \cup \pi_2)) = R(\pi_1) \cup R(\pi_2)$ (union);
\item $R((\pi_1;\pi_2)) = R(\pi_1) \circ R (\pi_2)$ (composition);
\item $R(\pi^*)$ is the reflexive and transitive closure of $R(\pi)$ (iteration);
\item $R(\phi?) = \{(u,u) \in W \times W \mid M,u \models \phi\}$ (test).
\end{itemize}
The semantics $N, w \models_\cpdl \phi$ extends propositional logic with dynamic operators:
\[
N,w \models [\pi] \psi \iff \text{for all $u \in W$, if $(w,u) \in R(\pi) $, then $N, u \models \psi$}.
\]

\begin{defi}[Rewriting]\label{def:rewrite-toCPDL}
For an \langcu-formula $\phi$, the \cpdl-formula $\rho(\phi)$ is constructed as follows:
\begin{enumerate}
\item Compute $\rho_1(\phi) = \phi \wedge \univ (\bigwedge_{\chi \in \mu(\phi)}\chi)$, where $\mu(\phi)$ comprises the following formulas: $\psi \ra K_a \neg K_a \neg \psi$ and $\neg K_a \neg K_a \psi \to \psi$, for all $\psi \in \cl(\phi) \cup \{ C_G \theta \mid \text{$G$ appears in $\phi$ and $\theta \in \cl(\phi)$} \}$ and $a$ appearing in $\phi$.
\item For each $a \in \ag$, replace every occurrence of $K_a$ with $[a]$.
\item For each $G \in \gr$, replace every occurrence of $C_G$ with $[((\bigcup_{a\in G}a);(\bigcup_{a\in G}a)^*)]$.
\item Replace every occurrence of $\univ$ with $[\nu]$.
\end{enumerate}
Let $\rho_{234}(\phi)$ denote the result of applying Steps (2)--(4) to $\phi$. Then define $\rho(\phi) = \rho_{234}(\rho_1(\phi))$.
\end{defi}
It is clear from the construction that if $\phi$ is an \langcu-formula, then $\rho_1(\phi)$ remains an \langcu-formula, while both $\rho (\phi)$ and $\rho_{234}(\phi)$ are $\cpdl$-formulas.

\begin{lem}\label{lem:sat-toCPDL}\label{lem:CU-CPDL}
\begin{enumerate}
\item For any \langcu-formula $\phi$, $\phi$ is satisfiable (in \lcu) if and only if $\rho(\phi)$ is satisfiable in \cpdl;
\item The satisfiability problem for \lcu is reducible to that for \cpdl in polynomial time.
\end{enumerate}
\end{lem}
\begin{proof}
Left to right. Suppose $\phi$ is satisfiable at a world $w$ in a model $M = (W, E, C, \beta)$, i.e., $M, w \models_\lcu \phi$. It can be verified that $M, w \models_\lcu \rho_1(\phi)$.
Construct a \cpdl model $N = (W,R,V)$ where $R(a) = \{(u, v) \in W \times W \mid C(a) \subseteq E(u, v)\}$ for all $a \in \ag$, and $V= \beta$.
For every $a \in \ag$ and $u, v \in W$, $C(a) \subseteq E(u, v) \iff (u, v) \in R(a)$, ensuring $M, u \models_\lcu K_a \psi \iff N, u \models_\cpdl [a] \rho_{234}(\psi)$.
For every $G \in \gr$, $R(((\bigcup_{a\in G} a);(\bigcup_{a\in G}a)^*)$ is the transitive closure of $\bigcup_{a\in G} R(a)$, match the path semantics for $C_G$, so $M, u \models_\lcu C_G \psi \iff N, u \models_\cpdl  [((\bigcup_{a\in G}a);(\bigcup_{a\in G}a)^*] \rho_{234} (\psi)$ for all $u \in W$.
Furthermore, $R(\nu) = W \times W$, so $M, u \models_\lcu U \psi \iff N, u \models_\cpdl  [\nu] \rho_{234}(\psi)$ for all $u \in W$.
An inductive proof will show that for all \langcu-formulas $\psi$ and $u \in W$, $M, u \models_\lcu \psi \iff N, u \models_\cpdl \rho_{234}(\psi)$. Since $M, w \models_\lcu \rho_1(\phi)$, it follows that $N, w \models_\cpdl \rho_{234}(\rho_1(\phi))$, i.e., $N, w \models_\cpdl \rho(\phi)$, proving $\rho(\phi)$ is satisfiable in \cpdl.

Right to left. Suppose $\rho(\phi)$ is satisfied at a world $w \in W$ in a \cpdl model $N = (W, R, V)$, i.e., $N, w \models_\cpdl \rho(\phi)$. It follows that $N, w \models_\cpdl \rho_{234}(\phi)$. Construct a model $M = (W, E, C, \beta)$ where:
\begin{itemize}
\item $E: W \times W \to \wp(\ag)$ with $E(u,v) =\{ a \in \ag \mid (u, v) \in R(a) \text{ or } (v,u) \in R(a)\}$ for all $u, v \in W$;
\item $C: \ag \to \wp(\ag)$ with $C(x) = \{x\}$ for all $x \in \ag$;
\item $\beta = V$.
\end{itemize}
Using agents as skills is justified by Footnote~\ref{ft}, and $M$ can be verified to be a model.

We prove by induction on $\psi$ that:
For all \langcu-formulas $\psi$ with $\psi \in \cl(\phi)$ and $u \in W$, $M, u \models_\lcu \psi \iff N, u \models_\cpdl \rho_{234}(\psi)$.

$\bullet$ Atomic, Boolean and Universal cases: straightforward and omitted.

$\bullet$ Case $\psi = K_a \chi$: $\rho_{234}(\psi) = [a] \rho_{234}(\chi)$.
\underline{Left to right.} Suppose $N, u \not \models_\cpdl [a] \rho_{234}(\chi)$. Then there exists $v \in W$ such that $(u, v) \in R(a)$ and $N, v \not\models_\cpdl \rho_{234}(\chi)$. By induction hypothesis, $M, v \not\models_\lcu \chi$ and by the definition of $M$, $C(a) \subseteq E(u, v)$, so $M, u \not\models _\lcu K_a \chi$.
\underline{Right to left.} Suppose $M, u \not\models_\lcu K_a \chi$. Then there exists $v \in W$ such that $M, v \not\models_\lcu \chi$, and either $(u, v) \in R(a)$ or $(v, u) \in R(a)$. By induction hypothesis, $N, v \not\models_\cpdl \rho_{234}(\chi)$. If $(u, v) \in R(a)$, then $N, u \not\models_\cpdl [a] \rho_{234}(\chi)$ directly. If $(v, u) \in R(a)$, since $N, w \models_\cpdl \rho_{234}(\univ (\bigwedge_{\theta \in \mu(\phi)} \theta))$, $N, v \models_\cpdl \neg [a] \neg [a] \rho_{234}(\chi) \to \rho_{234}(\chi)$. Hence, $N, v \not \models_\cpdl \neg [a] \neg [a] \rho_{234}(\chi)$, so $N, v \models_\cpdl [a] \neg [a] \rho_{234}(\chi)$. It follows that $N, u \models_\cpdl \neg [a] \rho_{234}(\chi)$, hence $N, u \not \models_\cpdl [a] \rho_{234}(\chi)$.

$\bullet$ Case $\psi = C_{G} \chi$: $\rho_{234}(\psi) = [((\bigcup_{a\in G}a);(\bigcup_{a\in G}a)^*)] \rho_{234}(\chi)$.
\underline{Left to right.} Suppose $N, u \not \models_\cpdl [((\bigcup_{a\in G}a);(\bigcup_{a\in G}a)^*)] \rho_{234}(\chi)$. Then there exist $u_0, \dots, u_n \in W$, $n \geq 1$ and $a_1, \dots, a_n \in G$, such that $u = u_0$ and $(u_{i-1},u_{i}) \in R(a_i)$ for all $1 \leq i \leq n$ and $N, u_n \not\models_\cpdl \rho_{234}(\chi)$. By the induction hypothesis, $M, u_n \not\models_\lcu \chi$. By the definition of $E$, for all $1 \leq i \leq n$, $C(a_i) \subseteq E( u_{i-1}, u_{i} )$. Thus, $M, u \not \models C_G \chi$ by the semantics.
\underline{Right to left.} Suppose $M, u \not \models_\lcu C_G \chi$. Then there exist $u_0, \dots, u_n \in W$, $n \geq 1$ and $a_1, \dots, a_n \in G$, such that $u = u_0$ and $M, u \not \models_\lcu C_G \chi$ and for all $1 \leq i \leq n$, $C(a_i) \subseteq E( u_{i-1}, u_{i} )$. By the induction hypothesis, $N, u_n \not\models_\cpdl \rho_{234}(\chi)$. By the definition of $E$, for all $1 \leq i \leq n$, either $(u_{i-1},u_{i}) \in R(a_i)$ or $(u_{i},u_{i-1}) \in R(a_i)$. Similarly to the proof in Case $\psi = K_a \psi$, from $N, u_n \not \models_\cpdl \rho_{234}(\chi)$ and either $(u_{n-1},u_{n}) \in R(a_n)$ or $(u_{n},u_{n-1}) \in R(a_n)$, it follows that $N, u_{n-1} \not\models_\cpdl [a_n]\rho_{234}(\chi)$. Hence $N, u_{n-1} \not\models_\cpdl [((\bigcup_{a\in G}a);(\bigcup_{a\in G}a)^*)]\rho_{234}(\chi)$. Repeat the inference, from $N, u_{n-1} \not\models_\cpdl [((\bigcup_{a\in G}a);(\bigcup_{a\in G}a)^*)]\rho_{234}(\chi)$, and since either $(u_{n-2},u_{n-1}) \in R(a_{n-1})$ or $(u_{n-1},u_{n-2}) \in R(a_{n-1})$, it follows that $N, u_{n-2} \not\models_\cpdl [a_{n-1}] [((\bigcup_{a\in G}a);(\bigcup_{a\in G}a)^*)] \rho_{234}(\chi)$. Hence $N, u_{n-2} \not\models_\cpdl [((\bigcup_{a\in G}a);(\bigcup_{a\in G}a)^*)]\rho_{234}(\chi)$. Repeat the inferences, it follows that $N, u \not\models_\cpdl [((\bigcup_{a\in G}a);(\bigcup_{a\in G}a)^*)]\rho_{234}(\chi)$. I.e., $N, u \not\models_\cpdl \rho_{234}(C_G \chi)$.

Therefore, the induction holds for all $\psi \in \cl(\phi)$. Since $N, w \models_{\cpdl} \rho_{234}(\phi)$, by the induction claim applied to $\psi = \phi$ (noting $\phi \in \cl(\phi)$), it follows that $M, w \models_\lcu \phi$, proving $\phi$ is satisfiable in \lcu.

(2) By Lemma~\ref{lem:sat-toCPDL}, the function $\rho$ can reduce the satisfiability problem of $\l_{C\univ}$ to that of \textrm{CPDL} in polynomial time.
\end{proof}

With the results established, we are now positioned to state the following theorem, drawing on Lemmas~\ref{lem:CDEF-CU}, \ref{lem:S5C-C}, and \ref{lem:CU-CPDL}.

\begin{thm}\label{thm:sat-CtoCDEF}
The satisfiability problem for any logic introduced in Section~\ref{sec:logics} that includes common knowledge but excludes update and quantifying modalities is EXPTIME complete.
\end{thm}

\subsubsection{Reduction from \texorpdfstring{\kutwo}{KU2} to \texorpdfstring{\lu}{LU}}

While Theorem~\ref{thm:sat-CtoCDEF} resolves the satisfiability problems for logics ranging from \lu to \lGA, it does not fully complete the roadmap outlined in Figure~\ref{fig:sat-C}. Specifically, the introduction of the universal modality ($U$) in our proofs, intended to streamline the analysis, results in 16 additional logics beyond those defined in Section~\ref{sec:logics}. These logics vary based on the inclusion of the operators $C$ (common knowledge), $D$ (distributed knowledge), $E$ (everyone knows), and $F$ (field knowledge). Among them, only those between \lcu and \lGA have been established as EXPTIME-complete for satisfiability, as shown in prior results (e.g., Lemma~\ref{lem:CU-CPDL}). The complexity of the satisfiability problems for the remaining logics remains unresolved. In this section, we address this gap, demonstrating that all logics incorporating the universal modality but excluding update and quantifying modalities have an EXPTIME-complete satisfiability problem.

To achieve this, we propose a transformation that converts any two-agent \langu-formula (with agents denoted $a, b \in \ag$) satisfiable in \kutwo---the classical bimodal logic extended with the universal modality---into an \langu-formula satisfiable in \lu. Recall that in a Kripke model $N = (W, R, V)$, the formula $U \phi$ holds at a world $w$, i.e., $N, w \models U \phi$, if and only if  $N, u \models \phi$ for all $u \in W$.

It is established in \cite[Corollary 5.4.8]{Spaan1993} that the satisfiability problem for \kutwo (denoted $L^\Box$ therein) is EXPTIME complete.

\begin{defi}[Rewriting]\label{def:rewrite-to=}
For a two-agent \lu-formula $\phi$ (with agents $a, b \in \ag$), define $\rho(\phi)$ as a four-agent \langu-formula (using agents $a_1, a_2, b_1, b_2 \in \ag$ that are distinct from $a, b$ and each other) by applying the following steps sequentially, where $p$ is a fresh atomic proposition not appearing in $\phi$:
\begin{enumerate}
\item Replace every occurrence of $K_a \theta$ in $\phi$ with $K_{a_1} K_{a_2} (p \ra \theta)$, every occurrence of $K_b \theta$ with $K_{b_1} K_{b_2} (p \ra \theta)$, and every occurrence of $\univ \theta$ in $\phi$ with $\univ( p \ra \theta)$. Denote the resulting formula by $\rho_1(\phi)$.

\item Define $\rho(\phi) = \rho_1(\phi) \wedge p \wedge \univ ((p \ra \bigwedge_{x \in \{a_1,a_2,b_1,b_2\}} K_{x} \neg p) \wedge (\neg p \ra \bigwedge_{x \in \{a_1,a_2,b_1,b_2\}} K_{x} p))$.
\end{enumerate}
\end{defi}
It is clear from the construction that if $\phi$ is a two-agent \langu-formula, then both $\rho (\phi)$ and $\rho_1(\phi)$ are four-agent \langu-formulas.

\begin{lem}\label{lem:sat-to=}\label{lem:red-to=}
\begin{enumerate}
\item For any two-agent \langu-formula $\phi$, $\phi$ is satisfiable in \kutwo if and only if $\rho(\phi)$ is satisfiable (in \lu);
\item The satisfiability problem for \kutwo is polynomial-time reducible to that for $\l_{\univ}$.
\end{enumerate}
\end{lem}
\begin{proof}
Left to right. Suppose there exists a Kripke model $N = (W, R, V)$ and a world $w \in W$ such that $N, w \models_\kutwo \phi$. Construct a model $M = (W', E, C, \beta)$ where:
\begin{itemize}
\item $W' = W \cup (W \times W)$ (with $W \times W$ denoted $W^2$ for short);
\item $E: W' \times W' \to \wp (\ag)$, defined as:
	\[
	E (x, y) = \left\{
	\begin{array}{ll}
	\emptyset,& \text{if $x, y \in W$,} \\
	\emptyset,& \text{if $x, y \in W^2$,} \\
	\emptyset,& \text{if $x \in W,\, y \in W^2,\, x \notin y$,} \\
	\{ a_1 \mid y \in R(a)\} \cup \{ b_1 \mid y \in R(b)\},& \text{if $x \in W,\, y \in W^2,\, x = l(y) \neq r(y)$,} \\
	\{ a_2 \mid y \in R(a)\} \cup \{ b_2 \mid y \in R(b)\},& \text{if $x \in W,\, y \in W^2,\, x = r(y) \neq l(y)$,} \\
	\{ a_1,a_2 \mid y \in R(a)\} \cup \{ b_1,b_2 \mid y \in R(b)\},& \text{if $x \in W,\, y \in W^2,\, y = (x, x)$,} \\
	\emptyset,& \text{if $y \in W,\, x \in W^2,\, y \notin x$,} \\
	\{ a_1 \mid x \in R(a)\} \cup \{ b_1 \mid x \in R(b)\},& \text{if $y \in W,\, x \in W^2,\, y = l(x) \neq r(x)$,} \\
	\{ a_2 \mid x \in R(a)\} \cup \{ b_2 \mid x \in R(b)\},& \text{if $y \in W,\, x \in W^2,\, y = r(x) \neq l(x)$,} \\
	\{ a_1,a_2 \mid x \in R(a)\} \cup \{ b_1,b_2 \mid x \in R(b)\},& \text{if $y \in W,\, x \in W^2,\, x = (y, y)$,} \\
	\end{array}
	\right.
	\]
where $l(z)$ and $r(z)$ denote the left and right elements of a pair $z \in W^2$;
\item $C:\ag \to \wp(\ag)$ with $C(x)=\{x\}$ for all $x \in \ag$;
\item $\beta: W' \to \wp(\pr)$ with $\beta(x) = V(x) \cup \{p\}$ and $\beta((x, y)) = \emptyset$ for all $x, y \in W$.
\end{itemize}
Using agents as skills is justified by Footnote~\ref{ft}, and $M$ can be verified to be a model.

We prove by induction:
for all two-agent \langu-formulas $\psi$ not containing $p$, and all $u \in W$, $M, u \models_\lu \rho_1(\psi) \iff N, u \models_\kutwo \psi$.

\begin{itemize}
\item Atomic and Boolean cases are omitted.
\item Case $\psi = K_a \chi$: $\rho_1(\psi) = K_{a_1} K_{a_2} (p \ra \rho_1(\chi))$.
Observe: for all $x \in W'$, $M, x \models_\lu p$ iff $x \in W$.
For any $u \in W$:

$M, u \not\models_\lu K_{a_1}K_{a_2} (p \to \rho_1(\chi))$ \\
iff there exists $v \in W$ such that $a_1 \in E(u,(u, v))$, $a_2 \in E((u, v),v)$ and $M, v \not\models_\lu \rho_1(\chi)$ \\
iff there exists $v \in W$ such that $(u, v) \in R(a)$ and $M, v \not\models_\lu \rho_1(\chi)$ \\
iff there exists $v \in W$ such that $(u,v) \in R(a)$ and $N, v \not\models_\kutwo \chi$ \\
iff $N, u \not\models_\kutwo K_a \chi$.

\item Case $\psi = K_b \chi$: analogous, using $b_1$ and $b_2$.

\item Case $\psi = \univ \chi$: $\rho_1(\psi) = \univ (p \ra \rho_1(\chi))$. For all $u \in W$:

$N, u \not \models_\kutwo \univ \chi$ \\
iff there exists $v \in W$ such that $N, v \not \models_\kutwo \chi$ \\
iff there exists $v \in W$ such that $M, v \not \models_\lu \rho_1(\chi)$ (by the induction hypothesis)\\
iff $M, u \not \models_\lu \univ (p \ra \rho_1(\chi))$ (since for any $u' \in W'$, $M, u'\models p$ iff $u' \in W$).
\end{itemize}

It then follows from the claim that $M, w \models_\lu \rho_1(\phi)$. It can be verified that $M, u \models_\lu \univ ((p \ra \bigwedge_{x \in \{a_1,a_2,b_1,b_2\}} K_{x} \neg p) \wedge (\neg p \ra \bigwedge_{x \in \{a_1,a_2,b_1,b_2\}} K_{x} p))$ for any $u \in W'$. Moreover, notice that $M, w \models p$. Thus, $M, w \models_\lu \rho(\phi)$, proving that $\rho(\phi)$ is satisfiable.

Right to left. Suppose there exists a model $M = (W, E, C, \beta)$ and a world $w \in W$ such that $M, w \models_\lu \rho(\phi)$. Then $M, w \models_\lu \rho_1(\phi)$ and $M, w \models_\lu p \wedge \univ ((p \ra \bigwedge_{x \in \{a_1,a_2,b_1,b_2\}} K_{x} \neg p) \wedge (\neg p \ra \bigwedge_{x \in \{a_1,a_2,b_1,b_2\}} K_{x} p))$.

Construct a two-agent Kripke model $N = (W',R,V)$ where:
\begin{itemize}
\item $W' = \{u \in W \mid M, u \models_\lu p \}$;
\item $R: \ag \to W' \times W'$ such that for any $u, v \in W'$:
	\begin{itemize}
	\item $(u,v) \in R(a)$ iff there exists $x \in W$ such that $C(a_1) \subseteq E(u, x)$ and $C(a_2) \subseteq E(x, v)$;
	\item $(u,v) \in R(b)$ iff there exists $x \in W$ such that $C(b_1) \subseteq E(u, x)$ and $C(b_2) \subseteq E(x, v)$;
	\end{itemize}
\item $V :W' \to \wp(\pr)$ with $V(u) = \beta(u)$ for all $u \in W'$.
\end{itemize}

We prove by induction:
for all two-agent \langu-formulas $\psi$ not containing $p$, and all $u \in W'$, $M,u \models_\lu \rho_1(\psi) \iff N, u \models_\kutwo \psi$.

\begin{itemize}
\item Atomic and Boolean cases are straightforward and omitted.

\item Case $\psi = K_a \chi$: $\rho_1(\psi) = K_{a_1} K_{a_2} (p \to \rho_1(\chi))$.
\underline{Left to right.} Suppose $N, u \not \models_\kutwo K_a \chi$, then there exists $v \in W'$ such that $(u,v) \in R(a)$ and $N, v \not \models_\kutwo \chi$. Then there exists $u' \in W$ such that $C(a_1) \subseteq E(u,u')$ and $C(a_2) \subseteq E(u',v)$. Notice that since $v \in W'$, so $M, v \models_\lu p$, by induction hypothesis, $M, v \not \models_\lu (p \ra \rho_1(\chi))$. Hence $M, u \not \models_\lu K_{a_1}K_{a_2} (p \to \rho_1(\chi))$.
\underline{Right to left.} Suppose $M, u \not \models_\lu K_{a_1}K_{a_2} (p \to \rho_1(\chi))$, then there exists $u',v \in W$ such that $C(a_1) \subseteq E(u,u')$ and $C(a_2) \subseteq E(u',v)$ and $M, v \not \models_\lu (p \to \rho_1(\chi))$. Hence $M, v \not \models_\lu \rho_1(\chi)$ and $M, v \models_\lu p$. So $v \in W'$ and $(u,v) \in R(a)$. By induction hypothesis, $N, v \not \models_\kutwo \chi$, thus $N, v \not \models_\kutwo K_a \chi$.

\item Case $\psi = \univ \chi$: $\rho_1(\psi) = \univ (p \ra \rho_1(\chi))$. For any $u \in W'$:

$N, u \not \models_\kutwo \univ \chi$ \\
iff there exists $v \in W'$ such that $N, v \not \models_\kutwo \chi$ \\
iff there exists $v \in W'$ such that $M, v \not \models_\lu \rho_1(\chi)$ (by the inductive hypothesis) \\
iff $M, u \not \models_\lu \univ (p \ra \rho_1(\chi))$ (since for any $v \in W$, $v \in W'$ iff $M, v\models p$).
\end{itemize}

Thus, $N, w \models_\kutwo \phi$, proving $\phi$ satisfiable in \kutwo.

(2) The transformation $\rho(\phi)$ is computable in polynomial time (linear in $|\phi|$), and statement (1) establishes it as a valid reduction.
\end{proof}

The computational complexity results from Sections~\ref{sec:mc} (model checking) and \ref{sec:sat} (satisfiability) are summarized in Table~\ref{tblr:complexity}.%
\footnote{The validity problem is dual to the satisfiability problem. Since co-PSPACE equals PSPACE and co-EXPTIME equals EXPTIME, the complexity of validity problems follows directly from satisfiability results.}
While the complexity of model checking is fully resolved for all 2048 logics proposed, determining the complexity of satisfiability and validity problems in a general manner proves more challenging. In particular, the complexity of the satisfiabilty and validity problems for \logic{\mathscr{C}} (logics with update modalities but no quantifiers) and \logic{\mathscr{D}} (logics with quantifiers) remains open.

\begin{longtblr}[
	caption = {Summary of the Complexity Results},
	label = {tblr:complexity},
	note{i} = {$\logic{\mathscr{A}}$ denotes any logic without common knowledge, update modalities or quantifiers, i.e., $\logic{DEF}$ or any of its sublogics. Formally, $\{ C, +, -, =, \equiv, \boxplus, \boxminus, \Box \} \cap \mathscr{A} = \emptyset$.},
	note{ii} = {$\logic{\mathscr{B}}$ denotes any logic with common knowledge but without update or quantifying modalities, i.e., $C \in \mathscr{B}$ and $\{ +, -, =, \equiv, \boxplus, \boxminus, \Box \} \cap \mathscr{B} = \emptyset$.},
	note{iii} = {$\logic{\mathscr{C}}$ denotes any logic with at least one update modality but without quantifiers, i.e., $\{ +, -, =, \equiv \} \cap \mathscr{C} \neq \emptyset$ and $\{ \boxplus, \boxminus, \Box \} \cap \mathscr{C} = \emptyset$.},
	note{iv} = {$\logic{\mathscr{D}}$ denotes any logic with at least one quantifier, i.e., $\{ \boxplus, \boxminus, \Box \} \cap \mathscr{D} \neq \emptyset$.},
]{
	columns = {c},
	column{1} = {l},
	colsep = {1em},
	row{1} = {font=\bfseries},
	rows = {2em, m},
	rowsep = {3pt},
	hline{1, 4} = {1pt, solid},
	hline{2},
}
Problems & $\logic{\mathscr{A}}$ & $\logic{\mathscr{B}}$ & $\logic{\mathscr{C}}$ & $\logic{\mathscr{D}}$
\\
Model checking & in P & in P & in P & {PSPACE\\[-3pt]complete}
\\
{Validity /\\Satisfiability} & {PSPACE\\[-3pt]complete} & {EXPTIME\\[-3pt]complete} & ? & ?
\end{longtblr}

\section{Discussion}
\label{sec:conclusion}

We have introduced a family of expressive epistemic logics that capture individual and group knowledge including common, mutual, distributed, and field knowledge, alongside epistemic actions such as knowing, forgetting, revising, and learning, as well as their necessity and possibility. Despite their high expressivity, these logics maintain reasonable computational complexity for central decision problems, namely satisfiability and model checking. Specifically:
\begin{itemize}
\item For logics without update modalities or quantifiers, satisfiability is PSPACE complete when common knowledge is absent, and EXPTIME complete when common knowledge is present. These results align with classical epistemic logics under standard Kripke semantics, as summarized in~\cite{FHMV1995}.
\item For logics without quantifiers, model checking is in P, consistent with many traditional epistemic logics.
\item For logics incorporating quantifiers, model checking becomes PSPACE complete, matching the complexity known from related frameworks such as Group Announcement Logic \cite{ABDS2010}, Coalition Announcement Logic \cite{Pauly2002,GAD2018,ADGW2021}, and Subset Space Arbitrary Announcement Logic \cite{BDK2013}.%
\footnote{It is noteworthy that model checking in Arbitrary Public Announcement Logic (APAL) has been claimed to be PSPACE complete \cite{BBDHHL2008}; however, we have not identified a detailed proof confirming this result.}
\end{itemize}
Our framework naturally generalizes to accommodate fuzzy skill sets and lattice-structured skills, enhancing its applicability to practical domains and real-world scenarios.

The decidability of validity and satisfiability problems in logics that employ quantification over epistemic updates has long intrigued logicians. Known negative results, such as the undecidability of Arbitrary Public Announcement Logic (APAL) and Group Announcement Logic~\cite{FD2008,ADF2016}, have motivated efforts toward identifying decidable fragments~\cite{FD2008,DFP2010,DF2022}. Even obtaining recursively axiomatizable systems constitutes notable progress~\cite{XW2018,BOS2023}, particularly given APAL's expected lack of recursive axiomatizability.
Past approaches, exemplified by~\cite{BBDHHL2008,ABDS2010,BDK2013}, predominantly rely on syntactic strategies---quantifying over formulas and indirectly updating models---which likely complicates satisfiability analysis. Our logic introduces an alternative semantic perspective, explicitly quantifying over semantic objects (updates of epistemic skills) instead of syntactic formulas. This semantic viewpoint complements other semantic frameworks, such as topological semantics explored in~\cite{WA2013SSPAL,BOS2017}, thereby enriching the theoretical landscape of epistemic update logics.

A primary goal of our ongoing research is to further delineate the decidability and computational complexity boundaries for satisfiability and validity problems within our logics. While we have established complexity results for simpler variants---for example, PSPACE-completeness for satisfiability without common knowledge, updates, or quantifiers (Theorem~\ref{thm:complexity-satpsp}), and EXPTIME-completeness for satisfiability with common knowledge but without updates or quantifiers (Theorem~\ref{thm:sat-CtoCDEF})---the computational complexity and decidability status of logics incorporating update modalities and quantifiers remain open challenges. In particular, the decidability of the full logic \langGUQ, which encompasses all knowledge modalities, update operations, and quantification mechanisms, remains unresolved.

Moreover, although some fragments of our logics have been completely axiomatized in earlier work~\cite{LW2022b,LW2024b}, a complete axiomatic system for the full logic has yet to be developed. Addressing these open problems constitutes an important direction for future research.

Additionally, we introduced a novel epistemic update modality, $(\equiv_b)_a$, representing the action wherein agent $a$ learns by adopting agent $b$'s skill set, effectively replacing $a$'s skills with those of $b$. We have also considered several variants to enable more nuanced skill modifications: incremental skill acquisition---adding $b$'s skills---via the operator $(+_b)_a$ (alternatively expressed using set notation as $(\cup_b)_a$); retaining only commonly beneficial skills via $(\cap_b)_a$; and removing undesirable skills via $(-_b)_a$ (or equivalently, $(\setminus_b)_a$). Further inspired by natural language, we are interested in studying the concept of ``deskilling,'' an epistemic update that reduces the complexity of skills required to distinguish epistemic possibilities, potentially enhancing knowledge by simplifying the underlying edge structure. Importantly, these diverse update modalities do not elevate the complexity of the model checking problem beyond P or PSPACE (depending on the presence of quantifiers), although they may complicate the satisfiability problem. Quantification over these richer learning operators offers a promising avenue for further study.

\section*{Acknowledgement}

We express our sincere gratitude to the anonymous reviewers of earlier drafts for their insightful and invaluable feedback that shaped this paper. We also thank the anonymous reviewers and the participants of GandALF 2024 for their constructive suggestions which significantly improved this work. Additionally, we acknowledge the financial support provided by the Project of Humanities and Social Sciences from the Ministry of Education of China (Grant No.\,24YJA72040002).

\bibliographystyle{alphaurl}
\bibliography{main}
\end{document}